%% file: main.tex
\documentclass{article}
\usepackage{microtype}
\usepackage{graphicx}
\usepackage{subcaption}
\usepackage{booktabs} 
\usepackage[hidelinks]{hyperref}

\usepackage[preprint]{neurips_2026}
\usepackage{algorithm}
\usepackage{algorithmic}

\usepackage[utf8]{inputenc} 
\usepackage[T1]{fontenc}    
\usepackage{natbib}
\usepackage{amsthm}
\usepackage{amsfonts}       
\usepackage{amsmath}
\usepackage{amssymb}
\usepackage{mathtools}
\usepackage{mathrsfs}
\usepackage{hyperref}       
\usepackage{url}            

\usepackage{nicefrac}       
\usepackage{microtype}      
\usepackage{xcolor}         
\usepackage{comment}

\usepackage{svg}
\usepackage{amsmath}
\usepackage{algorithm}
\usepackage{algorithmic}
\usepackage{enumerate}
\usepackage{times}
\usepackage{dsfont}
\usepackage{caption,color}

\usepackage{ulem}
\usepackage{units}

\newcommand{\cA}{\mathcal{A}}

\newcommand{\cD}{\mathcal{D}}
\newcommand{\cE}{\mathcal{E}}
\newcommand{\cF}{\mathcal{F}}
\newcommand{\cG}{\mathcal{G}}

\newcommand{\frakc}{\mathfrak{c}}

\newcommand{\frakb}{\mathfrak{b}}
\newcommand{\fraka}{\mathfrak{a}}
\newcommand{\frakd}{\mathfrak{d}}
\newcommand{\frake}{\mathfrak{e}}

\newcommand{\cM}{\mathcal{M}}
\newcommand{\at}{\tilde{\alpha}}

\newcommand{\cQ}{\mathcal{Q}}

\newcommand{\cR}{\mathcal{R}}
\newcommand{\cS}{\mathcal{S}}
\newcommand{\cT}{\mathcal{T}}
\newcommand{\cX}{\mathcal{X}}
\newcommand{\cY}{\mathcal{Y}}
\newcommand{\cZ}{\mathcal{Z}}

\newcommand{\bE}{\mathbb{E}}
\newcommand{\bN}{\mathbb{N}}
\newcommand{\bP}{\mathbb{P}}
\newcommand{\bR}{\mathbb{R}}

\newcommand{\ust}{^{\star}}

\newcommand{\reg}{\mathrm{reg}}
\newcommand{\regn}{\mathrm{reg}^{(n)}}

\newcommand{\eps}{\epsilon}

\newcommand{\seg}{S$\eps$G}

\newtheorem{assumption}{Assumption}
\newtheorem{condition}{Condition}
\newtheorem{remark}{Remark}

\newcommand{\br}[1]{\left(#1\right)}
\newcommand{\flbr}[1]{\left\{#1\right\}}

\newcommand{\ind}[1]{\mathbf{1}{\flbr{#1}}}

\newcommand{\mycomment}[1]{}

\newcommand{\rs}[1]{\textcolor{red}{#1}}

\newcommand{\al}[1]{\begin{align}#1\end{align}}
\newcommand{\nal}[1]{\begin{align*}#1\end{align*}}
\newcommand{\s}{^{(B)}}
\newcommand{\es}{^{(S \epsilon G)}}
\newtheorem{theorem}{Theorem}[section]
\newtheorem{lemma}[theorem]{Lemma}

\newtheorem{definition}[theorem]{Definition}

\newtheorem{proposition}{Proposition}

\title{Regret and Sample Complexity of Online Q-Learning via Concentration of Stochastic Approximation with Time-Inhomogeneous Markov Chains}

\author{
  Rahul Singh \thanks{MBZUAI, UAE}
  \\
  \And
  Siddharth Chandak \thanks{Stanford University, USA}\\
  \And
  Eric Moulines 
  \footnotemark[1]
  \thanks{EPITA, France}
  \\
\And
Vivek S. Borkar \thanks{Indian Institute of Technology Bombay, India}\\
\And
Nicholas Bambos\footnotemark[2]
}

\begin{document}

\maketitle

\begin{abstract}
We present the first regret bound for classical online Q-learning in infinite-horizon discounted Markov decision processes (MDPs), without relying on optimism or bonus terms. We first analyze Boltzmann Q-learning with decaying temperature and show that its regret depends critically on the suboptimality gap of the MDP: for sufficiently large gaps, the regret is sublinear, while for small gaps it deteriorates and can approach linear growth. To address this limitation, we study a Smoothed $\epsilon_n$-Greedy exploration scheme that combines $\epsilon_n$-greedy and Boltzmann exploration, for which we prove a gap-robust regret bound of near-$\tilde{O}(N^{9/10})$. We also obtain sample complexity guarantees, with both regret and sample complexity bounds holding with high probability. To analyze these algorithms, we develop a high-probability concentration bound for contractive Markovian stochastic approximation with iterate- and time-dependent transition dynamics. This bound may be of independent interest as the contraction factor in our framework is allowed to converge to one asymptotically.
\end{abstract}

\input{introduction}
\input{stochastic_approximation}
\input{q-learning}

\input{q-learning_performance}
\newpage
\bibliography{ref}
\bibliographystyle{abbrvnat}
\newpage
\appendix
\input{appendix_organization}

\input{proof_sketch}

\input{appendix_SA}

\input{appendix_SA_auxiliary}

\input{appendix_Q-learning}
\input{appendix_q-learning_auxiliary}
\input{appendix_eps_softmax}
\input{appendix_eps_softmax_auxiliary}
\input{appendix_miscellaneous}
\end{document}

%% file: introduction.tex
\section{Introduction}\label{sec:intro}
Reinforcement Learning (RL) provides a framework for sequential decision-making and has enabled major advances across domains ranging from large language models (LLMs)~\citep{ouyang2022training} to robotics~\citep{levine2016end} and healthcare~\citep{yu2021reinforcement}. Among the many RL algorithms developed over the past decades, Q-learning~\citep{watkins1992qlearning} remains a seminal method due to its simplicity and model-free nature, and underlies modern approaches such as Deep Q-Networks (DQN)~\citep{mnih2015human}. Classical analyses~\citep{tsitsiklis1994asynchronous,borkar2000ode} established guarantees only in the asymptotic regime, which provides limited insight in practice. This limitation motivates a growing body of finite-time analyses~\citep{qu2020finite,li2024q}.

In this work, we study the sample complexity and regret of online Q-learning in infinite-horizon discounted Markov decision processes (MDPs).~The majority of existing work on the sample complexity of Q-learning focuses on the offline setting, where state-action trajectories are generated by a fixed behavioral policy and the objective is to estimate the optimal $Q\ust$~\citep{qu2020finite,lisampleit,li2024q}. In contrast, online Q-learning uses the evolving Q-value estimates to determine the actions chosen by the agent. Offline analyses primarily focus on controlling the estimation error $\|Q_n - Q\ust\|_\infty$ and, as such, do not provide guarantees on the performance during learning. On the other hand, existing regret analyses for online Q-learning in this setting either consider model-based methods~\citep{he2021nearly} or variants of Q-learning augmented with optimism~\citep{ji2023regret}. 

Despite the practical importance of online Q-learning with simple exploration strategies, prior work does not provide regret or sample complexity guarantees for classical, fully model-free Q-learning without optimism terms. Most notably, Deep Q-Networks (DQN)~\citep{mnih2015human,mnih2013playing}, one of the most influential algorithms in modern reinforcement learning, employs $\epsilon$-greedy exploration. Similarly, Boltzmann (softmax) exploration has been widely adopted in RL applications due to its smooth and differentiable exploration behavior~\citep{sutton1998reinforcement}. Yet even for these foundational exploration strategies, no theoretical guarantees on sample complexity or regret have been established. 

To bridge this gap, we establish sample complexity and regret guarantees for Q-learning with standard Q-value updates under two exploration policies. We first analyze the widely used Boltzmann Q-learning algorithm, in which actions are selected according to a Boltzmann exploration rule. We then study a Smoothed $\epsilon_n$-Greedy (\seg) Q-learning algorithm, whose exploration policy is a convex combination of the uniform policy and the softmax policy. The smoothing is introduced solely to facilitate the analysis and does not alter the practical behavior of the algorithm, which closely mirrors $\epsilon$-greedy exploration. Our results rely on a novel concentration bound for a Markovian stochastic approximation (SA) framework that allows the contraction factor to converge to one asymptotically.

\subsection{Main contributions} \label{sec:contribution}
We summarize the main contributions of this work below. In addition to establishing regret and sample complexity guarantees for online Q-learning in infinite-horizon discounted MDPs, we develop a novel analytical framework for Markovian SA that may be of independent interest.

\textbf{Regret Guarantees for Boltzmann and \seg~Q-learning.} We establish the first regret guarantees for classical, model-free Q-learning without relying on additional optimism terms. We first characterize how the regret of Boltzmann Q-learning depends on the suboptimality gap of the Markov decision process, showing that it can be arbitrarily close to linear for small gaps. To address this limitation, we analyze the Smoothed $\epsilon_n$-Greedy (\seg) Q-learning algorithm and show that it achieves a regret bound of near-$\tilde{O}(N^{9/10})$. While this rate is worse than the $\tilde{O}(\sqrt{N})$ dependence achieved by optimism-based methods, it demonstrates that sublinear regret is attainable for unmodified online Q-learning. 

\textbf{Sample Complexity of Online Q-learning.} We provide high-probability sample complexity guarantees exhibiting sub-Gaussian tail behavior for online Q-learning under Boltzmann and \seg~exploration. Our results show how the exploration parameters, specifically the temperature schedule for the softmax exploration and the sequence $\epsilon_n$ for $\epsilon$-greedy policies, affect the convergence rate. Moreover, our concentration bounds enable us to establish almost sure convergence of the Q-learning iterates, which is nontrivial for online Q-learning with decaying exploration. Together with the regret guarantees, these results illustrate the exploration-exploitation tradeoff in Q-learning. 

\textbf{Concentration Bound for SA with Time-Inhomogeneous Markov Chains.} We develop a novel concentration bound for Markovian SA with time-inhomogeneous dynamics, where the transitions depend on the iterate and an external, time-varying control sequence $\zeta_n$. Our framework allows the mixing time to diverge and the contraction factor to approach one. This is essential for analyzing online Q-learning, where suboptimal actions are chosen with vanishing probability as $n \to \infty$. We combine novel Poisson equation techniques with a noise-averaging argument to obtain sharp concentration bounds despite the weakening contraction. Beyond Q-learning, the resulting concentration inequality may be of independent interest for RL and optimization algorithms.


\subsection{Past works}
We review related works on Q-learning, RL and SA below, focusing on sample complexity, regret guarantees, and finite-time analysis of SA. Our results are discussed in relation to these lines of work.

 There is a vast literature on SA focusing on asymptotic analysis (see \citep{kushner2003stochastic, borkar_book}). Motivated by applications in reinforcement learning and optimization, recent work has studied finite-time performance guarantees for SA. These results can be broadly divided into in-expectation bounds and high-probability bounds. Representative works establishing in-expectation bounds for SA include \citep{srikant2019finite, chen2020finite, chen2024lyapunov, chandak_cdc}. In-expectation bounds are typically easier to obtain than high-probability bounds, as the latter require controlling tail behavior rather than expectations. This work focuses on establishing high-probability guarantees.

There is also a growing literature on high-probability guarantees for RL algorithms via contractive SA. Some works \citep{borkar2021concentration, chen2025concentration} do not consider Markovian noise and therefore apply only to synchronous Q-learning. Other works that incorporate Markovian noise \citep{qu2020finite, chen2025concentration, qian2024almost} assume time-homogeneous Markov chains, which limits their applicability to offline Q-learning and algorithms such as TD(0). In contrast, our concentration bound allows for time-inhomogeneous Markov chains, enabling the analysis of online Q-learning.

Q-learning was introduced in~\citep{watkins1992qlearning}, and its almost sure convergence was established in~\citep{jaakkola1994convergence}. There is a large literature on finite-time analysis of offline Q-learning. These include works such as~\citep{beck2013improved, lee2024final, wainwright2019cone, wainwright2019variance, even2003learning}, which establish high-probability convergence rates for \(\|Q_n - Q\ust\|\). More recently, \citet{li2024q} derived a sample complexity of \(\tilde{O}\!\left((1-\gamma)^{-4} \varepsilon^{-2}\right)\) and showed its optimality. All of these works assume that samples are generated from a behavior policy and therefore correspond to the offline setting. Among results based on SA techniques, the best known sample complexity is $\tilde{O}\!\left((1-\gamma)^{-5} \varepsilon^{-2}\right)$~\citep{qu2020finite, chen2025concentration}. We recover this rate as a special case, while our main results apply to the online setting with adaptively generated samples.

Finite-time analyses of online Q-learning without modifying the Q-value updates are relatively scarce. Prior work includes~\citet{nanda2025minimal}, which establishes finite-time bounds on the estimation error for \seg~Q-learning. However, these results are limited to bounds in expectation and thus do not provide high-probability guarantees or imply regret bounds. On the other hand, \citet{chandak2022concentration} derive high-probability guarantees for SA with controlled Markovian noise, where the Markov process depends on the current iterate. Their analysis, however, does not allow the mixing time to diverge or the contraction factor to approach one, which is essential for establishing regret guarantees in online Q-learning. Our results relax these assumptions and strictly generalize this line of work.

Finally, we review regret guarantees for Q-learning in discounted MDPs. Existing regret analyses primarily consider optimism-based or model-based variants of Q-learning. The work of~\citep{ji2023regret} develops a model-free algorithm based on variance reduction and achieves a regret bound of \(O(\sqrt{N/(1-\gamma)^3})\), valid after a burn-in period. \citet{liu2020regret} proposes an optimism-based Double Q-learning algorithm with regret \(O(\sqrt{N/(1-\gamma)^{1.5}})\), while \citet{he2021nearly} introduces the model-based algorithm UCBVI-\(\gamma\) with the same regret scaling. More recently, \citet{kash2024slowly} use techniques from adversarial bandits to obtain a regret bound of \(O(\sqrt{N/(1-\gamma)^6})\). 

It is worth noting that our regret rates, while worse than the \(\tilde{O}(\sqrt{N})\) achieved by optimism-based methods, align with fundamental limitations of exploration without optimism. In particular, for multi-armed bandits, softmax exploration is known to yield linear regret in the worst case~\citep{cesa2017boltzmann}, while $\epsilon$-greedy strategies without optimism achieve regret of \(\tilde{O}(N^{2/3})\)~\citep{panageas2020multiarmedbandits}. Our \(\tilde{O}(N^{9/10})\) bound for \seg~Q-learning demonstrates that sublinear regret remains achievable in the significantly more complex MDP setting without algorithmic modifications.

\subsection{Outline and notation}
The paper is organized as follows. Section~\ref{sec:SA} presents the Markovian SA framework and our concentration bound result. Section~\ref{sec:q_learning} introduces the Q-learning algorithm and defines the regret for infinite-horizon discounted MDPs. Section~\ref{sec:performance} then presents the Boltzmann and \seg~exploration policies, along with sample complexity and regret guarantees. Proofs are deferred to the appendices.

Throughout this work, $\|\cdot\|$ denotes any compatible norm on $\mathbb{R}^d$ while $\|\cdot\|_\infty$ denotes the $\ell_\infty$ norm. We use $\tilde{O}(\cdot)$ notation to suppress logarithmic factors, i.e.,  $f = \tilde{O}(h)$ if $f = O(h \,\mathrm{polylog}(h))$.
 For event $\cE$, $\ind{\cE}$ is its indicator random variable. We use $\mathbf{0}$ for the all-zero vector of appropriate dimension.

%% file: stochastic_approximation.tex
\section{Concentration bounds for stochastic approximation}\label{sec:SA}
In this section, we first introduce the stochastic approximation (SA) formulation and then state the corresponding concentration bound. Consider the iteration
\al{
	x_{n+1}=x_n+\beta_n\left(F(x_n,y_n)-x_n+M_{n+1}\right),\;\; n\geq 0.\label{def:gen_SA}
}
Here, $x_n \in \cX \subset \bR^d$ denotes the iterate, where $\cX$ is a compact set. The function $F: \cX \times \cY \to \bR^d$ represents the update direction. The sequence $\{y_n\}$ is the underlying Markov chain that induces the Markovian noise, and $\{M_{n+1}\}$ is the martingale difference noise. For $\beta,n_0>0$, and $\fraka\in[0,1)$, the stepsize sequence $\{\beta_n\}$ is of the form
\al{
		\beta_n = \frac{\beta}{(n+n_0)^{1-\fraka}}, n\in\bN.
}

We next state the assumptions, which are motivated by online Q-learning and are standard in the SA literature. The first two assumptions concern the Markov chain and its transition dynamics.
\begin{assumption}\label{assum:1}
The process $\{y_n\}$ is a controlled Markov chain with a finite state space $\cY$ whose transition kernel depends on $(x_n,\zeta_n)$, where $\zeta_n \in \cZ \subseteq \bR^{d_1}$. Let $\{\mathcal{F}_n\}_{n \ge 0}$ denote the filtration generated by $\{(y_m, x_m, M_m)\}_{m=0}^n$. Then, for all $n \ge 0$,
\[
\mathbb{P}(y_{n+1} = y' \mid \mathcal{F}_n)
= p^{(\zeta_n, x_n)}(y_n, y') \quad \text{a.s.}
\]
\end{assumption}
The dynamics of $\{y_n\}$ depend on the iterates $\{x_n\}$ and sequence $\{\zeta_n\}$ which represents an external control process. For Boltzmann Q-learning, $\zeta_n$ corresponds to the temperature schedule $\lambda_n$, while for \seg~Q-learning it corresponds to $(\epsilon_n,\lambda_n)$, the exploration and temperature schedules, respectively. 

\begin{assumption}\label{assum:2}
For each $(\zeta,x) \in \cZ \times \cX$, the Markov chain with transition kernel $p^{(\zeta,x)}$ on the state space $\cY$ is irreducible and admits a unique stationary distribution, denoted by $\mu^{(\zeta,x)} = \{\mu^{(\zeta,x)}(i)\}_{i \in \cY}$.~Define
\[
\mu_{\min}(\zeta) := \min_{x \in \cX} \min_{i \in \cY} \mu^{(\zeta,x)}(i).
\]
We assume that $\mu_{\min}(\zeta) > 0$ for all $\zeta \in \cZ$.

Moreover, there exists a designated state $i\ust \in \cY$ such that, for any $(\zeta,x) \in \cZ \times \cX$ and any initial state, the expected hitting time of $i\ust$ is upper bounded by $\frakc_1 / \mu_{\min}(\zeta)$, for some constant $\frakc_1 > 0$.
\end{assumption}
The condition $\mu_{\min}(\zeta) > 0$ is standard for irreducible Markov chains on finite state spaces. The second part of the assumption controls the chain's sensitivity through the expected hitting time of state $i\ust$.~An equivalent condition could be in terms of the mixing time, since for irreducible finite-state Markov chains, bounds on hitting times and mixing times are equivalent up to constants \citep{levin2017markov}. In particular, the expected hitting time of $i\ust$ scales as $1 / \mu^{(\zeta,x)}(i\ust)$, yielding the stated bound of order $1 / \mu_{\min}(\zeta)$. Unlike prior work with controlled Markov chains~\citep{chandak2022concentration}, we do not assume a uniform bound on the hitting time. This is necessary to analyze online Q-learning where hitting times may diverge as suboptimal actions are played with vanishing probability.

We are interested in the fixed point of the stationary average of the map $F(\cdot,\cdot)$. We define the stationary average of the function $F(\cdot,i)$ under the stationary distribution corresponding to $(\zeta,w)$, with $\zeta\in\cZ, w \in \cX$, as
\nal{
    \bar{F}^{(\zeta,w)}(\cdot)
    := \sum_{i \in \cY} \mu^{(\zeta,w)}(i)\, F(\cdot,i).
}
Then we have the following key contractive assumption. 
\begin{assumption}\label{assum:7}
For each $\zeta \in \cZ$ and $w \in \cX$, the mapping $\bar{F}^{(\zeta,w)} : \cX \to \cX$ is contractive with contraction factor $\alpha(\zeta) \in [0,1)$, i.e.,
\[
\bigl\|\bar{F}^{(\zeta,w)}(x) - \bar{F}^{(\zeta,w)}(z)\bigr\|
\;\le\; \alpha(\zeta)\,\|x - z\|,
\quad \forall x,z \in \cX,
\]
where $\|\cdot\|$ denotes any compatible norm on $\bR^d$. Moreover, the fixed point of $\bar{F}^{(\zeta,w)}$ does not depend on $(\zeta,w)$ and we denote this common fixed point by $x\ust$. 
\end{assumption}
The assumption that the stationary averaged map is contractive and admits a common fixed point is standard in the analysis of Q-learning; see, e.g., \citep{qu2020finite,chandak2022concentration}. In the case of Q-learning, this common fixed point is precisely the optimal Q-value function $Q\ust$ and the norm with respect to which the contraction property holds is the $\ell_\infty$ norm. The following assumption specifies how the contraction factor depends on the transition dynamics of the Markov chain.
\begin{assumption}\label{assum:at}
    Let $\at(\zeta) := 1 - \alpha(\zeta)$. We assume that $\at(\zeta) = \frakc_2\, \mu_{\min}(\zeta),$
for some constant $\frakc_2 > 0$. 
\end{assumption}
The above assumption states that the contraction factor  approaches one as $\mu_{\min}(\zeta)$ decays or equivalently as the mixing time explodes. The dependence of the contraction factor on the mixing time is often explicit in analyses of Q-learning and also appears implicitly in other algorithms, such as variants of TD learning~\citep{chandak2025concentration}. We incorporate this dependence directly into the SA framework to extend concentration bounds for SA to time-inhomogeneous Markov chains. Our next assumption characterizes the time-inhomogeneous nature of the Markov chain induced by the control process and specifies how the control process is allowed to vary over time. 
\begin{assumption}\label{assum:3} 
The control process $\{\zeta_n\}$ satisfies the following conditions:
\begin{itemize}
    \item[(a)] \textbf{Time-varying contraction factor.}
    \[ \at(\zeta_n) \ge \frac{\frakc_3}{(n+n_0)^{\kappa_1}}, \quad \frakc_3>0,\kappa_1 \in [0,1]. \]
    \item[(b)] \label{5b}\textbf{Temporal smoothness.} For all $i \in \cY$, $x \in \cX$, and $n \in \bN$, \[ \sum_{j\in\cY} |p^{(\zeta_n,x)}(i,j) - p^{(\zeta_{n+1},x)}(i,j)| \le L_1(n),\] such that $L_1(n)=\frakc_4/(n+n_0)^{1-\kappa_2}$, where $\frakc_4>0,\kappa_2 \in [0,1]$.
    \item[(c)] \textbf{Iterate sensitivity.} For each $\zeta \in\cZ$, the map $x\mapsto p^{(\zeta,x)}(i,j)$ is Lipschitz, i.e., 
	$$
		\sum_{j\in\cY} |p^{(\zeta,x)}(i,j) - p^{(\zeta,x')}(i,j)| \le L_2(\zeta) \|x-x'\|,\quad \forall i\in \cY, x,x' \in \cX
	$$
    Moreover,
    $
    L_2(\zeta_n) \le \frakc_{5} (n+n_0)^{\kappa_3}\log(n+n_0),
    $
    where $\frakc_{5}>0,~\kappa_3 \in [0,1)$.
\end{itemize}
\end{assumption}
We use the same $n_0$ in the stepsize definition and in Assumption~\ref{assum:3} for notational simplicity; these constants can be chosen independently without affecting the analysis. Condition (a) controls how fast the contraction factor approaches one and can be interpreted as a constraint on the mixing time's growth, since the previous assumptions imply $\mu_{\min}(\zeta_n) \ge \frakc_3/(\frakc_2(n+n_0)^{\kappa_1})$. Conditions (b) and (c) ensure that the transition kernel $p^{(\zeta_n,x_n)}$ varies slowly enough to allow effective averaging of the Markovian noise. For Q-learning, $\kappa_1$, $\kappa_2$, and $\kappa_3$ are determined by the exploration policy, and are central to the regret guarantees. While mixing time diverges under vanishing exploration, growth of the Lipschitz constant $L_2(\zeta)$ is also unavoidable for softmax exploration: as the temperature decays, the transition kernel becomes increasingly sensitive to the iterate (see Figure~\ref{fig:boltzmann_sensitivity}). 

Finally, we have the standard assumption that $\{M_{n+1}\}$ is a bounded martingale difference sequence.
\begin{assumption}\label{assum:6}
The process $\{M_{n+1}\}$, where $M_n\in\bR^d$, is an $\bR^d$-valued bounded martingale difference sequence with respect to the filtration $\{\cF_n\}_{n \ge 0}$, i.e., $\mathbb{E}[M_{n+1} \mid \cF_n] = \mathbf{0} \quad \text{a.s.}$ 
\end{assumption}

\subsection{Concentration bound}
We establish the following high-probability concentration bound on $\|x_n-x\ust\|$. 
\begin{theorem}\label{th:main_concentration}
For any $\delta > 0$, iteration~\eqref{def:gen_SA} satisfies the following with probability at least $1 - \frac{\pi^2 d}{6}\delta$:
$$
\|x_n - x\ust\|
\;\le\;
\tilde{O}\left(\sqrt{\frac{\log\!\left(n^2/\delta\right)}{n}}\;
n^{\,2\kappa_1 + \max\{\fraka + \kappa_3,\;\kappa_2\}}\right),$$
for all $n\in\bN$. Moreover, the iterates $x_n$ converge almost surely to the fixed point $x\ust$.
\end{theorem}
We make the dependence on $\delta$ explicit here to illustrate how it enters the bound; henceforth, we suppress this dependence as it appears in the same form in all subsequent results. The above result is a maximal or uniform concentration bound which holds for all time $n\in\bN$. We make a few remarks on the conditions under which the bound holds and its dependence on the hyperparameters.
\begin{remark}
The bound in Theorem~\ref{th:main_concentration} holds for sufficiently small stepsize (Appendix~\ref{sec:cond_n0_sa}) and for hyperparameter choices satisfying the conditions in Appendix \ref{app:assu_hyper}. These conditions are mild and are met by standard stepsize schedules and exploration policies used in online Q-learning. 
\end{remark}
\begin{remark}
The optimal rate is achieved for $\fraka = 0$, corresponding to the stepsize $\Theta(1/n)$, a phenomenon commonly observed in the SA literature~\citep{chen2020finite,qu2020finite}. The dependence on $\kappa_1,\kappa_2,\kappa_3$ is also intuitive: smaller $\kappa_1$ (slower increase of the contraction factor) and smaller $\kappa_2$ and $\kappa_3$ (slower variation of the transition kernel) lead to better rates. The resulting optimal convergence rate is
$\tilde{O}\left(\sqrt{\frac{\log(n^2/\delta)}{n}}\right)$
achieved when $\fraka\! =\! \kappa_1 \!=\! \kappa_2\! =\! \kappa_3 \!= 0$. This matches known concentration bounds for SA with time-homogeneous Markov chains~\citep{chen2025concentration}.
\end{remark}

At a high level, the key challenge in our analysis is handling time-varying transition kernels together with a growing contraction factor that approaches one. To address the resulting time-inhomogeneous Markovian noise, we use solutions of the Poisson equation and develop a novel analysis of how these solutions vary as the transition kernel evolves with the iterates and the control process. This is combined with a noise-averaging argument that yields sharp bounds on the decay of the noise despite these variations and allows us to decouple the stochastic fluctuations from the weakening contraction. 

 \begin{figure}[t]
   \centering
     \begin{subfigure}[b]{0.32\linewidth}
         \centering
         \includegraphics[width=\linewidth]{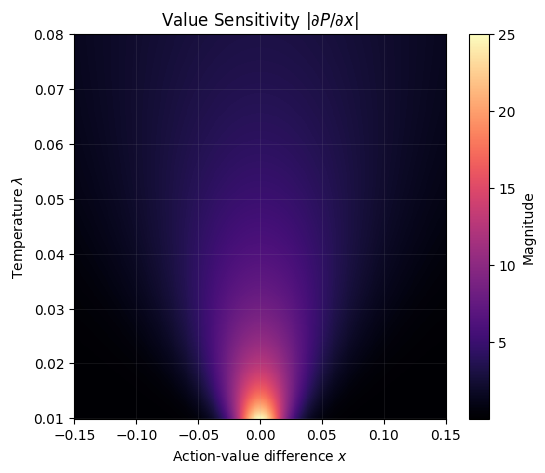}
         \caption{Value Sensitivity $|\frac{\partial P}{\partial x}|$}
         \label{fig:val_sens}
     \end{subfigure}
     \hspace{25pt}
     \begin{subfigure}[b]{0.32\linewidth}
         \centering
         \includegraphics[width=\linewidth]{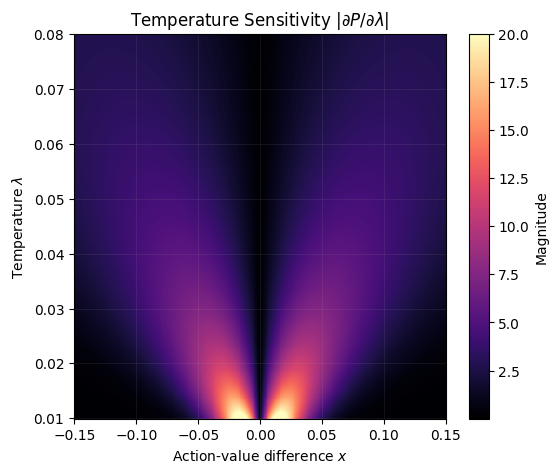}
         \caption{Temp. Sensitivity $|\frac{\partial P}{\partial \lambda}|$}
         \label{fig:temp_sens}
     \end{subfigure}
     \caption{
     Sensitivity of the Boltzmann exploration rule:  heatmaps for a two-action policy, where the probability of selecting the first action is $P(x,\lambda) = (1+e^{-x/\lambda})^{-1}$. As $\lambda$ decreases, sensitivity near the decision boundary increases, illustrating the growth of the Lipschitz constant.}
     \label{fig:boltzmann_sensitivity}
\end{figure}

%% file: q-learning.tex
\section{Online Q-learning}\label{sec:q_learning}
Consider a Markov Decision Process (MDP) $\cM$ with finite state and action spaces $\cS$ and $\cA$, respectively. The controlled transition probabilities are given by $p : \cS \times \cA \times \cS \to [0,1]$, where
$\bP(s_{n+1} = s' \mid s_n = s, a_n = a) = p(s,a,s'),$  for $s,s' \in \cS,\, a \in \cA$. Let $\gamma \in [0,1)$ be the discount factor, and let $r : \cS \times \cA \to [0, R_{\max}]$ denote the reward function. The objective is to choose actions $\{a_n\}_{n \ge 0}$ so as to maximize the expected discounted return
$\bE\!\left[\sum_{n=0}^{\infty} \gamma^n r(s_n,a_n)\right]$, where $s_n$ and $a_n$ denote the state and action at time $n$, respectively. We denote the MDP by $\cM = (\cS,\cA,p,r,\gamma)$.

Let $V\ust : \cS \to \bR$ denote the optimal value function,
\[
V\ust(s) = \sup_{\pi} \bE_\pi\!\left[\sum_{n=0}^{\infty} \gamma^n r(s_n,a_n) \mid s_0 = s \right],
\]
where the supremum is taken over all stationary (possibly randomized) policies. The corresponding optimal action-value function $Q\ust : \cS \times \cA \to \bR$ is defined as
\[
Q\ust(s,a) = r(s,a) + \gamma \sum_{s' \in \cS} p(s,a,s') V\ust(s').
\]
Equivalently, $Q\ust$ is the unique fixed point of the Bellman optimality operator $\cT : \bR^{|\cS||\cA|} \to \bR^{|\cS||\cA|}$:
\[
(\cT Q)(s,a) = r(s,a) + \gamma \sum_{s' \in \cS} p(s,a,s') \max_{a' \in \cA} Q(s',a').
\]
The optimal policy can be characterized directly in terms of $Q\ust$. In particular, actions that maximize the expected discounted return in state $s$ are given by the maximizers of $Q\ust(s,\cdot)$. We define the set of optimal actions in state $s$ as
$\cA\ust(s) := \left\{ a \in \cA : Q\ust(s,a) = \max_{a' \in \cA} Q\ust(s,a') \right\}.$
Any stationary policy $\pi\ust$ that selects actions from $\cA\ust(s)$ with probability one for each state $s$ is optimal. Consequently, the problem of computing an optimal policy reduces to finding the optimal action-value function $Q\ust$. If the reward function $r$ and transition kernel $p$ were known, $Q\ust$ could be computed using dynamic programming. In contrast, when the MDP is accessed through a single trajectory of the controlled Markov chain, the Q-learning algorithm uses the following Q-value updates,
\al{\label{iter:Q-learn}
	Q_{n+1}(s,a)=Q_n(s,a)+\beta_n \ind{s_n\!=\!s,a_n\!=\!a} \left(r(s,a)+\gamma \max_{a'} Q_n(s_{n+1},a')-Q_n(s,a)\right),
}
for all $(s,a)\in \cS\times \cA$ and for all $n\in\bN$. The indicator $\ind{s_n=s,a_n=a}$ arises from the asynchronous nature of Q-learning, since at each time step only the action-value corresponding to the currently visited state–action pair $(s_n,a_n)$ is updated, while all other components remain unchanged.

This asynchronous update can be studied in both offline and online settings: in the offline setting actions are generated according to a fixed behavior policy, while in the online setting considered here, actions are selected using the current Q-values. In online Q-learning, both convergence to $Q\ust$ and cumulative reward are of interest. This setting is relevant in applications where data are generated sequentially and poor decisions incur real costs, such as autonomous driving and healthcare  support.

\subsection{Regret in discounted Q-learning}
We next define the regret of an online learning algorithm in the discounted MDP setting.
\begin{definition}
The regret incurred by an online learning algorithm up to time $N\in\bN$ is defined as
\nal{
	 \cR_N \coloneqq  (1-\gamma) \sum_{n=1}^{N}\br{V\ust(s_{n}) -  \bE\left[\sum_{m=n}^{\infty} \gamma^{m-n} r(s_m,a_m)\; \middle|\; \cF_{n} \right]  }.
}    
\end{definition}
This definition is standard in the literature; see, e.g., Definition~4.1 of~\citep{he2021nearly}. We establish bounds on $\cR_N$ that hold with high probability over the randomness of the learning trajectory. At each time $n$, the regret expression compares the optimal discounted value $V\ust(s_n)$ with the expected discounted value obtained by the algorithm from that time onward and aggregates this loss over time. The normalization factor $(1-\gamma)$ converts this discounted future loss into a quantity comparable to a single time-step reward, so that each term in the sum represents the loss incurred at time $n$.

This regret definition highlights a subtle tension in online Q-learning. On the one hand, the algorithm must quickly reduce the probability of suboptimal actions, as they directly incur regret. On the other hand, changing the policy parameters too aggressively can harm exploration and, more subtly, cause the induced policy to drift away to a possibly worse policy. Balancing rapid exploitation, sustained exploration, and controlled policy drift is a key challenge in the regret analysis.
\vspace{-5pt}


%% file: q-learning_performance.tex
\section{Performance of Boltzmann and \seg~Q-learning}\label{sec:performance}
\vspace{-5pt}
Through sample complexity and regret guarantees, we analyze Boltzmann Q-learning and \seg~Q-learning, which share the update rule \eqref{iter:Q-learn} and differ only in their exploration policies.
\subsection{Boltzmann Q-learning}\label{sec:BQL}
Boltzmann Q-learning selects actions according to a softmax (Boltzmann) distribution over the action space. At each time step, actions with higher estimated Q-values are selected with higher probability, while suboptimal actions continue to be explored depending on the temperature parameter.~This exploration strategy provides a smooth tradeoff between exploration and exploitation and has been widely used in reinforcement learning due to its simplicity.

Let $\sigma: \bR^{|\cA|} \mapsto \Delta(\cA)$ denote the softmax function~\citep{sutton1998reinforcement}. Then in Boltzmann Q-learning, the actions are sampled at time $n$ according to the distribution $\sigma\br{Q_n(s_n,\cdot)\slash \lambda_n}$, i.e., the action $a$ is chosen at state $s_n$ with probability proportional to $\exp\left(Q_n(s_n,a)/\lambda_n\right)$. Here $\{\lambda_n\}$ denotes the temperature sequence or schedule. We choose the temperature schedule to be of the form
	\nal{
		\lambda_n = \frac{1}{\frakb \log(n+n_0)},~\frakb >0,n_0,n\in\bN.
	}
This choice is tuned to optimize the regret bound in our analysis, and logarithmic temperature schedules of this form are known to be optimal in the stochastic multi-armed bandit setting \citep{cesa2017boltzmann}. The complete algorithm is presented in Algorithm~\ref{algo:Q_learning}. The parameter $\frakb$ controls exploration, with larger values leading to faster decay of the temperature and hence reduced exploration. We first present the sample complexity result obtained for $\fraka=0$.
\begin{theorem}\label{th:Q_conc}
	For any $\delta>0$, the iterates $\{Q_n\}$ of the Boltzmann Q-learning algorithm (Algorithm \ref{algo:Q_learning}) satisfy the following with probability at least $1-  \frac{\pi^2 d}{6}\delta$:
	\nal{
		& \|Q_n - Q\ust\|_{\infty}  \le \tilde{O}\left(\frac{n^{-\left(0.5-3 \kappa\s_1 \right)  }}{(1-\gamma)^2}\right),~\forall n\in\bN,
	}
where $\kappa\s_1=\frakb R_{\max}/(1-\gamma)$, $0<\kappa\s_1 < \frac{1}{6}$. In addition, $Q_n \to Q\ust$ almost surely.
\end{theorem}

The above result is obtained by casting the Boltzmann Q-learning algorithm into our Markovian SA framework, where the iterate $x_n$ is given by $Q_n$ and the Markov chain $y_n$ is given by $(s_n,a_n)$. Moreover, the corresponding SA parameters are $\kappa\s_1=\frac{\frakb R_{\max}}{1-\gamma}$ and $\kappa\s_2=\kappa\s_3=0$. The following remark discusses the dependence of sample complexity on $\frakb$ and $\gamma$.

\begin{remark}
The convergence rate of Boltzmann Q-learning improves as $\frakb$ decreases, and can be made arbitrarily close to $\tilde{O}(1\slash \sqrt{n})$ by choosing  $\frakb$ sufficiently small. In the limit $\frakb \downarrow 0$, the policy approaches uniform action sampling, which may be viewed as a special case; we do not emphasize this regime here since it is subsumed by the \seg~exploration policy. The resulting sample complexity is
$\tilde{O}\!\left( \left((1-\gamma)^{-4} \eps^{-2}\right)^{1/(1-6\kappa\s_1)}\right)$, approaching $\tilde{O}\!\left((1-\gamma)^{-4} \eps^{-2}\right)$ as $\frakb \to 0$. This is the first sample complexity result for classical Boltzmann Q-learning, to the best of our knowledge.
\end{remark}


\begin{figure}[t]
\centering
\begin{minipage}{0.48\textwidth}
\begin{algorithm}[H]
\caption{Boltzmann Q-learning}
\label{algo:Q_learning}
\begin{algorithmic}[1]
\REQUIRE stepsize sequence $\beta_n$, discount factor $\gamma\! \in\! (0,1)$, temperature $\lambda_n\! =\frac{1}{\frakb \log(n+n_0)}\! $
\STATE Initialize $Q_0(s,a)$ for all $(s,a)$
\STATE Observe initial state $s_1$
\FOR{$n = 1,2, \ldots$}
    \STATE Sample action $a_n \sim \sigma(Q_n(s_n,\cdot)/\lambda_n)$
    \STATE Execute $a_n$, observe $r(s_n,a_n)$ and $s_{n+1}$
    \STATE Update only $Q_n(s_n,a_n)$ according to~\eqref{iter:Q-learn}
\ENDFOR
\end{algorithmic}
\end{algorithm}
\end{minipage}
\hfill
\begin{minipage}{0.48\textwidth}
\begin{algorithm}[H]
\caption{\seg~Q-learning}
\label{algo:eps-softmax}
\begin{algorithmic}[1]
\REQUIRE stepsize sequence $\beta_n$, discount factor $\gamma \in (0,1)$, temperature $\lambda_n = 1/(n+n_0)^{\frake}$, exploration $\eps_n = 1/(n+n_0)^{\frakd}$
\STATE Initialize $Q_0(s,a)$ for all $(s,a)$
\STATE Observe initial state $s_1$
\FOR{$n = 1,2, \ldots$}
    \STATE Sample action $a_n \sim f^{(\eps_n,\lambda_n,Q_n)}(s_n,\cdot)$
    \STATE Execute $a_n$, observe $r(s_n,a_n)$ and $s_{n+1}$
    \STATE Update only $Q_n(s_n,a_n)$ according to~\eqref{iter:Q-learn}
\ENDFOR
\end{algorithmic}
\end{algorithm}
\end{minipage}
\end{figure}


Our regret guarantees depend on the suboptimality gap of the MDP, defined as 
\nal{
gap(\cM) := \min_{s\in\cS} \min_{a\in\cA\setminus \cA\ust(s)} \left(V\ust(s) -  Q\ust(s,a)  \right).
}
Intuitively, $gap(\cM)$ is the smallest difference between the value of an optimal action and a suboptimal action, over all states. Smaller gaps correspond to harder problems, as optimal and suboptimal actions are harder to distinguish. We now present our regret bound. Due to space constraints, we report only the optimal bound, obtained by setting $\fraka=\kappa\s_1=\frakb R_{max}/(1-\gamma)$ and choosing $\fraka$ to be sufficiently small. Regret guarantees for general hyperparameters are deferred to Appendix~\ref{app:pf_q_learning}.

\begin{theorem}\label{th:regret_bound}
For any $\delta>0$, the regret incurred by the Boltzmann Q-learning algorithm satisfies the following with probability at least $1- \delta \frac{\pi^2 d}{6}$:
\nal{
\cR_N \le \frac{\tilde{O}(N^{.5 + 4\fraka })}{(1-\gamma)^{3.5}} + \tilde{O}\br{ N^{1- \frac{\frakb gap(\cM)}{2}}},\quad \forall N\in\bN,
}
where $\fraka$ must satisfy $\fraka \le \frac{1}{8}$.
\end{theorem}
We emphasize that our regret guarantees are uniform in time, i.e., they hold for all $N$ and do not require prior knowledge of the horizon. The first term reflects how quickly the iterates $Q_n$ concentrate around $Q\ust$ and improves with increased exploration. The second term captures the regret due to suboptimal action selection and depends on $gap(\cM)$. 

Even with knowledge of $gap(\cM)$, the regret can be near-linear when $gap(\cM)$ is small. This is because $\frakb$ must be chosen small enough to ensure exploration, which slows down exploitation. This behavior is consistent with stochastic bandits, where it has been shown that  Boltzmann policy with any monotone temperature schedule yields suboptimal performance~\citep{cesa2017boltzmann}.
\vspace{-3pt}

\subsection{\seg~Q-learning}
\vspace{-3pt}
The Smoothed $\eps_n$-Greedy (\seg) Q-learning algorithm selects actions according to a convex combination of the uniform policy and a softmax policy. The softmax component is for purely technical reasons and ensures smoothness of the transition kernel with respect to the Q-values. This smoothing does not change the behavior of the algorithm, which remains effectively $\epsilon_n$-greedy. Formally, in \seg~Q-learning the action $a$ at time $n$ is sampled from the distribution $f^{(\eps_n,\lambda_n,Q_n)}(s_n,\cdot)$, where
\al{
f^{(\eps,\lambda,Q)}(s,a) := \eps/|\cA| + (1-\eps)\,\sigma\!\left(Q(s,a)/\lambda\right), \label{def:f_e_l_q}
}
where $\eps \in [0,1]$ and $\lambda > 0$. Equivalently, the policy can be viewed as a two-stage randomization: an action is drawn from a uniform distribution over actions with probability $\eps_n$, and from a Boltzmann distribution with temperature $\lambda_n$ with probability $1-\eps_n$. When $\lambda_n$ decays sufficiently fast, the Boltzmann component concentrates on greedy actions, so the policy closely approximates standard $\epsilon_n$-greedy exploration. Accordingly, for $\frakd,\frake \in [0,1]$, we consider sequences of the following form.
\nal{
\eps_n = \frac{1}{(n+n_0)^{\frakd}}, \quad \lambda_n = \frac{1}{(n+n_0)^{\frake}}, \qquad n_0,n \in \bN.
}

We first present the optimal convergence rate obtained for $\fraka=\frakd=0$, which corresponds to a uniform policy over $\cA$. The general bound for arbitrary choices of $(\fraka,\frakd, \frake)$ is deferred to Appendix~\ref{app:eps_softmax}.
\begin{theorem}\label{th:conc_q_eps_softmax}
For any $\delta>0$, the iterates of the \seg~Q-learning algorithm (Algorithm \ref{algo:eps-softmax}) with $\fraka=\frakd=0$ satisfy the following with probability at least $1- \delta \frac{\pi^2 d}{6}$:
\nal{
\|Q_n - Q\ust \|_{\infty} &\le \frac{1}{(1-\gamma)^{2.5}}\tilde{O}\br{\frac{1}{\sqrt{n}} },~\forall n\in\bN.
}
In addition, $Q_n \to Q\ust$ almost surely. 
\end{theorem}
The resulting sample complexity $\tilde{O}\!\left((1-\gamma)^{-5}\eps^{-2}\right)$ matches prior SA-based bounds~\citep{qu2020finite,chen2025concentration}. Hence, our result recovers known sample complexity bounds for offline Q-learning obtained via SA techniques, while extending them to the online setting.

We now quantify the regret of \seg~Q-learning. We present the regret bound for the choice $\fraka=0.1, \frakd=0.1 - \frake$, and $\frake<0.05$ which yields the best rate within our analysis, and defer the general result to Appendix~\ref{app:eps_softmax}.
\begin{theorem}\label{th:regret_bound_eps_softmax}
For any $\delta>0$, the regret incurred by the \seg~Q-learning algorithm with $\fraka= 0.1, \frakd=0.1 - \frake$ and $\frake<0.05$, satisfies the following with probability at least $1- \delta \frac{\pi^2 d}{6}$:
\nal{
\cR_N \le \tilde{O}\left(\frac{N^{\frac{9}{10}+2\frake}}{(1-\gamma)^{3}}\right),\quad \forall N\in\bN.
}
\end{theorem}
The bound implies that the regret can be made arbitrarily close to $\tilde{O}(N^{9/10})$ by choosing $\frake$ sufficiently small. Smaller $\frake$ corresponds to slower temperature decay and hence increased exploration, which improves the regret rate. To our knowledge, this is the first sublinear regret guarantee for classical Q-learning. $\epsilon$-greedy–type algorithms are weaker than optimism-based methods such as UCB. This separation is well documented in stochastic bandits, where $\epsilon$-greedy incurs $O(N^{2/3})$ regret, while optimism-based algorithms achieve the optimal $O(\sqrt{N})$ rate \citep{panageas2020multiarmedbandits}.
\vspace{-6pt}

\section{Conclusion}\label{sec:conclusion}
\vspace{-5pt}
In this work, we establish sample complexity and regret guarantees for online Q-learning without relying on optimism or variance reduction. Our analysis covers Boltzmann and Smoothed $\epsilon_n$-Greedy exploration schemes and yields the first sublinear regret guarantees for unmodified online Q-learning. A key technical contribution is a novel high-probability bound for SA with controlled, time-inhomogeneous Markovian noise, which allows the mixing time to diverge and the contraction factor to approach one. 

\textbf{Limitations and Future Directions.} We highlight multiple future directions of this work: studying whether the  regret can be improved to the $\tilde{O}(N^{2/3})$ rates obtained for $\epsilon_n$-greedy exploration in bandits, or whether a fundamental separation exists. Second, an important extension of this work would be to allow for Q-learning with function approximation which will greatly increase the scope of this work. Another interesting direction is extending our bounds to average reward RL.

%% file: appendix_organization.tex
\section{Organization of the appendix}
Appendix \ref{app:assu_hyper} formalizes additional assumptions, introduces notation, and states the required conditions on the hyperparameters. Appendix~\ref{sec:proof_outline} provides a proof outline for the general concentration bound. This includes the concentration result for the general stochastic approximation iterates $\{x_n\}$ for the general stochastic approximation. Appendix~\ref{sec:app_pf_sec_SA} proves the main results for the general stochastic approximation framework, while Appendix~\ref{sec:aux_SA} establishes auxiliary results used in the proof of Appendix~\ref{sec:app_pf_sec_SA}. Similarly, Appendix~\ref{app:pf_q_learning} proves the main results for Boltzmann Q-learning, while Appendix~\ref{app:aux_q-learning} contains auxiliary results used in Appendix~\ref{app:pf_q_learning}.~Appendices~\ref{app:eps_softmax} and~\ref{sec:eps_softmax_auxiliary} present the main and auxiliary results, respectively, for \seg~Q-learning. Finally, Appendix~\ref{app:miscellaneous} collects miscellaneous technical results that are used at various points in Appendices \ref{sec:app_pf_sec_SA}-\ref{sec:eps_softmax_auxiliary}.

%% file: proof_sketch.tex
\section{Additional assumptions, notation, and conditions on the hyperparameters}\label{app:assu_hyper}
In this appendix, we first formalize some assumptions used in the stochastic approximation formulation, and introduce the corresponding notation. These assumptions are standard in the literature on stochastic approximation and Q-learning, but were omitted from Section \ref{sec:SA} due to space constraints. We then state the conditions imposed on the hyperparameters, including the stepsize parameters and $\kappa_1,\kappa_2,\kappa_3$.

\subsection{Assumptions}
The first assumption states that the function $F(x,y):\cX\times\cY\mapsto\cX$ is Lipschitz in $x$.
\begin{assumption}	\label{assum:8} 
	The maps $x\mapsto [F(x,i)](\ell), \ell = 1,2,\ldots,d, i\in\cY$ are Lipschitz, i.e., there exists a $\frakc_6 >0$ such that
	\nal{
		|[F(x,i)](\ell)-[F(x',i)](\ell)| & \le \frakc_{6} \|x-x'\|,       
	}   
for all $\ell \in \{1,2,\dots,d\}, x,x'\in \cX,i\in \cY$.     
\end{assumption}
The next assumption states that the iterates $\{x_n\}$ and the noise sequence $\{M_{n}\}$ are bounded.
\begin{assumption}\label{assum:9} 
	$\{M_n\}$ satisfies $|M_n(\ell)| \le \frakc_7~a.s.,~\forall \ell \in \{1,2,\ldots,d\}$.	The process $\{x_n\}$ satisfies $x_n \in \cX \subset \bR^{d}$ with $\cX$ a compact set and hence $\|x_n\|  \le \frakc_8~a.s$.
\end{assumption}
Since $\cY$ is finite and, by Assumption~\ref{assum:8}, the map
$x\mapsto F(x,i)$ is Lipschitz on the compact set $\cX$ for each
$i\in\cY$, the update map $F$ is uniformly bounded on $\cX\times\cY$.
Hence, we may define
\al{
\frakc_9
:=
\sup_{x\in\cX,\ i\in\cY}\|F(x,i)\|
<\infty.\label{def:c9}
}
Therefore,
\[
\|F(x_n,y_n)\|\le \frakc_9,
\qquad
\forall n\in\mathbb{N},
\qquad
\text{a.s.}
\]

For discounted Q-learning, the iterates $\{Q_n\}$ are bounded by $\frac{R_{\max}}{1-\gamma}$, provided that the initialization satisfies the same bound.~Finally, we have the following assumption on the stationary distribution.
\begin{assumption}\label{assum:14}
	\nal{
		\|\mu^{(\zeta,x)} -  \mu^{(\zeta',x')}\| 
		\le  \frac{\frakc_{10}  \| p^{(\zeta,x)} - p^{(\zeta',x')} \|}{\mu_{\min}(\zeta)},
	}   
	where $\frakc_{10} >0$. 
\end{assumption}
Bounds of this form are standard in the sensitivity analysis of stationary distributions of finite-state Markov chains; see, e.g., \citep{cho2000markov}, where such results are typically stated in the $\ell_\infty$ norm. Since all norms are equivalent in finite-dimensional spaces, these bounds can be translated to the norm $\|\cdot\|$ used here at the cost of a dimension-dependent constant, which is absorbed into $\frakc_{10}$. Moreover, this assumption is not fundamental and can be derived under irreducibility and finiteness of the state space. We state it explicitly for clarity and to streamline the presentation of the concentration analysis.

\subsection{Notation}
For constants associated with the Boltzmann Q-learning we will use the superscript $B$.~For those associated with Smoothed $\eps_n$-Greedy Q-learning, we use the superscript $S\epsilon G$.~For example, we write $\kappa\s_1,\kappa\es, \kappa\s_{2}, \kappa\es_{2}$. Moreover, we define the following functions, which will be used throughout the proofs:

\nal{
 g(n,\delta) :&= \sqrt{\frac{\log\br{\frac{n^2}{\delta}}}{(n+n_0)^{1-(\fraka+2\kappa_1)}}},\\
 g_1(n) : & = \frac{\log(n+n_0)}{\br{n+n_0}^{1-(\fraka + 2\kappa_1+\kappa_3)}}\\
 g_2(n) :&= \frac{1}{(n+n_0)^{1-(2\kappa_1 +\kappa_2)}}, \mbox{ where } n\in \bN, \delta >0.
}
If $p$ denotes the one-step transition probabilities of a Markov chain, then we use $p^{n}$ to denote the corresponding $n$-step transition probabilities. For example, in Boltzmann Q-learning and \seg~Q-learning, if $p^{(\eps,\lambda)}$ and $p^{(\eps,\lambda,Q)}$ denote the transition probabilities induced by the policies $\sigma(Q\slash \lambda)$ and $f^{(\eps,\lambda,Q)}$, respectively, then $\br{p^{(\eps,\lambda)}}^{n}$ and $\br{p^{(\eps,\lambda,Q)}}^{n}$ denote the corresponding $n$-step transition probabilities. If $\cE$ is an event, then $\ind{\cE}$ denotes its indicator random variable.  

\subsection{Conditions on the hyperparameters}
We now state the conditions imposed on the hyperparameters $\fraka,\kappa_1,\kappa_2,\kappa_3, n_0,\beta$. These conditions are required for the analysis and are assumed to hold throughout.

\begin{condition}\label{con:sa1}
The quantities $\fraka,\kappa_1,\kappa_2,\kappa_3 \in [0,1]$ satisfy 

\nal{
2\fraka + 6\kappa_1+3\kappa_3<1, \fraka +6\kappa_1+\kappa_2+2\kappa_3 <1.
}
In addition, the following case-specific conditions are imposed according to whether $\fraka =0$ or $\fraka >0$.

Case A: $\fraka = 0$. The following conditions hold,
\begin{enumerate}[(i)]
    \item $\beta \ge 2(1-2\kappa_1)$
    \item $2\kappa_1+\kappa_3 < 1$.
\end{enumerate}

Case B: $\fraka >0$. The following conditions hold,

\begin{enumerate}[(i)]
    \item $\fraka + 2\kappa_1 <1$
    \item $\fraka + 2\kappa_1+\kappa_3<1$ 
    \item $n_0 \ge \br{ \frac{ 2 \left[ (1-\fraka)-2\kappa_1\right]  }{\beta}   }^{\frac{1}{1-\rho}}$, $n_0 \ge \br{\frac{2\br{ 1-\fraka -\kappa_1 }}{\beta}}^{1\slash \fraka }$, $n_0 \ge \br{\frac{2\br{1-\br{\fraka + 2\kappa_1+\kappa_3}}}{\beta}}^{\frac{1}{\fraka}}$,\linebreak $n_0 \ge \br{\frac{1-\br{2\kappa_1 + \kappa_2}}{\beta}}^{1\slash \fraka}$.
\end{enumerate}

We further impose the following conditions depending on the relation between $\fraka$ and $\kappa_1$.

Case C:\label{condition_CaseC} $\fraka = \kappa_1$. The following hold,
\begin{enumerate}[(i)]
    \item $3\fraka + 4\kappa_1 + 2\kappa_3 < 1$, $\beta \frakc_3 \ge 1- (3\fraka + 4\kappa_1 + 2\kappa_3)$
    \item $\beta \frakc_3 \ge 2-\br{4\fraka + 6\kappa_1+4\kappa_3}$
    \item $\fraka + 3\kappa_1+\kappa_2 + \kappa_3 < 1$, $\beta \frakc_3 \ge 2-\br{ 2\fraka + 6\kappa_1+2\kappa_2 + 2\kappa_3}$.
\end{enumerate}

Case D: $\fraka > \kappa_1$. The following hold,
\begin{enumerate}[(i)]
    \item $\fraka+6\kappa_1 +2\kappa_3 < 1$
    \item $4\kappa_1+\kappa_2+\kappa_3 <1$
    \item $n_0 \ge \br{\frac{1-\br{\fraka + 6\kappa_1 + 2\kappa_3} }{\beta \frakc_3}}^{1\slash (\fraka - \kappa_1)}$, $n_0 \ge \br{\frac{2\br{1-\br{\fraka + 4\kappa_1+2\kappa_3}}}{\beta \frakc_3}}^{1\slash (\fraka - \kappa_1)}$, \linebreak $n_0 \ge \br{\frac{2\br{1-\br{4\kappa_1+\kappa_2 + \kappa_3}}}{\beta \frakc_3}}^{1\slash (\fraka - \kappa_1)}$.   \label{condition_caseD_iii}
\end{enumerate}
\end{condition}

\section{Outline of proofs}\label{sec:proof_outline}
In this section we provide a high-level overview of the proofs of the concentration bounds for general SA presented in Section~\ref{sec:SA}.~Detailed derivations are provided in Appendix~\ref{sec:app_pf_sec_SA} and Appendix~\ref{sec:aux_SA}.

\subsection{Concentration of $\{x_n\}$}\label{sec:concentration_xn}
\textit{Proof sketch}: We begin by rewriting the iterations~\eqref{def:gen_SA} as,
\nal{
x_{n+1} = x_n + \beta_n \br{\bar{F}^{(\zeta_n,x_n)}(x_n)-x_n + M_{n+1}+ F(x_n,y_n) - \bar{F}^{(\zeta_n,x_n)}(x_n)}.
}
Define 
\al{
w_n := M_{n+1} + F(x_n,y_n) - \bar{F}^{(\zeta_n,x_n)}(x_n),~n\in\bN. \label{def:noise_eff}
}
$w_n$ is seen to be the ``effective noise'' at time $n$. With this notation, the recursion for $x_n$ becomes,
\al{
x_{n+1} = x_n + \beta_n\br{\bar{F}^{(\zeta_n,x_n)}(x_n)-x_n + w_n}, n\in\bN.\label{def:compact_recur_x}
}
Next, we define an auxiliary process $\{z_n\}$ that evolves as follows,
\al{
z_{n+1}  = z_n + \beta_n \br{\bar{F}^{(\zeta_n,x_n)}(x_n)-z_n },~n\in \bN,\label{def:zn_1}
}
with $z_0 = x_0$. The sequence $\{z_n\}$ is the averaged noise sequence, and is the key to obtaining our concentration result. We want to derive an upper-bound on $\|x_n-x\ust\|$.~We will separately upper-bound the quantities $\|x_n-z_n\|$ and $\|z_n-x\ust\|$ in Sections~\ref{sec:z-x} and~\ref{sec:z-xstar} respectively. Since $\|x_n-x\ust\| \le \|x_n-z_n\|+\|z_n-x\ust\|$, these two bounds are then combined to obtain the desired concentration bound for $\|x_n-x\ust\|$.
\subsubsection{Bounding $\|x_n - z_n\|$}\label{sec:z-x}
\textit{Proofsketch}: We first derive an expression for $x_n-z_n$ by converting the Markov noise $F(x_n,y_n)- \bar{F}^{(\zeta_n,x_n)}(x_n)$ to martingale noise through the Poisson equation.~The resulting decomposition consists of two summation terms, labeled~\eqref{x_z_eq_1} and~\eqref{x_z_eq_2}. The term~\eqref{x_z_eq_1} is a martingale and is controlled with high probability on a ``good set'' $\cG$. The term~\eqref{x_z_eq_2} is further decomposed in Appendix~\ref{pf:prop3}, and its components are bounded separately.~These bounds are then combined in Proposition~\ref{prop:z_x_bound} to obtain the desired bound on $\|x_n-z_n\|$.

Unlike the recursion for $x_n$, the recursion for $z_n$ involves averaging with respect to the stationary distribution $\mu^{(\zeta_n,x_n)}$. Thus, the Markov noise is averaged out in the $z_n$ recursion. To compare $x_n$ and $z_n$, we convert the Markov noise in the recursion for $\{x_n\}$ into martingale noise via the \textit{Poisson trick}. This is a standard tool in the analysis of stochastic approximation with Markov noise, going back to~\citep{metivier1984applications}. We begin by defining the Poisson equation~\citep{meyn2012markov}. The following can be found in~\citep{chandak2022concentration}.
\begin{definition}(Poisson Equation)
	Consider a Markov chain on $\cY$ with transition probability $p^{(\zeta,x)}$, where $\zeta\in\cZ, x\in\cX$. Associated with this is the Poisson equation, a system of $|\cY|$ linear equations parameterized by $(\zeta,x)$ and solved for $H(i), i\in\cY$,
	\al{
		H(i)  = F(x,i) - \sum_{j \in\cY} \mu^{(\zeta,x)}(j)F(x,j)  + \sum_{j \in\cY} p^{(\zeta,x)}(i,j)H(j),~i\in \cY.
		\label{def:Poisson_1}
	}
Consider the Markov chain $\{y_n\}$ with transition probabilities given by $p^{(\zeta,x)}$.~Let $\tau$ denote the hitting time of the designated state $i\ust$ from Assumption~\ref{assum:2}.~For any state $i\in\mathcal{Y}$ define, 
		\nal{
			H^{(\zeta,x)}(i) : = \bE\br{\sum_{n=0}^{\tau} F(x,y_n) \Big| y_0 = s  } -  \bE\br{\sum_{n=0}^{\tau} F(x,y_n) \Big| y_0 = i\ust}.
		}
Denote $H^{(\zeta,x)} = \left\{H^{(\zeta,x)}(i)\right\}_{i\in \cY}$, where $H^{(\zeta,x)}(i) = \br{[H^{(\zeta,x)}(i)](1),[H^{(\zeta,x)}(i)](2),\ldots,[H^{(\zeta,x)}(i)](d)}$.~By the hitting-time representation of the solution to the Poisson equation, $H^{(\zeta,x)}$ solves the Poisson equation~\eqref{def:Poisson_1}~\citep{borkar_markov_book,chandak2022concentration}.
\end{definition}
Several properties of the solution $H^{(\zeta,x)}$ that are useful for the concentration proof are established in Lemma~\ref{lemma:sensitivity_v_x}.~For $m,n \in \bN$, $m\le n$ define 
\nal{
\chi(m,n) := 
\begin{cases}
    \Pi_{j=m}^{n} \br{1-\beta_j}, \mbox{ if } m < n,\\
    1 \mbox{ if } m = n.
\end{cases}
}
Next, we derive an expression for the difference $x_n - z_n$, in which we convert the Markov noise to martingale noise by using the Poisson equation.
\begin{lemma}\label{lemma:decompose_w}
Define
	\al{
		M'_{n+1} : = H^{(\zeta_n,x_n)}(y_{n+1})- \sum_{i\in\cY} p^{(\zeta_n,x_{n})}(y_{n},i)H^{(\zeta_n,x_n)}(i), n\in \bN.\label{def:M_n_prime}
	}
$\{M'_{n}\}$ is a martingale difference sequence w.r.t. the filtration generated by $\{(x_n,y_n)\}$.~Moreover,
\begin{subequations}
	\al{
		& x_n - z_n  = \sum_{i=0}^{n-1} \beta_i \chi(i+1,n) \br{M_{i+1}+M'_{i+1} }\label{x_z_eq_1}\\
		& + \sum_{i=0}^{n-1} \beta_i \chi(i+1,n) \br{H^{(\zeta_i,x_i)}(y_i) - H^{(\zeta_i,x_i)}(y_{i+1})}.\label{x_z_eq_2}
	}
\end{subequations}
\end{lemma}
Next, we will analyze the two terms~\eqref{x_z_eq_1} and~\eqref{x_z_eq_2} which appear in the decomposition of $x_n - z_n$ separately. The following result, proved in Appendix~\ref{pf:martingale_conc}, provides a high-probability bound for~\eqref{x_z_eq_1}.
\begin{lemma}\label{lemma:bound_m+mtilde}
Define $\cG$ to be the set on which the following holds: 
\al{
\left\|\sum_{i=0}^{n-1} \beta_i \chi(i+1,n) \br{M_{i+1}+M'_{i+1} }\right\|< \frac{16d\frakc_9 \frakc_1 \frakc_2\sqrt{\beta}}{\frakc_3}~g(n,\delta), \forall n \in \bN.\label{def:G}
}
Then 
\nal{
\bP(\cG^c) \le \delta \frac{\pi^2 d}{6}.
}
\end{lemma}
The term~\eqref{x_z_eq_1} is therefore controlled on the event $\cG$. We next bound the term~\eqref{x_z_eq_2} and combine this bound with Lemma~\ref{lemma:bound_m+mtilde}. This yields the main result of this section, which is proved in Appendix~\ref{pf:prop3}.
\begin{proposition}\label{prop:z_x_bound}
Consider the event $\cG$ defined in~\eqref{def:G}, that was shown to have a probability greater than $1-\delta\frac{\pi^2 d}{6}$ in Lemma~\ref{lemma:bound_m+mtilde}.~On this event,
\nal{
	\|x_n - z_n\| \le  \Tilde{O}\br{g(n,\delta) + g_1(n)+g_2(n)},~\forall n\in\bN.
}
\end{proposition}

\subsubsection{Bounding $\|z_n - x\ust\|$}\label{sec:z-xstar}
\textit{Proof sketch}: We begin by deriving a recursion for $\|z_{n} -x\ust\|$ in Lemma~\ref{lemma:z_x_recur}. This recursion contains an additive term $\beta_n \Delta_n$.~Lemma~\ref{lemma:Delta} then provides a high-probability upper-bound on $\|\Delta_n\|$. Combining this bound with the recursion from Lemma~\ref{lemma:z_x_recur} yields Proposition~\ref{prop:z_xust}, which gives the bound on $\|z_n - x\ust\|$, and is the main result of this section.

The following result is proved in Appendix~\ref{pf:lemma:z_x_recur}.
\begin{lemma}\label{lemma:z_x_recur}
	We have,
	\nal{
	\|z_{n+1} -x\ust\| & \le \br{1-\beta_n \at(\zeta_n)} \|z_{n} -x\ust\| +\beta_n \|\Delta_n\|,~n\in\bN,
	}
	where $\Delta_n :  = \bar{F}^{(\zeta_n,x_n)}(x_n)-\bar{F}^{(\zeta_n,z_n)}(z_n)$.
\end{lemma}
Lemma~\ref{lemma:z_x_recur} shows that, to control $\|z_n - x\ust\|$, it remains to bound $\|\Delta_n\|$. This is done in the next lemma, whose proof is given in Appendix~\ref{pf:lemma:Delta}.
\begin{lemma}\label{lemma:Delta}
	On the event $\cG$ we have,
	\nal{
		\|\Delta_n\| \le (n+ n_0)^{\kappa_1+\kappa_3}\Tilde{O}\br{g(n,\delta)+g_1(n)+g_2(n)},~\forall n\in\bN.
	}
    \end{lemma} 
Substituting the bound from Lemma~\ref{lemma:Delta} into the recursion of Lemma~\ref{lemma:z_x_recur} gives the following bound on $\|z_n - x\ust\|$. The proof is provided in Appendix~\ref{pf:prop:z_xust}.
\begin{proposition}\label{prop:z_xust}
On the event $\cG$ we have, 
\nal{
 \|z_n - x\ust\| \le \br{n+n_0}^{2\kappa_1+\kappa_3} \tilde{O}\br{ g(n,\delta) + g_1(n)+ g_2(n)},~\forall n\in\bN.
}
\end{proposition}
We are now in a position to prove the concentration bound claimed in Theorem~\ref{th:main_concentration}.
\begin{proof}[Proof of Theorem~\ref{th:main_concentration}]
	Using the triangle inequality, $\|x_n -x\ust\| \le \|x_n - z_n\|+\|z_n - x\ust\|$. We substitute the bound on $\|z_n - x\ust\|$ from Proposition~\ref{prop:z_xust}, and the bound on $\|x_n - z_n\|$ from Proposition~\ref{prop:z_x_bound}. The proof is then completed by noting that the bound in Proposition~\ref{prop:z_xust} dominates the one in Proposition~\ref{prop:z_x_bound}. 
\end{proof}

%% file: appendix_SA.tex
\section{Proofs for Concentration of Markovian Stochastic Approximation (Section~\ref{sec:SA})}\label{sec:app_pf_sec_SA}
Section \ref{sec:cond_n0_sa} discusses the conditions that the parameter $n_0$ must satisfy for the results in Section~\ref{sec:SA} to hold. Henceforth, while performing analysis, it will be assumed that the parameter $n_0$ has been chosen sufficiently large so that it satisfies these conditions. 

\subsection{Proof of Lemma~\ref{lemma:decompose_w}}
We restate the lemma for convenience: $\{M'_{n}\}$ is a martingale difference sequence w.r.t. the filtration generated by $\{(x_n,y_n)\}$. Moreover
	\nal{
		 x_n - z_n  & = \sum_{i=0}^{n-1} \beta_i \chi(i+1,n) \br{M_{i+1}+M'_{i+1} } \\
		 & + \sum_{i=0}^{n-1} \beta_i \chi(i+1,n) \br{H^{(\zeta_i,x_i)}(y_i) - H^{(\zeta_i,x_i)}(y_{i+1})}.
	}

\begin{proof}
It follows directly from the definition of $M'_n$ and the Markov property that 
$\{M'_{n}\}$ is a martingale difference sequence with respect to the filtration generated by $\{(x_n,y_n)\}$.

Subtracting the recursion for $z_n$ in~\eqref{def:zn_1} from the recursion for $x_n$ in~\eqref{def:compact_recur_x}, we obtain
	\al{
		x_{n+1} - z_{n+1} = (1-\beta_n) \br{x_n - z_n} + \beta_n w_n,~n\in\bN,\label{x-z}
	}
where the noise term $w_n$ is given by $M_{n+1} + F(x_n,y_n) - \bar{F}^{(\zeta_n,x_n)}(x_n)$.~The term $F(x_n,y_n) - \bar{F}^{(\zeta_n,x_n)}(x_n)$ is the Markov noise. We rewrite it using the Poisson equation.~Since $H^{(\zeta_n,x_n)}$ solves the Poisson equation~\eqref{def:Poisson_1} associated with the Markov chain with transition probability $p^{(\zeta_n,x_n)}$, we have $F(x_n,y_n) - \bar{F}^{(\zeta_n,x_n)}(x_n) = M'_{n+1} + H^{(\zeta_n,x_n)}(y_n) - H^{(\zeta_n,x_n)}(y_{n+1})$. Therefore,
    \nal{
    	w_n & = M_{n+1} + F(x_n,y_n) - \bar{F}^{(\zeta_n,x_n)}(x_n) \notag \\
    	& = M_{n+1} + M'_{n+1} + H^{(\zeta_n,x_n)}(y_n) - H^{(\zeta_n,x_n)}(y_{n+1}).
    }
Substituting this expression for $w_n$ into the recursion~\eqref{x-z}, and using $x_0 = z_0$, gives the desired decomposition of $x_n - z_n$.    
\end{proof}

\subsection{Proof of Lemma~\ref{lemma:bound_m+mtilde}}\label{pf:martingale_conc}
Recall that Lemma~\ref{lemma:bound_m+mtilde} defines $\cG$ to be the event on which 
\nal{
\left\|\sum_{i=0}^{n-1} \beta_i \chi(i+1,n) \br{M_{i+1}+M'_{i+1} }\right\|
 <  \frac{16d\frakc_9 \frakc_1 \frakc_2\sqrt{\beta} }{\frakc_3} g(n,\delta), \forall n \in \bN.
}
We need to show that $\bP(\cG^c) \le \delta \frac{\pi^2 d}{6}$.

\begin{proof}
Lemma~\ref{lemma:sensitivity_v_x} gives a bound on $\|H^{(\zeta_n,x)}(i)\|$. Substituting this bound into the definition of $M'_n$ yields
\nal{
		\|M'_{n}\| & \le \frac{4\frakc_9 \frakc_1 \frakc_2}{\frakc_3} (n+n_0)^{\kappa_1},~n\in\bN.
}
Fix a $\ell \in \{1,2,\ldots,d\}$. For each fixed $n$, the partial sums $\{\sum_{i=0}^{n-1} \beta_i \chi(i+1,n) \br{M_{i+1}(\ell)+M'_{i+1}(\ell) } \}$ form a martingale. The $i$-th martingale difference is bounded as follows,
\nal{
| \beta_i \chi(i+1,n) \br{M_{i+1}(\ell)+M'_{i+1}(\ell)} | 
& \le \beta_i \chi(i+1,n) | \br{M_{i+1}(\ell)+M'_{i+1}(\ell) } | \\
& \le  \beta_i \chi(i+1,n) \br{ \frac{4\frakc_9 \frakc_1 \frakc_2}{\frakc_3}(i+n_0)^{\kappa_1}+ \frakc_7}\\
& \le  \frac{8\frakc_9 \frakc_1 \frakc_2}{\frakc_3} \beta_i \chi(i+1,n)(i+n_0)^{\kappa_1},
}
where the second inequality follows by substituting the bound on $\|M'_{i}\|$ derived above, and that on $\|M_{i}\|$ from Assumption~\ref{assum:9}. The final inequality follows because $n_0$ is chosen sufficiently large so that Condition~\ref{con:n0}-\eqref{con:n0-I} holds. Note that $\chi(m+1,n) = \Pi_{j=m+1}^n (1-\beta_j)$.~We now apply the Azuma--Hoeffding inequality. To this end, we first bound the sum of the squared increments. We have,
\begin{align}
    \sum_{i=0}^{n}(\beta_i\chi(i+1,n)(i+n_0)^{\kappa_1})^2 &\leq \sum_{i=0}^{n}\chi(i+1,n)(i+n_0)^{2\kappa_1}\beta_i^2 \notag\\
& \le \sum_{i=0}^{n}\chi(i+1,n)(i+n_0)^{2\kappa_1} \frac{\beta^2}{(i+n_0)^{2-2\fraka}} \notag\\
& = \sum_{i=0}^{n} \frac{\beta^2}{(i+n_0)^{2-2\fraka-2\kappa_1}} \Pi_{j=i+1}^{n} \br{1 - \beta_j }, \label{eq:adhoc_bound_azuma}
\end{align}
where we used $\beta_i = \frac{\beta}{(i+n_0)^{1-\fraka}}$.~We bound the final summation in \eqref{eq:adhoc_bound_azuma} using Lemma~\ref{lemma:recursion_bound} with $b_i = \frac{\beta^2}{(i+n_0)^{2-2\fraka -2\kappa_1}}$, $b = \beta^2$, $\rho' = 2-2\fraka-2\kappa_1$, $a_i = \beta_i = \frac{\beta}{(i+n_0)^{1-\fraka}}$ and hence $\rho = 1-\fraka$, $a=\beta$. We have $\rho' - \rho = (1-\fraka) - 2\kappa_1$.

There are two cases:

Case A: $\rho = 1$, equivalently $\fraka = 0$. In this case we need $\rho' = 2(1-\kappa_1) \in  (1,2]$, equivalently $0\le \kappa_1 < \frac{1}{2}$. We also require the hyperparameter $\beta$ to satisfy $\beta \ge 2(1-2\kappa_1)$. 

Case B: $\rho < 1$, equivalently $\fraka >0$. Here we use Part 2 of Lemma~\ref{lemma:recursion_bound}. The condition $\rho'>\rho$ becomes $2-2\fraka -2\kappa_1 > 1-\fraka $, i.e. $\fraka + 2\kappa_1 <1$. In addition Lemma~\ref{lemma:recursion_bound} requires $n_0 \ge \br{\frac{2\br{(1-\fraka) - 2\kappa_1}}{\beta}}^{1\slash (1-\rho)}$.

The above requirements are guaranteed by Condition~\ref{con:sa1} 1 and the corresponding
hyperparameter conditions. Therefore, applying Lemma~\ref{lemma:recursion_bound} to~\eqref{eq:adhoc_bound_azuma} gives 
\nal{
\sum_{i=0}^{n}(\beta_i\chi(i+1,n)(i+n_0)^{\kappa_1})^2 &\le 2\frac{\beta^2}{(n+n_0)^{2-2\fraka -2\kappa_1}} \br{  \frac{(n+n_0)^{1-\fraka}}{\beta} } \\
& = \frac{2\beta}{(n+n_0)^{1-(\fraka +2\kappa_1)}}.
}
Consequently, applying the Azuma--Hoeffding inequality to the martingale $\{\sum_{i=0}^{n-1} \beta_i \chi(i+1,n) \br{M_{i+1}(\ell)+M'_{i+1}(\ell) } \}$ shows that the probability of the event
\nal{
|\sum_{i=0}^{n-1}\beta_i \chi(i+1,n) \br{M_{i+1}(\ell)+M'_{i+1}(\ell) }|  > 
\sqrt{2\br{ \frac{  8\frakc_9 \frakc_1 \frakc_2 }{\frakc_3}  }^2 \frac{2\beta}{(n+n_0)^{1-(\fraka +2\kappa_1)}} \log\br{\frac{n^2}{\delta}}}
} 
is at most $\frac{\delta}{n^2}$. The proof follows by noting that 
\nal{
\|\sum_{i=0}^{n-1}\beta_i \chi(i+1,n) \br{M_{i+1}+M'_{i+1} }\| \le \sum_{\ell = 1}^{d}|\sum_{i=0}^{n-1}\beta_i \chi(i+1,n) \br{M_{i+1}(\ell)+M'_{i+1}(\ell) }|,
}
and then applying a union bound over all $n$ and $\ell$.
\end{proof}

\subsection{Proof of Proposition~\ref{prop:z_x_bound} (Bounding $\|x_n - z_n\|$)}\label{pf:prop3}
Proposition~\ref{prop:z_x_bound} claims that on the set $\cG$, we have $\|x_n - z_n\| \le  \Tilde{O}\br{g(n,\delta) + g_1(n)+g_2(n)},~\forall n\in\bN$.

\begin{proof}
Recall from Lemma~\ref{lemma:decompose_w} that $x_n-z_n$ decomposes into the two terms ~\eqref{x_z_eq_1} and~\eqref{x_z_eq_2}.~We first control the term~\eqref{x_z_eq_2}.~We have 
	\al{
		\eqref{x_z_eq_2} = & \sum_{i=0}^{n-1} \beta_i \chi(i+1,n) \br{H^{(\zeta_i,x_i)}(y_i) - H^{(\zeta_i,x_i)}(y_{i+1})} \notag\\
		& = \underbrace{\sum_{i=0}^{n-1} \beta_i \chi(i+1,n) \br{H^{(\zeta_i,x_i)}(y_i) - H^{(\zeta_{i+1},x_{i+1})}(y_{i+1})}}_{T_1(n)} \notag\\
		& + \underbrace{\sum_{i=0}^{n-1} \beta_i \chi(i+1,n) \br{H^{(\zeta_{i+1},x_{i+1})}(y_{i+1}) - H^{(\zeta_{i+1},x_{i})}(y_{i+1})}}_{T_2(n)} \notag\\
		& +\underbrace{\sum_{i=0}^{n-1} \beta_i \chi(i+1,n) \br{H^{(\zeta_{i+1},x_i)}(y_{i+1}) - H^{(\zeta_i,x_{i})}(y_{i+1})}}_{T_3(n)}. 
		\label{def:T1-T3}
	}
The three terms $T_1(n),T_2(n),T_3(n)$ are bounded separately in Appendix~\ref{sec:T1_T3_bound}, specifically in Lemmas~\ref{lemma:T1},~\ref{lemma:T2}, and~\ref{lemma:T3} respectively.~Taking norms in the decomposition of $x_n - z_n$ from Lemma~\ref{lemma:decompose_w} and then utilizing the above decomposition of \eqref{x_z_eq_2} into $T_1(n)+T_2(n)+T_3(n)$, we get,
	\nal{
		\| x_n - z_n \|  \le \| \sum_{i=0}^{n-1} \beta_i \chi(i+1,n) \br{M_{i+1}+M'_{i+1} } \| 
		 + \|T_1(n)\| + \|T_2(n)\| + \|T_3(n)\|.
	} 	
On the event $\cG$, Lemma~\ref{lemma:bound_m+mtilde} controls the martingale term. Substituting this
bound together with the bounds on $\|T_1(n)\|$, $\|T_2(n)\|$, and $\|T_3(n)\|$
from Lemmas~\ref{lemma:T1},~\ref{lemma:T2},~\ref{lemma:T3} yields the following bound on the event $\cG$: 
\nal{
\|x_n - z_n\| & \le \frac{16d\sqrt{\beta}\frakc_9 \frakc_1 \frakc_2}{\frakc_3} \sqrt{\frac{\log\br{\frac{n^2}{\delta}}}{(n+n_0)^{1-(\fraka+2\kappa_1)}}}+  \br{  \frac{8\frakc_9 \frakc_1 \frakc_2}{(2\fraka+\kappa_1)(n+n_0)^{1-(\fraka + \kappa_1)}} }\\
& + \frac{2\beta \frakc_1 \frakc_2 \br{\frakc_{10} + 2\frakc_9 \frakc_1}\frakc_2 \frakc_9 \frakc_{5}\br{\frakc_7 + \frakc_8 +\frakc_9} \log(n+n_0) }{(2\fraka +2\kappa_1+\kappa_3)(n+ n_0)^{1-(\fraka+2\kappa_1+\kappa_3)}} \\
& + \frac{2\frakc^2_2\frakc_1 \frakc_9 \frakc_4 \br{ \frakc_{10} + 2 \frakc_1}}{(n+n_0)^{1-\br{2\kappa_1 +\kappa_2}  }}.
}
Since $n_0$ has been chosen to be sufficiently large so that Condition~\ref{con:n0}-\eqref{con:n0-II} is met, we have the following simplified bound,
\nal{
	\|x_n - z_n\| &\le \frac{16d\sqrt{\beta}\frakc_9 \frakc_1 \frakc_2}{\frakc_3}\sqrt{\frac{\log\br{\frac{n^2}{\delta}}}{(n+n_0)^{1-(\fraka+2\kappa_1)}}} \\
    & + \frac{4\beta \frakc_1 \frakc_2 \br{\frakc_{10} + 2\frakc_9 \frakc_1}\frakc_2 \frakc_9 \frakc_{5}\br{\frakc_7 + \frakc_8 +\frakc_9} \log(n+n_0) }{(2\fraka +2\kappa_1+\kappa_3)(n+ n_0)^{1-(\fraka+2\kappa_1+\kappa_3)}}\\
& +\frac{2\frakc^2_2\frakc_1 \frakc_9 \frakc_4 \br{ \frakc_{10} + 2 \frakc_1}}{\br{n+n_0}^{1-\br{2\kappa_1+\kappa_2}     }}\\
    & = C_1 g(n,\delta) + C_2 g_1(n) +C_3 g_2(n),
}
where
\al{
C_1 :   & = \frac{16d\sqrt{\beta}\frakc_9 \frakc_1 \frakc_2}{\frakc_3}, \notag\\
C_2:    & = \frac{4\beta \frakc_1 \frakc_2 \br{\frakc_{10} + 2\frakc_9 \frakc_1}\frakc_2 \frakc_9 \frakc_{5}\br{\frakc_7 + \frakc_8 +\frakc_9} }{(2\fraka +2\kappa_1+\kappa_3)}, \notag\\
C_3 : & = 2\frakc^2_2\frakc_1 \frakc_9 \frakc_4 \br{ \frakc_{10} + 2 \frakc_1}.\label{def:C1-C3}
}
The probability bound for $\cG^c$ follows from Lemma~\ref{lemma:bound_m+mtilde}.
\end{proof}

\subsection{Proof of Lemma~\ref{lemma:z_x_recur}}\label{pf:lemma:z_x_recur}
Lemma~\ref{lemma:z_x_recur} claims that 
\nal{
\|z_{n+1} -x\ust\|  \le \br{1-\beta_n \at(\zeta_n)} \|z_{n} -x\ust\| +\beta_n \|\Delta_n\|,~n\in\bN,
}
where $\Delta_n  = \bar{F}^{(\zeta_n,x_n)}(x_n)-\bar{F}^{(\zeta_n,z_n)}(z_n)$.

\begin{proof}
By the definition of $\Delta_n$, recursion~\eqref{def:zn_1} can be re-written as,
\nal{
z_{n+1}  = z_n + \beta_n \br{\bar{F}^{(\zeta_n,z_n)}(z_n)-z_n + \Delta_n }.
}
Hence 
\nal{
	z_{n+1} -x\ust & = (1-\beta_n)(z_n - x\ust) + \beta_n\br{\bar{F}^{(\zeta_n,z_n)}(z_n) + \Delta_n -x\ust}\\
	& = (1-\beta_n)(z_n - x\ust) + \beta_n\br{\bar{F}^{(\zeta_n,z_n)}(z_n)-\bar{F}^{(\zeta_n,z_n)}(x\ust) + \Delta_n},
}
where the second step follows since under Assumption~\ref{assum:7}, $x\ust$ is the unique fixed point of the function $\bar{F}^{(\zeta_n,z_n)}(\cdot)$.~Taking norms and using triangle inequality gives
\al{
	\|z_{n+1} -x\ust\| & \le (1-\beta_n)\| z_n - x\ust\| + \beta_n \| \bar{F}^{(\zeta_n,z_n)}(z_n)-\bar{F}^{(\zeta_n,z_n)}(x\ust) + \Delta_n\| \notag\\
	& \le \br{1-\beta_n (1-\alpha(\zeta_n))}\|z_{n} -x\ust\| +\beta_n \|\Delta_n\| \notag \\
	& = \br{1-\beta_n \at(\zeta_n)} \|z_{n} -x\ust\| +\beta_n \|\Delta_n\|, 
}
where the second inequality uses the contraction property of $\bar{F}^{(\zeta_n,z_n)}(\cdot)$ from Assumption~\ref{assum:7}. This proves the lemma.
\end{proof}

\subsection{Proof of Lemma~\ref{lemma:Delta}}\label{pf:lemma:Delta}
Lemma~\ref{lemma:Delta} states that on the event $\cG$,
	\nal{
		\|\Delta_n\| \le (n+ n_0)^{\kappa_1+\kappa_3}\Tilde{O}\br{g(n,\delta)+g_1(n)+g_2(n)},~\forall n\in\bN.
	}
\begin{proof}
	To bound $\|\Delta_n\|$ we note, 
	\al{
		\| \Delta_n\| & = \|\bar{F}^{(\zeta_n,x_n)}(x_n)-\bar{F}^{(\zeta_n,z_n)}(z_n)\| \notag \\
		& \le \|\bar{F}^{(\zeta_n,x_n)}(x_n)-\bar{F}^{(\zeta_n,x_n)}(z_n)\| + \| \bar{F}^{(\zeta_n,x_n)}(z_n) -\bar{F}^{(\zeta_n,z_n)}(z_n)\| \notag\\
		& \le \alpha(\zeta_n) \|x_n - z_n\|  + \| \bar{F}^{(\zeta_n,x_n)}(z_n) -\bar{F}^{(\zeta_n,z_n)}(z_n)\| \notag\\
		&\le \alpha(\zeta_n) \|x_n - z_n\| + \frac{\frakc_9 \frakc_{10} }{\mu_{\min}(\zeta_n)} \|p^{(\zeta_n,x_n)} - p^{(\zeta_n,z_n)}\| \notag\\
		&\le \alpha(\zeta_n) \|x_n - z_n\| + \frac{\frakc_9 \frakc_{10} L_2(\zeta_n)}{\mu_{\min}(\zeta_n)} \|x_n - z_n \| \notag\\
		&\le \alpha(\zeta_n) \|x_n - z_n\| + \frac{\frakc_{5} \frakc_9 \frakc_{10} \frakc_2}{\frakc_{3}} \log(n+n_0)(n+n_0)^{\kappa_1+\kappa_3}   \|x_n - z_n\| \notag\\
		& = \left\{ \alpha(\zeta_n) + \frac{\frakc_{5} \frakc_9 \frakc_{10}  \frakc_2 (n+n_0)^{\kappa_1+\kappa_3} \log(n+n_0) }{\frakc_{10}}  \right\} \|x_n - z_n\|\notag\\
        & \le \frac{2 \frakc_{5} \frakc_9 \frakc_{10}  \frakc_2}{\frakc_{3}} (n+n_0)^{\kappa_1+\kappa_3} \log(n+n_0) \|x_n - z_n\|
        ,\label{ineq:d_n}
	}
where the second inequality follows from the contraction property in Assumption~\ref{assum:7}. The third inequality uses the bound on $F(x_n,y_n)$ from~Assumption~\ref{assum:9} and the Lipschitz property of the stationary distribution in~Assumption~\ref{assum:14}. The fourth inequality follows from the Lipschitz property of transition probabilities in Assumption~\ref{assum:3} while the second last inequality follows since $\frac{L_2(\zeta_n)}{\mu_{\min}(\zeta_n)} \le \frac{\frakc_{5} \frakc_2}{\frakc_{3}} (n+ n_0)^{\kappa_1+\kappa_3}\log(n+n_0)$.~Last inequality follows from the fact that $n_0$ has been chosen to be sufficiently large so that Condition~\ref{con:n0}-\eqref{con:n0-III} is met.

By Proposition~\ref{prop:z_x_bound}, on the event $\cG$,
\nal{
\|x_n - z_n\|\le C_1 g(n,\delta)+C_2 g_1(n)+C_3 g_2(n),
}
where $C_1,C_2,C_3$ are as in~\eqref{def:C1-C3}.~Proof then follows by substituting this bound in~\eqref{ineq:d_n} and letting
\al{
C_4 :=  \frac{2 \frakc_{5} \frakc_9 \frakc_{10}  \frakc_2}{\frakc_{3}} C_1, \qquad 
C_5:= \frac{2 \frakc_{5} \frakc_9 \frakc_{10}  \frakc_2}{\frakc_{3}} C_2, \qquad 
C_6 := \frac{2 \frakc_{5} \frakc_9 \frakc_{10}  \frakc_2}{\frakc_{3}} C_3,\label{def:C4-6}
}
so that 
\nal{
\|\Delta_n\| \le (n+ n_0)^{\kappa_1+\kappa_3}\log(n+n_0) \left[C_4 g(n,\delta) + C_5 g_1(n) + C_6 g_2(n)\right].
}
This proves the claimed bound.
\end{proof}

\subsection{Proof of Proposition~\ref{prop:z_xust}}\label{pf:prop:z_xust}
Proposition~\ref{prop:z_xust} states that, on the event $\cG$,
\nal{
	\|z_n - x\ust\| \le \br{n+n_0}^{2\kappa_1+\kappa_3} \tilde{O}\br{ g(n,\delta) + g_1(n)+ g_2(n)},~\forall n\in\bN.
}
\begin{proof}
From Lemma~\ref{lemma:z_x_recur} we have, 
\al{
	\|z_{n+1} -x\ust\|  \le  \br{1-\beta_n \at(\zeta_n)} \|z_{n} -x\ust\| +\beta_n \|\Delta_n\|.\label{adhoc:prop_2_1}
}	
On the event $\cG$, Lemma~\ref{lemma:Delta} gives,
\al{
\|\Delta_n\| \le (n+ n_0)^{\kappa_1+\kappa_3}\log(n+n_0) \left[C_4 g(n,\delta) + C_5 g_1(n) + C_6 g_2(n)\right].\label{adhoc:prop_2_2}
}
We now identify the decay rates of the three terms in the right-hand side,
\al{
\beta_n (n+ n_0)^{\kappa_1+\kappa_3}g(n,\delta) &= \tilde{O}\br{ \frac{\beta}{n^{\frac{3}{2}-\br{\frac{3\fraka}{2} + 2\kappa_1 + \kappa_3} }}  },\notag \\
\beta_n (n+ n_0)^{\kappa_1+\kappa_3}g_1(n) & = \tilde{O}\br{ \frac{\beta}{n^{2-\br{2\fraka + 3\kappa_1+2\kappa_3} }  }},\notag\\
\beta_n (n+ n_0)^{\kappa_1+\kappa_3}g_2(n) & = \tilde{O}\br{ \frac{\beta}{n^{2-\br{\fraka+ 3\kappa_1+\kappa_2 + \kappa_3} }  }}.\label{adhoc:prop_2_3}
}
From Assumption~\ref{assum:3}-(a) we have $\beta_n \at(\zeta_n) \ge \br{\frac{\beta}{(n+n_0)^{1-\fraka}} }  \br{\frac{\frakc_3}{(n+n_0)^{\kappa_1}} } = \frac{\beta \frakc_3}{(n+n_0)^{1-\fraka+\kappa_1}}$.~We substitute~\eqref{adhoc:prop_2_2} and~\eqref{adhoc:prop_2_3} into~\eqref{adhoc:prop_2_1} and then  we use Lemma~\ref{lemma:recursion_bound} to bound the three resulting summation terms. For all three summations we use $a = \beta \frakc_3, \rho = 1-\fraka +\kappa_1$, $b=1$, while $\rho'$ is taken to be (i) $\frac{3}{2}-\br{\frac{3\fraka}{2} + 2\kappa_1 + \kappa_3}$ and (ii) $2-\br{2\fraka + 3\kappa_1+2\kappa_3}$ and (iii) $2-\br{ \fraka + 3\kappa_1+\kappa_2 + \kappa_3}$. Then consider the following two possibilities ($1-\rho = \fraka-\kappa_1$) which arise when applying Lemma~\ref{lemma:recursion_bound} (note that there are two parts in Lemma~\ref{lemma:recursion_bound}):

Case A: $\rho=1$ which is equivalent to $\fraka = \kappa_1$. In this case Part 1 of Lemma~\ref{lemma:recursion_bound} applies.~We require $\rho' >1$ and $a>2(\rho'-1)$. This condition reduces to the following (corresponding to the three summation terms discussed above):
\begin{enumerate}[(i)]
    \item $\frac{3\fraka}{2} + 2\kappa_1 + \kappa_3 < \frac{1}{2}$, $\beta \frakc_3 \ge 2(\frac{1}{2}-\br{\frac{3\fraka}{2} + 2\kappa_1 + \kappa_3} ) = 1- (3\fraka + 4\kappa_1 + 2\kappa_3)$ 
    \item $2\fraka + 3\kappa_1+2\kappa_3 < 1$, $\beta \frakc_3 \ge 2-\br{4\fraka + 6\kappa_1+4\kappa_3}$
    \item $\fraka + 3\kappa_1+\kappa_2 + \kappa_3 < 1$, $\beta \frakc_3 \ge 2-\br{ 2\fraka + 6\kappa_1+2\kappa_2 + 2\kappa_3}$
\end{enumerate}
These requirements are guaranteed by Condition~\ref{con:sa1} Case~C.

Case B: $\rho<1$, equivalently $\fraka > \kappa_1$. In this case, Part 2 of Lemma~\ref{lemma:recursion_bound} applies.~The required conditions are $\rho'>\rho$ and $n_0 \ge \br{ \frac{2(\rho'-\rho)}{\beta \frakc_3} }^{1\slash (\fraka - \kappa_1)}$. This yields the following conditions for the three cases corresponding to the three summation terms discussed above:
\begin{enumerate}[(i)]
    \item $\frac{\fraka}{2}+3\kappa_1 +\kappa_3 < \frac{1}{2}$, $n_0 \ge \br{\frac{1-\br{\fraka + 6\kappa_1 + 2\kappa_3} }{\beta \frakc_3}}^{1\slash (\fraka - \kappa_1)}$.
    \item $1> \fraka + 4\kappa_1+2\kappa_3$, $n_0 \ge \br{\frac{2\br{1-\br{\fraka + 4\kappa_1+2\kappa_3}}}{\beta \frakc_3}}^{1\slash (\fraka - \kappa_1)}$
    \item $1>4\kappa_1+\kappa_2+\kappa_3$, $n_0 \ge \br{\frac{2\br{1-\br{4\kappa_1+\kappa_2 + \kappa_3}}}{\beta \frakc_3}}^{1\slash (\fraka - \kappa_1)}$.
\end{enumerate}
These requirements are guaranteed by Condition~\ref{con:sa1} Case D-\eqref{condition_caseD_iii}.

An application of Lemma~\ref{lemma:recursion_bound} gives the following bound on $\|z_n-x\ust\|$ in terms of $\|\Delta_n\|$,
\nal{
\|z_n - x\ust\| \le \frac{2\beta_n \|\Delta_n \|}{\beta_n \at(\zeta_n)}.
}
Upon substituting $\at(\zeta_n)= \frac{\frakc_{3}}{(n+n_0)^{\kappa_1}}$ this gives,
\nal{
	\|z_n - x\ust\| \le \log(n+n_0)\frac{\br{n+n_0}^{2\kappa_1+\kappa_3}}{\frakc_{3}} \br{C_4 g(n,\delta) + C_5 g_1(n)+C_6 g_2(n)}.
}
This completes the proof.
\end{proof}

\subsection{Proof of Theorem~\ref{th:main_concentration} (bounding $\|x_n-x\ust\|$)}
We prove the explicit high-probability concentration bound here. The $\tilde{O}(\cdot)$ statement in Theorem~\ref{th:main_concentration} follows by suppressing constants and logarithmic factors.

Fix $\delta>0$. On the event $\cG$ Proposition~\ref{prop:z_x_bound} gives,
\nal{
\|x_n - z_n\| \le C_1 g(n,\delta) + C_2 g_1(n) +C_3 g_2(n),~\forall n\in\bN,
}
while Proposition~\ref{prop:z_xust} gives,
\nal{
 \|z_n - x\ust\| \le \log(n+n_0)\frac{\br{n+n_0}^{2\kappa_1+\kappa_3}}{\frakc_{3}} \br{C_4 g(n,\delta) + C_5 g_1(n)+C_6 g_2(n)},~\forall n\in\bN.
}
Therefore, by the triangle inequality,
\nal{
\|x_n -x\ust\|& \le \|x_n - z_n\|+ \|z_n - x\ust\| \\
& \le \left[ C_1 g(n,\delta) + C_2 g_1(n) +C_3 g_2(n)\right] \\
& + \left[\log(n+n_0)\frac{\br{n+n_0}^{2\kappa_1+\kappa_3}}{\frakc_{3}} \br{C_4 g(n,\delta) + C_5 g_1(n)+C_6 g_2(n)}\right]~\forall n\in\bN.
}
The quantities $C_1,C_2,C_3,C_4,C_5,C_6$ are as in~\eqref{def:C1-C3} and~\eqref{def:C4-6}. For sufficiently large $n_0$, the second term dominates the first, and hence we have,
\nal{
\|x_n -x\ust\| \le  2\left[\log(n+n_0)\frac{\br{n+n_0}^{2\kappa_1+\kappa_3}}{\frakc_{3}} \br{C_4 g(n,\delta) + C_5 g_1(n)+C_6 g_2(n)}\right]~\forall n\in\bN.
}
It remains to prove almost sure convergence. For all $\delta >0$ we have that the probability of the following event is greater than or equal to $1- \delta \frac{\pi^2 d}{6}$:
\nal{
	\left\{\|x_n - x\ust\|  \le 2\log(n+n_0)\frac{\br{n+n_0}^{2\kappa_1+\kappa_3}}{\frakc_{3}} \br{C_4 g(n,\delta) + C_5 g_1(n)+C_6 g_2(n)}, \forall n\in\bN \right\}
	}
Now, we choose a sequence $\{\delta_k\}_{k\in\bN}$ with $\delta_k =\frac{1}{k^2}$. Define the event 
\nal{
& \cE_k := \\
& \left\{ 
\exists n\in\bN: \|x_n - x\ust\| > 2\log(n+n_0)\frac{\br{n+n_0}^{2\kappa_1+\kappa_3}}{\frakc_{3}} \br{C_4 g(n,\delta_k) + C_5 g_1(n)+C_6 g_2(n)}   
\right\}.
}
The high-probability bound above implies
\nal{
\bP(\cE_k) <  \frac{\pi^2 d}{6 k^2}.
}
Since $\sum_{k=1}^{\infty} \frac{\pi^2 d}{6 k^2} < \infty$, the Borel-Cantelli lemma yields
\nal{
\bP\br{ \cE_k \mbox{ i.o.} } = 0.
}
Thus, with probability one, for all sufficiently large $k$, the above bound holds uniformly over all $n\in\bN$. Finally, for each fixed $k$, the right-hand side converges to zero as
$n\to\infty$.~Indeed, suppressing logarithmic factors,
\nal{
& \log(n+n_0)\frac{\br{n+n_0}^{2\kappa_1 +\kappa_3}}{\frakc_{10}} \br{C_4 g(n,\delta_k) + C_5 g_1(n)+C_6 g_2(n)} \\
& = \tilde{O}\br{\sqrt{ \frac{1}{(n+n_0)^{1-(\fraka + 2\kappa_1)}}  } 
+\frac{1}{(n+n_0)^{1-(\fraka+2\kappa_1 +\kappa_3)}} 
+ \frac{1}{(n+n_0)^{1-(2\kappa_1 + \kappa_2)}} }.
}
The exponents are positive under Condition~\ref{con:sa1}. Therefore $\|x_n - x\ust\|\to 0$ almost surely. This also yields an almost sure rate of convergence.~This completes the proof.

\subsection{Conditions on $n_0$ for analyzing SA~\eqref{def:gen_SA}}\label{sec:cond_n0_sa}
Recall that $n_0$ enters the following quantities:
\begin{enumerate}[(i)]
    \item The step-sizes,
    \nal{
    \beta_n = \frac{\beta_0}{(n+n_0)^{1-\fraka}}.
    }
    \item The contraction factor in Assumption~\ref{assum:at}, 
    \nal{
    \at(\zeta_n) \ge \frac{\frakc_3}{(n+n_0)^{\kappa_1}}.
    }
    \item The quantity $L_1(n)$ in Assumption~\ref{assum:3}, where 
    \nal{
    L_1(n) = \frac{\frakc_4}{(n+n_0)^{1-\kappa_2}}.
    }
    Note that $L_1(n)$ bounds the rate at which transition probabilities change with $n$.
    \item The quantity $L_2(n)$ in Assumption~\ref{assum:3}, where 
    \nal{
    L_2(\zeta_n) \le \frakc_{5} (n+n_0)^{\kappa_3}\log(n+n_0).
    }
    The quantity $L_2(\zeta_n)$ controls the sensitivity of the transition probabilities at time $n$.
\end{enumerate}
For the results in Section~\ref{sec:SA} on general stochastic approximation, we assume that $n_0$ is sufficiently large, so that the following conditions hold. 

\begin{remark}
    Note that if we instead work with the stepsize $\beta_n=\frac{\beta}{(n+1)^{1-\fraka}}$, then the results would instead hold from some time $n'$ onwards, where $n'$ satisfies the below conditions which depend on the problem parameters. Hence, our results hold even without any knowledge of the problem parameters. We only require the condition that $n_0$ is large enough to ensure that the results frome time $0$.
\end{remark} 

\begin{condition}\label{con:n0}
The parameter $n_0$ is chosen sufficiently large so that the following conditions are satisfied:
\begin{enumerate}[(I)]
    \item \label{con:n0-I}
    \al{
    \frac{2\frakc_9 \frakc_1 \frakc_2}{\frakc_3}  n^{\kappa_1}_0 & > \frakc_7. \label{con:1}
    }
    \item \label{con:n0-II}
    \al{
    \frac{2\beta \frakc_1 \frakc_2 \br{\frakc_{10} + 2\frakc_9 \frakc_1}\frakc_2 \frakc_9 \frakc_{5}\br{\frakc_7 + \frakc_8 +\frakc_9} \log(n+n_0) }{(2\fraka +2\kappa_1+\kappa_3)(n+ n_0)^{1-(\fraka+2\kappa_1+\kappa_3)}} 
& > \frac{8\frakc_9 \frakc_1 \frakc_2}{(2\fraka+\kappa_1)(n+n_0)^{1-(\fraka + \kappa_1)}  },\notag\\
&~\forall n\in\bN.\label{con:2}
    }
    \item \label{con:n0-III}
    \al{
    \frakc_{5} \frakc_9 \frakc_{10}  \frakc_2 (n+n_0)^{\kappa_1+\kappa_3} \log(n+n_0) & > 1, \qquad \forall n \in\bN.\label{con:3}
    }
    \item \label{con:n0-IV}
    \al{
    \frac{4\frakc_9 \frakc_1 \frakc_2}{(2\fraka +\kappa_1)}   \br{  \frac{1}{(n+n_0)^{1-2\fraka - \kappa_1}} } &> n^{\fraka}_0 \br{\frac{2\frakc_9 \frakc_1\frakc_2 }{(n+ n_0)^{1-\kappa_1}}} + \frac{2\frakc_9 \frakc_1\frakc_2 }{(n+n_0)^{1-\fraka-\kappa_1}},\forall n\in\bN. \label{con:4}
    }
    \item \label{con:n0-V}
    \al{
    \br{\frakc_{10} + 2\frakc_1} \frakc_9 \frakc_2 \frakc_{5} \log(\ell+n_0) (\ell+n_0)^{\kappa_1 +\kappa_3} & > 
2 \frakc_{6},~\forall \ell \in\bN. \label{con:5}
    }
\end{enumerate}
\end{condition}

%% file: appendix_SA_auxiliary.tex
\section{Auxiliary results for Appendix~\ref{sec:app_pf_sec_SA}}
\label{sec:aux_SA}
\subsection{Some properties of the solution to the Poisson equation defined by~\eqref{def:Poisson_1}}
\begin{lemma}\label{lemma:sensitivity_v_x}
Consider the Poisson equation defined by~\eqref{def:Poisson_1} and parameterized by $(\zeta,x)$; its solution is denoted by $H^{(\zeta,x)}$.~Then, for every $n\in\bN$, the following bounds hold,
	\begin{enumerate}[(i)]
		\item \label{lemma:6-1-1}
        For all $x\in\cX$
		\nal{
			\| H^{(\zeta_n,x)}\| & \le \frac{2\frakc_9 \frakc_1 \frakc_2}{\frakc_3} (n+n_0)^{\kappa_1},~a.s.
		}
		\item \label{lemma:6-1-2}
        For all $x,x' \in \cX$ and $\ell \in \left\{1,2,\ldots,d\right\}$
		\nal{
			& \| [H^{(\zeta_n,x)}(\cdot)](\ell) - [H^{(\zeta_n,x')}(\cdot)](\ell) \|_{\infty} \\
			&\le  
            \frac{\frakc_1 \frakc_2 (n+n_0)^{\kappa_1}}{\frakc_3} \br{2 \frakc_{6} +
            \br{\frakc_{10}+2\frakc_1}\frakc_9 \frac{\frakc_{5} \frakc_2(n+n_0)^{\kappa_3+\kappa_1}\log(n+n_0)}{\frakc_3} } \|x-x'\|,~a.s.
		}
		\item \label{lemma:6-1-3}	
        For all $x\in \cX$ and $\ell \in \left\{1,2,\ldots,d\right\}$        
		\nal{
			\| [H^{(\zeta_n,x)}(\cdot)](\ell) - [H^{(\zeta_{n+1},x)}(\cdot)](\ell) \|   \le \frac{ \frakc^2_2\frakc_1 \frakc_9 \frakc_4 \br{ \frakc_{10} + 2 \frakc_1} }{\frakc^2_3 (n+n_0)^{ 1-\br{2\kappa_1 + \kappa_2}  }},~a.s.		
		}			
	\end{enumerate}	
\end{lemma}
\begin{proof}
	\begin{enumerate}[(i)]
		\item We will show $\|H^{(\zeta,x)}\| \le \frac{2\frakc_9 \frakc_1}{\mu_{\min}(\zeta)}, \forall \zeta \in \bR_+, x\in\cX$.
		Consider the Markov chain $\{y_n\}$ with transition probabilities given by $p^{(\zeta,x)}$.~Let $\tau$ denote the hitting time of the designated state $i\ust$ from Assumption~\ref{assum:2}. We have, for any state $s\in\mathcal{Y}$,
		\nal{
			H^{(\zeta,x)}(s) = \bE\br{\sum_{n=0}^{\tau} F(x,y_n) \Big| y_0 = s  } -  \bE\br{\sum_{n=0}^{\tau} F(x,y_n) \Big| y_0 = i\ust}.
		}
		Since $\|F(x,y)\|\le \frakc_9$, and by Assumption~\ref{assum:2} $\sup_{s\in\cY}\bE\left[\tau \mid y_0 = s\right] \le \frac{\frakc_1}{\mu_{\min}(\zeta)}$, each of these two expectations has a norm at most $\frakc_9 \frac{\frakc_1}{\mu_{\min}(\zeta)}$. Hence $\|H^{(\zeta,x)}\| \le \frac{2\frakc_9 \frakc_1}{\mu_{\min}(\zeta)}, \forall \zeta \in \bR_+, x\in\cX$.~Substituting the lower bound $\mu_{\min}(\zeta_n)\ge \frac{\frakc_3}{\frakc_2 (n+n_0)^{\kappa_1}}$ which follows from Assumption~\ref{assum:at} and Assumption~\ref{assum:3}-(a) gives $\| H^{(\zeta_n,x)}\| \le \frac{2\frakc_9 \frakc_1 \frakc_2}{\frakc_3} (n+n_0)^{\kappa_1},~a.s.$ This completes the proof. 
		\item By the sensitivity estimate for Poisson equations stated in Appendix B.2 of~\citep{chandak2022concentration}, for every $\ell\in\{1,\ldots,d\}$ and every $x,x'\in\mathcal{X}$,
		\nal{
			& \| [H^{(\zeta_n,x)}(\cdot)](\ell) - [H^{(\zeta_n,x')}(\cdot)](\ell) \|_{\infty} \\ 
			&\le  \frac{\frakc_1}{\mu_{\min}(\zeta)} \br{2 \frakc_{6} +
            \br{\frakc_{10}+2\frakc_1}\frakc_9 \frac{L_2(\zeta)}{\mu_{\min}(\zeta)} } \|x-x'\| \\
            &\le  \frac{\frakc_1 \frakc_2 (n+n_0)^{\kappa_1}}{\frakc_3} \br{2 \frakc_{6} +
            \br{\frakc_{10}+2\frakc_1}\frakc_9 \frac{\frakc_{5} \frakc_2(n+n_0)^{\kappa_3+\kappa_1}\log(n+n_0)}{\frakc_3} } \|x-x'\|,
		}
		where we have used the upper-bound $L_2(\zeta_n) \le \frakc_{5} (n+n_0)^{\kappa_3}\log(n+n_0)$ from Assumption~\ref{assum:3}(c), and the lower-bound $\mu_{\min}(\zeta_n)\ge \frac{\frakc_3}{\frakc_2(n+n_0)^{\kappa_1}}$ from Assumption~\ref{assum:3}(a).		
		\item Again using the Poisson-equation sensitivity estimate from Appendix B.2 of~\citep{chandak2022concentration}, for every $n\in\mathbb{N}$, $x\in\mathcal{X}$, and $\ell\in\{1,\ldots,d\}$,
		\nal{
			\| [H^{(\zeta_n,x)}(\cdot)](\ell) - [H^{(\zeta_{n+1},x)}(\cdot)](\ell) \|  
			& \le \frac{\frakc_1}{\mu_{\min}(\zeta_n)} \br{ \frac{\frakc_{10}}{\mu_{\min}(\zeta_n)}L_1(n) \frakc_9 + L_1(n)\frac{2\frakc_9 \frakc_1}{\mu_{\min}(\zeta_n)}}\\
			&  = \frac{\frakc_1 \frakc_9 L_1(n)}{\mu_{\min}(\zeta_n)^2} \br{ \frakc_{10} + 2 \frakc_1}.
		}			
		Substituting $L_1(n) = \frac{\frakc_{4}}{\br{n+n_0}^{1-\kappa_2}}$ from Assumption~\ref{assum:3} (b); and using $\mu_{\min}(\zeta_n)\ge \frac{\frakc_3}{\frakc_2(n+n_0)^{\kappa_1}}$ which in turn follows from Assumption~\ref{assum:7} and Assumption~\ref{assum:6}, gives
		\nal{
			\| [H^{(\zeta_n,x)}(\cdot)](\ell) - [H^{(\zeta_{n+1},x)}(\cdot)](\ell) \|   \le \frac{ \frakc^2_2\frakc_1 \frakc_9 \frakc_4 \br{ \frakc_{10} + 2 \frakc_1} }{\frakc^2_3 (n+n_0)^{ 1-\br{2\kappa_1 + \kappa_2}  }},~a.s.		
		}	
        This completes the proof.
    \end{enumerate}	
\end{proof}

\subsection{Bounds on $T_1(n),T_2(n),T_3(n)$ used in proof of Proposition~\ref{prop:z_x_bound}}
\label{sec:T1_T3_bound}
\begin{lemma}[Bounding $\|T_1(n)\|$]\label{lemma:T1}  
For every $n\in\bN$
	\nal{
		\|T_1(n)\| \le   \frac{8\frakc_9 \frakc_1 \frakc_2}{\frakc_3 (n+n_0)^{1-(\fraka + \kappa_1)}}.
	}        
\end{lemma}

\begin{proof}
	We have,
	\al{
		T_1(n) & = \sum_{i=0}^{n-1} \beta_i \chi(i+1,n) \br{H^{(\zeta_i,x_i)}(y_i) - H^{(\zeta_{i+1},x_{i+1})}(y_{i+1})}\notag \\
		& = \left\{\beta_0 \chi(1,n) \br{H^{(\zeta_0,x_0)}(y_0) -\beta_{n-1} \chi(n,n) H^{(\zeta_n,x_{n})}(y_{n})} \right\}\notag\\
		& + \sum_{i=1}^{n-2} H^{(\zeta_i,x_i)}(y_i) \left\{ \beta_i \chi(i+1,n) - \beta_{i-1} \chi(i,n) \right\}\notag\\
		& = \left\{\beta_0 \chi(1,n) \br{H^{(\zeta_0,x_0)}(y_0) -\beta_{n-1} \chi(n,n) H^{(\zeta_n,x_{n})}(y_{n})} \right\}\notag\\
		& + \sum_{i=1}^{n-2} H^{(\zeta_i,x_i)}(y_i) \chi(i+1,n) \left\{ \beta_i - \beta_{i-1} (1-\beta_i)  \right\}\notag\\
		& = \left\{\beta_0 \chi(1,n) \br{H^{(\zeta_0,x_0)}(y_0) -\beta_{n-1} \chi(n,n) H^{(\zeta_n,x_{n})}(y_{n})} \right\}\notag\\
		& + \sum_{i=1}^{n-2}  \chi(i+1,n)  H^{(\zeta_i,x_i)}(y_i)\left\{ \beta_i - \beta_{i-1} + \beta_{i-1} \beta_i  \right\}.\label{T1_manipulate}
	}
	We continue with bounding the summation term above. Since from Lemma~\ref{lemma:sensitivity_v_x}-\eqref{lemma:6-1-1} we have $\| H^{(\zeta_i,x_i)}(y_i)\|\le \frac{2\frakc_9 \frakc_1}{\mu_{\min}(\zeta_i)}$ we get,
	\nal{
		\|H^{(\zeta_i,x_i)}(y_i) \left\{ \beta_i - \beta_{i-1} + \beta_{i-1} \beta_i  \right\}\| & 
		\le   \frac{2\frakc_9 \frakc_1}{\mu_{\min}(\zeta_i)} \br{|  \beta_i - \beta_{i-1} | + \beta^2_i} \\
		& \approx   \frac{2\frakc_9 \frakc_1 \frakc_2 (i+n_0)^{\kappa_1} }{\frakc_3}  \br{ \frac{\beta}{(i+n_0)^{2-\fraka}} + \frac{\beta^2}{(i+n_0)^{2(1-\fraka)}} } \\
		&\le \frac{2\frakc_9 \frakc_1 \frakc_2}{\frakc_3}  \br{  \frac{2\beta}{(i+n_0)^{2-(2\fraka +\kappa_1)}} }.
	}    
    This gives us the following bound on the summation term,
	\nal{
		& \sum_{i=1}^{n-2} \| H^{(\zeta_i,x_i)}(y_i) \chi(i+1,n) \left\{ \beta_i - \beta_{i-1} + \beta_{i-1} \beta_i  \right\} \| \\
		&\le \frac{4\beta \frakc_9 \frakc_1 \frakc_2}{\frakc_3} \sum_{i=1}^{n-2} \frac{ 1 }{(i+n_0)^{2-2\fraka -\kappa_1}} \Pi_{j=i}^{n}(1-\beta_j).
	}
    We now use Lemma~\ref{lemma:recursion_bound} to bound the summation term above.~In the setting of Lemma~\ref{lemma:recursion_bound} we use $b_i = \frac{ 1 }{(i+n_0)^{2-2\fraka -\kappa_1}}$, $a_j = \beta_j = \frac{\beta}{(n_0 + j)^{1-\fraka}}$, $\rho = 1-\fraka,\rho'= 2-2\fraka-\kappa_1, a = \beta$. There are two cases:

    Case A: $\fraka =0$, equivalently $\rho=1$. In this case we apply Part 1 of Lemma~\ref{lemma:recursion_bound}, which requires $\rho' = 2-\kappa_1 \in (1,2]$, i.e. $0\le \kappa_1 < 1$, and also $\beta \ge 2(1-\kappa_1)$. 

    Case B: $\fraka >0$, equivalently $\rho <1$: In this case we apply Part 2 of Lemma~\ref{lemma:recursion_bound}, which requires $2-2\fraka-\kappa_1 > 1-\fraka$, i.e. $\fraka + \kappa_1 < 1$ and also $n_0 \ge \br{\frac{2\br{ 1-\fraka -\kappa_1 }}{\beta}}^{1\slash \fraka }$.

The above requirements are guaranteed by Condition~\ref{con:sa1}.~In both these cases we get the following bound,
    \nal{
    \frac{4\beta \frakc_9 \frakc_1 \frakc_2}{\frakc_3}  \sum_{i=1}^{n-2} \frac{ 1 }{(i+n_0)^{2-2\fraka -\kappa_1}} \Pi_{j=i}^{n}(1-\beta_j) 
    &\le  \frac{8\frakc_9 \frakc_1 \frakc_2 (n_0 + n)^{1-\fraka} }{\frakc_3 (n+n_0)^{2-2\fraka -\kappa_1}}\\
    & = \frac{8\frakc_9 \frakc_1 \frakc_2}{\frakc_3 (n+n_0)^{1-\fraka -\kappa_1}}.
    }
    Since $\|H^{(\zeta_n,x)}(y)\| \le \frac{2\frakc_9 \frakc_1}{\mu_{\min}(\zeta_n)}\le \frac{2\frakc_9 \frakc_1\frakc_2}{\frakc_3} (n+n_0)^{\kappa_1}$ and $\beta_i \chi(i+1,n) \le \frac{\beta}{(i+n_0)^{1-\fraka}}\br{\frac{i+1+n_0}{n+n_0}}^\beta$ (Lemma~\ref{lemma:chi_approx}) we get, 
	\nal{
		& \| \beta_0 \chi(1,n) \br{H^{(\zeta_0,x_0)}(y_0) -\beta_{n-1} \chi(n,n) H^{(\zeta_n,x_{n})}(y_{n})} \| \\
        & \le \frac{\beta }{\frakc_3 n^{1-\fraka -\beta}_0 (n+n_0)^{\beta}} \left[2\frakc_9 \frakc_1 \frakc_2 n^{\kappa_1}_0 + \frac{2\beta \frakc_9 \frakc_1 \frakc_2}{(n+n_0)^{1-(\fraka+\kappa_1)}}\right]\\
        & \le \frac{\beta }{\frakc_3 n^{1-\fraka -\beta}_0 (n_0)^{\beta}} \left[2\frakc_9 \frakc_1 \frakc_2 n^{\kappa_1}_0 + \frac{2\beta \frakc_9 \frakc_1 \frakc_2}{(n+n_0)^{1-(\fraka+\kappa_1)}}\right]\\
        & = \frac{2\beta \frakc_9 \frakc_1 \frakc_2}{\frakc_3 n^{1-(\fraka + \kappa_1)}_0}
        + \frac{2\beta^2 \frakc_9 \frakc_1 \frakc_2}{\frakc_3 (n_0)^{1-\fraka}(n+n_0)^{1-(\fraka+\kappa_1)}}.
    }
	Proof is completed by substituting the above two bounds into~\eqref{T1_manipulate} and noting that since $n_0$ has been taken to be sufficiently large, Condition~\ref{con:n0}-\eqref{con:n0-IV} is met.
\end{proof}

\begin{lemma}\label{lemma:T2}(Bounding $T_2(n)$)
	For every $n\in\bN$,
	\nal{
		\| T_2(n) \|  \le   \frac{2\beta \frakc_1 \frakc_2 \frakc^{-1}_3 \br{\frakc_{10} + 2\frakc_1} \frakc_2 \frakc_9 \frakc_{5}\br{\frakc_7 + \frakc_8 +\frakc_9} \log(n+n_0)}{(2\fraka+2\kappa_1 +\kappa_3)(n+ n_0)^{1-(\fraka+2\kappa_1 +\kappa_3)}}, n\in\bN.
	}
\end{lemma}
\begin{proof}
Recall that $T_2(n) = \sum_{i=0}^{n-1} \beta_i \chi(i+1,n) \br{H^{(\zeta_{i+1},x_{i+1})}(y_{i+1}) - H^{(\zeta_{i+1},x_{i})}(y_{i+1})}$.

\textit{Step 1}: We first prove the following one-step bound:
	\nal{
		\|H^{(\zeta_{i+1},x_{i+1})}(y_{i+1}) - H^{(\zeta_{i+1},x_{i})}(y_{i+1})\| \le \beta \frac{\frakc_1 \frakc_2  \br{\frakc_{10} + 2\frakc_1} \frakc_2 \frakc_9 \frakc_{5}\br{\frakc_7 + \frakc_8 +\frakc_9} \log(i+n_0)}{\frakc_3 (i+ n_0)^{1-\br{\fraka+2\kappa_1 +\kappa_3}}}.
	}
	Consider the iterations~\eqref{def:gen_SA}.~By Assumption~\ref{assum:9}, $\|F(x_n,Y_n)-x_n+M_{n+1}\| \le \frakc_7 + \frakc_8 +\frakc_9$. Since $\beta_n = \frac{\beta}{(n+n_0)^{1-\fraka}}$, we obtain $\|x_{i+1}-x_i\| = \|\beta_i \left(F(x_i,y_i)-x_i+M_{i+1}\right) \|\le \beta_i \br{\frakc_7 + \frakc_8 +\frakc_9}$.~We have,
	\nal{
		& \|H^{(\zeta_{i+1},x_{i+1})}(y_{i+1}) - H^{(\zeta_{i+1},x_{i})}(y_{i+1})\| \\
		& \le \frac{\frakc_1 \frakc_2}{\frakc_3} (i+ n_0)^{\kappa_1}  \br{2 \frakc_{6} +\br{\frakc_{10} + 2 \frakc_1} \frakc_9 \frakc_2 \frakc_{5} \log(i+n_0) (i + n_0)^{\kappa_1 +\kappa_3}  } \|x_i - x_{i+1}\|\\
		& \le \frac{\beta_i  \frakc_1 \frakc_2}{\frakc_3}(i+n_0)^{\kappa_1}  \br{2 \frakc_{6} +\br{\frakc_{10} + 2\frakc_1} \frakc_9 \frakc_2 \frakc_{5} \log(i+n_0) (i+n_0)^{\kappa_1 +\kappa_3}  }\br{\frakc_7 + \frakc_8 +\frakc_9}\\
		& \le  \beta \frac{\frakc_1 \frakc_2  \br{\frakc_{10} + 2\frakc_1} \frakc_2 \frakc_9 \frakc_{5}\br{\frakc_7 + \frakc_8 +\frakc_9} \log(i+n_0)}{\frakc_3 (i+ n_0)^{1-\br{\fraka+2\kappa_1 +\kappa_3}}},
	}   
	where the first inequality follows from Lemma~\ref{lemma:sensitivity_v_x}-\eqref{lemma:6-1-2}, the second follows from the bound on $\|x_i - x_{i+1}\|$ that was derived above, while in the third we have used $\beta_i = \frac{\beta}{(i+n_0)^{1-\fraka}}$. In the last inequality we have also used the fact that since $n_0$ is sufficiently large so that Condition~\ref{con:n0}-\eqref{con:n0-V} is met, the second term inside the parentheses dominates $2\frakc_6$.  

	\textit{Step 2}: We now return to $T_2(n)$. We have,
	\nal{
		& \| T_2(n) \| \\
        & = \| \sum_{i=0}^{n-1} \beta_i \chi(i+1,n) \br{H^{(\zeta_{i+1},x_{i+1})}(y_{i+1}) - H^{(\zeta_{i+1},x_{i})}(y_{i+1})}\|\\
		&\le \beta \sum_{i=0}^{n-1}\frac{\chi(i+1,n)}{(i+n_0)^{1-\fraka}} \|H^{(\zeta_{i+1},x_{i+1})}(y_{i+1}) - H^{(\zeta_{i+1},x_{i})}(y_{i+1})\|\\
		& \le \beta^2 \br{\frakc_1 \frakc_2 \frakc^{-1}_3 \br{\frakc_{10} + 2\frakc_1} \frakc_2 \frakc_9 \frakc_{5}\br{\frakc_7 + \frakc_8 +\frakc_9} \log(n+n_0)}\\
        &\qquad \times\sum_{i=0}^{n-1} \frac{\chi(i+1,n)}{\br{i+n_0}^{1-\fraka}} \frac{1}{(i+ n_0)^{1-\br{\fraka+2\kappa_1+\kappa_3}}}\\
		& = \beta^2 \br{\frakc_1 \frakc_2 \frakc^{-1}_3  \br{\frakc_{10} + 2\frakc_1} \frakc_2 \frakc_9 \frakc_{5}\br{\frakc_7 + \frakc_8 +\frakc_9} \log(n+n_0)}\sum_{i=0}^{n-1}  \frac{\chi(i+1,n)}{(i+ n_0)^{2-\br{2\fraka+2\kappa_1+\kappa_3}}},
	}
	where in the second inequality we have used the bound derived in Step 1.~We bound the final summation using Lemma~\ref{lemma:recursion_bound}. In the notation of Lemma~\ref{lemma:recursion_bound}, we take $a_j = \beta_j = \frac{\beta}{(n_0 + j)^{1-\fraka}}$, $\rho=1-\fraka$, $b_i = \frac{1}{(i+ n_0)^{2-\br{2\fraka+2\kappa_1+\kappa_3}}}$, $b=1$, $\rho' = 2-\br{2\fraka+2\kappa_1+\kappa_3}$, $a=\beta$.~There are two cases:

    Case A: $\fraka = 0$, equivalently $\rho=1$. In this case we use Part 1 of Lemma~\ref{lemma:recursion_bound}, which requires $\rho' = 2-\br{2\kappa_1+\kappa_3} \in (1,2]$, i.e. $2\kappa_1+\kappa_3 < 1$. It also requires $\beta \ge 2\br{1-\br{2\kappa_1+\kappa_3}}$.

    Case B: $\fraka >0$, equivalently $\rho<1$. In this case we use Part 2 of Lemma~\ref{lemma:recursion_bound}, which requires $\rho'>\rho$, i.e., $\fraka + 2\kappa_1+\kappa_3<1$. It also requires $n_0 \ge \br{\frac{2\br{1-\br{\fraka + 2\kappa_1+\kappa_3}}}{\beta}}^{\frac{1}{\fraka}}$.

The above requirements are guaranteed by Condition~\ref{con:sa1}. Hence Lemma~\ref{lemma:recursion_bound} gives
\nal{
\|T_2(n)\| & \le 2\beta^2 \br{\frakc_1 \frakc_2 \frakc^{-1}_3 \br{\frakc_{10} + 2\frakc_1} \frakc_2 \frakc_9 \frakc_{5}\br{\frakc_7 + \frakc_8 +\frakc_9} \log(n+n_0)}  \frac{(n+n_0)^{1-\fraka}}{\beta (n+ n_0)^{2-\br{2\fraka+2\kappa_1+\kappa_3}}}\\
& =  \frac{2\beta \br{\frakc_1 \frakc_2 \frakc^{-1}_3 \br{\frakc_{10} + 2\frakc_1} \frakc_2 \frakc_9 \frakc_{5}\br{\frakc_7 + \frakc_8 +\frakc_9} \log(n+n_0)} }{(n+ n_0)^{1-\br{\fraka+2\kappa_1+\kappa_3}}}.
}

\end{proof}

\begin{lemma}\label{lemma:T3}(Bounding $T_3(n)$)
For every $n\in\bN$
\nal{
\|T_3(n)\| \le \frac{2\frakc^2_2\frakc_1 \frakc_9 \frakc_4\br{ \frakc_{10} + 2 \frakc_1}}{\frakc^2_3 (n+n_0)^{1-\br{2\kappa_1 +\kappa_2}  }}.
} 
\end{lemma}
\begin{proof}
Recall that $T_3(n) = \sum_{i=0}^{n-1} \beta_i \chi(i+1,n) \br{H^{(\zeta_{i+1},x_i)}(y_{i+1}) - H^{(\zeta_i,x_{i})}(y_{i+1})}$. We have, 
	\nal{
		\|T_3(n) \| & \le  \sum_{i=0}^{n-1} \beta_i \chi(i+1,n)  \|\br{H^{(\zeta_{i+1},x_i)}(y_{i+1}) - H^{(\zeta_i,x_{i})}(y_{i+1})}\| \\
		& \le   \frac{\frakc^2_2\frakc_1 \frakc_9 \frakc_4 \br{ \frakc_{10} + 2 \frakc_1} }{\frakc^2_3} \sum_{i=0}^{n-1} \beta_i \chi(i+1,n)  \frac{1}{(i+n_0)^{1-\br{2\kappa_1 +\kappa_2}  }}\\
		& =   \frac{\beta \frakc^2_2\frakc_1 \frakc_9 \frakc_4 \br{ \frakc_{10} + 2 \frakc_1}}{\frakc^2_3} \sum_{i=0}^{n-1}  \frac{1}{(i+n_0)^{2-\br{\fraka + 2\kappa_1 +\kappa_2}  }} \Pi_{j=i+1}^{n} \br{1-\beta_j},        
	}
	where the second inequality follows from Lemma~\ref{lemma:sensitivity_v_x}~\eqref{lemma:6-1-3} and we have used $\beta_i = \frac{\beta}{\br{i + n_0}^{1-\fraka}}$.
    We now bound the final summation using Lemma~\ref{lemma:recursion_bound}. In the notation of Lemma~\ref{lemma:recursion_bound}, take $a_j = \beta_j$, $a=\beta,\rho = 1-\fraka$, $b_i = \frac{1}{(i+n_0)^{2-\br{\fraka + 2\kappa_1 +\kappa_2}  }}$, $\rho' = 2-\br{\fraka + 2\kappa_1 +\kappa_2}$. There are two cases:

    Case A: $\fraka = 0$, equivalently $\rho=1$. In this case we use Part 1 of Lemma~\ref{lemma:recursion_bound}, which requires $\rho' = 2-\br{\fraka + 2\kappa_1 +\kappa_2} \in (1,2]$, i.e. $2\kappa_1 +\kappa_2<1$. It also requires $\beta \ge 2\br{\rho'-1 } = 2 \br{1-\br{\fraka + 2\kappa_1 +\kappa_2}}$, i.e. $\beta \ge 2 \br{1-\br{\fraka + 2\kappa_1 +\kappa_2}}$.
    
    Case B: $\fraka >0$, equivalently $\rho<1$: In this case we use Part 2 of Lemma~\ref{lemma:recursion_bound}, which requires $\rho' > \rho$, i.e. $2\kappa_1 + \kappa_2 <1$. It also requires $n_0 \ge \br{\frac{1-\br{2\kappa_1 + \kappa_2}}{\beta}}^{1\slash \fraka}$.

It is easily verified that the above are satisfied under Condition~\ref{con:sa1}.~Substituting the bound on summation term yields:
\nal{
\|T_3(n)\| & \le  \frac{2\beta \frakc^2_2\frakc_1 \frakc_9 \frakc_4 \br{ \frakc_{10} + 2 \frakc_1}}{\frakc^2_3}\br{\frac{(n+n_0)^{1-\fraka}}{\beta(n+n_0)^{2-\br{\fraka + 2\kappa_1 +\kappa_2}  }} }\\
& = \frac{2\frakc^2_2\frakc_1 \frakc_9 \frakc_4 \br{ \frakc_{10} + 2 \frakc_1}}{\frakc^2_3 (n+n_0)^{1-\br{2\kappa_1 +\kappa_2}  }}.
}
This completes the proof.
\end{proof}

%% file: appendix_Q-learning.tex
\section{Boltzmann Q-learning}\label{app:pf_q_learning}
This section proves the results for Boltzmann Q-learning stated in Section~\ref{sec:BQL}.
We first verify that Boltzmann Q-learning fits into the stochastic approximation
framework~\eqref{def:gen_SA}. We then use the concentration result from Theorem~\ref{th:main_concentration} to prove the high-probability convergence bound for $\{Q_n\}$, and finally
derive the regret bound. The regret proof uses the decomposition of the instantaneous regret into the two
terms~\eqref{adhoc:2_1} and~\eqref{adhoc:2_2}, introduced below in Section~\ref{subsec:BQL_regret}. Section~\ref{pf:prop_5} bounds both terms, Section~\ref{sec:regret_BQL_detailed} combines these bounds to prove Theorem~\ref{th:regret_bound}, and
Section~\ref{sec:con_n0_sm} records the additional requirements on $n_0$ that are specific to
Boltzmann Q-learning. Throughout Sections~\ref{app:pf_q_learning} and~\ref{app:aux_q-learning}, $n_0$ is assumed to satisfy
Condition~\ref{con:n0} together with the Boltzmann-specific conditions stated in
Section~\ref{sec:con_n0_sm}.

\subsection{Assumptions, conditions and definitions}
In this section, we show that Boltzmann Q-learning can be analyzed within
the stochastic approximation framework~\eqref{def:gen_SA}. We impose a mild recurrence
condition on the underlying MDP in Assumption~\ref{assum:16} and require the hyperparameters
to satisfy Condition~\ref{con:hyper1}.~Proposition~\ref{prop:q_learn_assum} then identifies the constants appearing in the general stochastic approximation framework and verifies the assumptions
needed to apply Theorem~\ref{th:main_concentration} to the Boltzmann Q-learning iterates.

We begin with some notation. Throughout our discussions involving MDPs and online Q-learning algorithms, for a state $s\in\cS$ we will denote by $\mathcal{T}_{hit.,s}$ the hitting time of state $s$.~In the analysis of online Q-learning, we consider the Markov chain $\{(s_n,a_n)\}$ induced by a
fixed pair $(\zeta,Q)$. For Boltzmann Q-learning, the control variable $\zeta$ is the temperature parameter $\lambda$, whereas for \seg~Q-learning, it is the pair $(\eps,\lambda)$. Thus, in the
Boltzmann case, the control space is $\mathcal{Z}=\mathbb{R}_+$, the Markov
chain in the stochastic approximation framework is $y_n=(s_n,a_n)$ and the state space is $\mathcal{Y}=\mathcal{S}\times\mathcal{A}$.~Accordingly, $\cT_{\mathrm{hit},(s,a)}$ denotes the hitting time of the
state-action pair $(s,a)$ for the chain $\{y_n\}$.

\begin{definition}
	Diameter of the MDP $\cM$ is defined as $\cD(\cM) :=\sup_{\pi} \sup_{s,s'\in\cS} \bE\left\{ \mathcal{T}_{hit.,s'}| s_0 = s\right\} $, where supremum is taken over the class of all stationary policies for the MDP $\cM$.
\end{definition}
For \(Q\in\mathbb{R}^{\mathcal{S}\times\mathcal{A}}\) and \(\lambda>0\), we
write \(\sigma(Q/\lambda)\) for the stationary Boltzmann policy that samples $a_n \sim \sigma\br{Q(s_n,\cdot)\slash \lambda}$ by $\sigma(Q\slash \lambda)$.

We also define the compact set $\cQ:= \left\{Q \in \bR^{\cS\times \cA}: 0\le Q(s,a) \le \frac{R_{\max}}{1-\gamma},\forall (s,a)\right\}$.~We impose the following recurrence condition on the MDP $\cM$ while analyzing the Boltzmann Q-learning algorithm. 
\begin{assumption}\label{assum:16}
	Define
	\nal{
		\mu\s_{\min,\cS} := \min_{s\in\cS}\inf_{\substack{ \lambda\in\bR_+ \\ Q\in\cQ}}\sum_{a\in\cA} \mu^{(\lambda,Q)}(s,a)
	}    
    as the minimum value of the stationary probability over all the states. Then $\mu\s_{\min,\cS} >0$.
\end{assumption}
We now show that Boltzmann Q-learning fits into the stochastic approximation
framework introduced in Section~\ref{sec:SA}. The next proposition identifies the
corresponding constants and the values of \(\kappa_1^{(B)},\kappa_2^{(B)},\kappa_3^{(B)}\).

\begin{proposition}\label{prop:q_learn_assum}
The Boltzmann Q-learning recursion can be written in the stochastic approximation form~\eqref{def:gen_SA}. Moreover, if the MDP $\cM$ satisfies
Assumption~\ref{assum:16}, then Assumptions~\ref{assum:1}-~\ref{assum:14} hold with
\nal{
& \frakc_1 = \cD(\cM),\quad\frakc_2 = 1-\gamma,\quad\frakc_{3} = \frac{(1-\gamma)}{|\cA|}\mu\s_{\min,\cS},\quad
\frakc_4 = \frac{|\cA|R_{\max}}{1-\gamma}\frakb,\quad \frakc_5 = |\cS|^2|\cA|^2 \frakb\\
& \frakc_6 =  1+\gamma,\quad\frakc_7 = \frakc_8 = \frakc_9 = \frac{R_{\max}}{1-\gamma},\quad\frakc_{10} = \mu\s_{\min,\cS}\cD(\cM).
}
Furthermore, the functions $L_1: \bN \mapsto \bR_+$, and $L_2: \bR_+ \mapsto \bR_+$ may be chosen as
\nal{
L_1(n) = \frac{|\cA|R_{\max}}{\br{1-\gamma}\br{n+n_0}}\frakb, \qquad L_2(\lambda) = \frac{|\cS|^2 |\cA|^2}{\lambda}.
}
The contraction margin satisfies
\nal{
\tilde{\alpha}(\lambda) = (1-\gamma)\mu_{\min}(\lambda).  
}
Also
\nal{
\kappa\s_1 = \frac{R_{\max}}{1-\gamma}\frakb, \kappa\s_2 = 0, \kappa\s_3 = 0.
}
\end{proposition}
\begin{proof}
We first verify that the online Q-learning recursion is an instance of the
stochastic approximation recursion~\eqref{def:gen_SA}, and Boltzmann Q-learning algorithm can be analyzed within the general framework of stochastic approximation that was introduced in Section~\ref{sec:SA}. The controlled Markov process $\{y_n\}$ is taken to be $y_n = (s_n,a_n)$, while $\zeta_n = \lambda_n$. Temperature sequence is chosen as $\lambda_n = \frac{1}{\frakb \log(n+n_0)},~\frakb >0$.~For fixed \((\lambda,Q)\), the stationary policy $\sigma(Q\slash \lambda)$
induces a Markov chain on \(\mathcal{S}\times\mathcal{A}\) with transition
kernel
\al{
	p^{(\lambda,Q)}((s,a),(s',a'))  = p(s,a,s') \sigma(Q(s',\cdot)\slash \lambda)(a'),~\mbox{ where }(s,a), (s',a')\in\cS\times\cA.\label{def:CTP}
}
Define
\al{
	[F(Q,s,a)](i,u) : &= \ind{s=i,a=u} \left(r(i,u)+\gamma \sum_{s'\in\cS}p(i,u,s')\max_{a'} Q(s',a')-Q(i,u) \right) \notag\\
	&+ Q(i,u), \label{def:F_QL} \mbox{ where }~(i,u)\in \cS \times \cA.
}
Also define
\al{
	M_{n+1}(s,a) & = \gamma \ind{s_n=s,a_n=a} \left(\max_{a'\in\cA} Q_n(s_{n+1},a') - \sum_{s'\in\cS}p(s,a,s')\max_{a'\in\cA} Q_n(s',a')\right), \label{def:M_QL}\\
	& n=0,1,\ldots.\notag
}
With these definitions, the Q-learning iterations can be written componentwise as
\al{
	Q_{n+1}(s,a)
	=Q_n(s,a)+\beta_n\left([F(Q_n,s_n,a_n)](s,a)-Q_n(s,a)+M_{n+1}(s,a)\right),
	~(s,a)\in\cS\times\cA.\label{def:Q_iterations}
}
It is easily verified that $\{M_{n}\}$ is a martingale difference sequence w.r.t. the filtration generated by the process $\{(s_n,a_n)\}$. The maps $F(\cdot,s,a)$ are seen to be Lipschitz, and $\frakc_6$ can be taken to be $1+\gamma$.~Thus, with $y_n=(s_n,a_n)$ and $\zeta_n=\lambda_n$, Boltzmann Q-learning is an instance of the stochastic approximation recursion~\eqref{def:gen_SA}.

Recall $\cQ= \left\{Q \in \bR^{\cS\times \cA}: 0\le Q(s,a) \le \frac{R_{\max}}{1-\gamma},\forall (s,a)\right\}$.~We will firstly show that $Q_n \in \cQ$, and hence the iterates remain within a compact set.~Assume that $Q_m \in\cQ, m\le n$. Now $|Q_{n+1}(s_n,a_n)| \le (1-\beta_n)|Q_n(s_n,a_n)| + \beta_n | \br{r(s_n,a_n) + \gamma \max_{a\in\cA}Q_{n}(s_{n+1},a) }| \le (1-\beta_n)\frac{R_{\max}}{1-\gamma} + \beta_n \br{R_{\max} +\gamma\frac{R_{\max}}{1-\gamma}  } = \frac{R_{\max}}{1-\gamma}$. Non-negativity of $Q_n(s,a)$ is immediate from the recursion.~Hence $\frakc_8$ can be taken to be $\frac{R_{\max}}{1-\gamma}$.~This also shows boundedness of $M_n$, so that $\frakc_7 = \frac{R_{\max}}{1-\gamma}$. Boundedness of $F(Q_n,s_n,a_n)$ also follows and we can set $\frakc_9 = R_{\max} + \frac{\gamma R_{\max}}{1-\gamma}= \frac{R_{\max}}{1-\gamma}$.~Hence we have shown Assumption~\ref{assum:9} holds.

Assumption~\ref{assum:8}, i.e. the Lipschitz property of $F$~\eqref{def:F_QL} is easily seen to hold.~The contraction property of Assumption~\ref{assum:7} is proved in Lemma~\ref{lemma:contraction}. In particular, the stationary averaged map has contraction factor 
\nal{
\alpha(\lambda) = 1 -  (1-\gamma) \mu_{\min}(\lambda),
}
where 
\nal{
\mu_{\min}(\lambda)  : = \inf_{Q\in\cQ}\min_{(s,a)\in\cS\times\cA} \mu^{(\lambda,Q)}(s,a),~\lambda \in \bR_+.
}
Therefore
\[
\tilde{\alpha}(\lambda)
=
1-\alpha(\lambda)
=
(1-\gamma)\mu_{\min}(\lambda),
\]
and hence Assumption~\ref{assum:at} holds with $\frakc_2 = 1-\gamma$. 

The lower bound on $\sigma(Q(s,\cdot)\slash \lambda)(a)$ from Lemma~\ref{lemma:lower_bound_action_prob} yields the following lower bound on action probabilities,
\nal{
	\min_{a\in\cA,s\in\cS} \min_{Q \in\cQ}\sigma(Q(s,\cdot)\slash \lambda)(a) \ge  \frac{1}{|\cA|\exp\br{\frac{R_{\max}}{\lambda (1-\gamma)}}}.
}	
Since $\mu^{(\lambda,Q)}(s,a) = \br{\sum_{b\in\cA}\mu^{(\lambda,Q)}(s,b)} \sigma(Q(s,\cdot)\slash \lambda)(a)$, this gives us $\mu_{\min}(\lambda) \ge  \frac{\mu\s_{\min,\cS}}{|\cA|\exp\br{\frac{R_{\max}}{\lambda (1-\gamma)}}}$, where $\mu\s_{\min,\cS}$ is as in~Assumption~\ref{assum:16}. This shows $\mu_{\min}(\lambda)>0$.~Choose a state-action pair $(s\ust,a\ust) \in \cS\times \cA$ as designated state-action pair for the process $\{y_n\} = \{(s_n,a_n)\}$. Under the application of policy $\sigma(Q\slash \lambda)$, we can show that the mean hitting time to $(s\ust,a\ust)$ can be upper-bounded as $\frac{\cD(\cM)}{\mu_{\min}(\lambda)}$, where $\cD(\cM)$ is the MDP diameter. Hence, the hitting-time bound in Assumption~\ref{assum:2} holds with $\frakc_1 = \cD(\cM)$.~Since $\lambda_n = \frac{1}{\frakb \log(n+n_0)}$, we have $\at(\lambda_n)\ge \frac{(1-\gamma)\mu\s_{\min,\cS}}{|\cA|} \frac{1}{(n+n_0)^{\frakb R_{\max}\slash (1-\gamma)}}$. Hence Assumption~\ref{assum:3}~(a) holds with $\kappa_1 = \frac{\frakb R_{\max}}{1-\gamma}$ and $\frakc_{3} = \frac{(1-\gamma)}{|\cA|}\mu\s_{\min,\cS}$.

To verify Assumption~\ref{assum:14}, we note that
\nal{
	 \|\mu^{(\lambda,Q)} -  \mu^{(\lambda',Q')}\|_{\infty} 
	&\le  \frac{\max_{(s,a),(s',a')\neq (s,a)} \bE\br{ \cT_{hit.,(s',a')}| s_0 = s, a_0 = a }  }{2} \| p^{(\lambda,x)} - p^{(\lambda',x')} \|_{\infty}\\
	& \le \frac{\cD(\cM)}{\br{\cA|\exp\br{\frac{R_{\max}}{\lambda (1-\gamma)}}}^{-1}} \| p^{(\lambda,x)} - p^{(\lambda',x')} \|_{\infty}\\
	& = \cD(\cM)\frac{\mu\s_{\min,\cS}}{ \mu_{\min}(\lambda)} \| p^{(\lambda,x)} - p^{(\lambda',x')} \|_{\infty},    
}   
where the first inequality follows from Theorem 3.1 of~\citep{cho2000markov}, the second can be shown, while the third follows since we have shown $\mu_{\min}(\lambda)\ge \mu\s_{\min,\cS} \br{|\cA|\exp\br{\frac{R_{\max}}{\lambda (1-\gamma)}}}^{-1}$. Hence Assumption~\ref{assum:14} holds with $\frakc_{10} = \mu\s_{\min,\cS}\cD(\cM)$.

We will now show that Assumption~\ref{assum:3} (c) on the transition kernel also holds. Upon using the expression for $p^{(\lambda,Q)}$ from~\eqref{def:CTP} we get,
\nal{
	& \sum_{(s',a') \in\cS \times \cA} |p^{(\lambda,Q)}((s,a),(s',a')) - p^{(\lambda,Q')}((s,a),(s',a'))| \\
	& \le \sum_{(s',a') \in\cS \times \cA} p(s,a,s') | \sigma(Q(s',\cdot)\slash \lambda)(a') - \sigma(Q'(s',\cdot)\slash \lambda)(a')|\\
	& \le |\cS||\cA|\|Q-Q'\|_{\infty}\sum_{(s',a') \in\cS \times \cA} p(s,a,s') \frac{1}{\lambda}\\
	& \le \frac{|\cS|^2|\cA|^2}{\lambda} \|Q-Q'\|_{\infty},
}
the third inequality follows from the fundamental theorem of calculus after using the expression for the derivatives of $\sigma(Q\slash \lambda)$ w.r.t. $Q$ that was derived in Lemma~\ref{lemma:diff_softmax_Q}. Thus, we have shown that Assumption~\ref{assum:3} (c) holds with $L_2(\lambda) = \frac{|\cS|^2|\cA|^2}{\lambda}$. Since $\lambda_n = \frac{1}{\frakb \log(n+n_0)},~\frakb >0$, this gives us $L_2(\lambda_n) \le  |\cS|^2|\cA|^2 \frakb \log(n+n_0)$. Hence we can take $\kappa_3$ to be $0$ and $\frakc_{5}=|\cS|^2|\cA|^2\frakb$ in Assumption~\ref{assum:3} (c).

We will now show that Assumption~\ref{assum:3} (b) holds. For $y = (s,a), y' = (s',a')$ we have ,
\nal{
|p^{(\lambda_n,Q)}(y,y') - p^{(\lambda_{n+1},Q)}(y,y')| &  = p(s,a,s') |   \sigma(Q(s',\cdot)\slash \lambda_n)(a') -  \sigma(Q(s',\cdot)\slash \lambda_{n+1})(a') | \\
& \le  \frac{p(s,a,s')}{\lambda^2_n} \|Q(s',\cdot)\|_{1} |\lambda_n - \lambda_{n+1}| \\
&=  p(s,a,s')\frakb^2 \log(n+n_0)^2 \|Q(s',\cdot)\|_{1} |\lambda_n - \lambda_{n+1}|  \\
& \le \frac{p(s,a,s')\frakb  \|Q(s',\cdot)\|_{1}}{n+n_0}\\
& \le \frac{p(s,a,s')}{(n+n_0)}  \left\{\frac{\frakb |\cA| R_{\max}}{(1-\gamma)}  \right\},
}
to get first inequality we use the fundamental theorem of calculus and the bound on the magnitude of the derivative of probabilities $\sigma(Q\slash \lambda)$ w.r.t. $\lambda$ from Lemma~\ref{lemma:diff_softmax_Q}. The second inequality follows from $\lambda_n = \frac{1}{\frakb \log(n+n_0)}$, so that $\lambda_n - \lambda_{n+1} \approx \frac{1}{\frakb (n+n_0)\log(n+n_0)^2 }$. Upon performing a summation over state $s'$ we get 
\nal{
\sum_{y'\in\cY} |p^{(\lambda_n,Q)}(y,y') - p^{(\lambda_{n+1},Q)}(y,y')| \le  \frac{\frakb |\cA| R_{\max}}{(1-\gamma)(n+n_0)}  .
}		
This shows Assumption~\ref{assum:3} (b) holds with $\frakc_4=\frac{\frakb |\cA|R_{\max}}{1-\gamma}$ and $\kappa_2 =0$.
\end{proof}
\begin{lemma}\label{lemma:contraction}
For each $(\lambda,Q)\in \bR_+ \times \cQ$, the operator $\sum\limits_{(s,a)\in\cS \times \cA}\mu^{(\lambda,Q)}(s,a) F(\cdot,s,a)$ is a contraction in the $\|\cdot\|_{\infty}$ norm, with contraction factor given by 
\nal{
\alpha(\lambda) = 1 -  (1-\gamma) \mu_{\min}(\lambda),
}
where
\nal{
    \mu_{\min}(\lambda)  = \inf_{Q\in\cQ}\min_{(s,a)\in\cS\times\cA} \mu^{(\lambda,Q)}(s,a).
}
\end{lemma}
\begin{proof}
	For $Q,Q', \tilde{Q}\in\bR^{\cS\times\cA}$, the magnitude of the $(i,b)$-th component of the vector $\sum_{(s,a)\in\cS\times \cA}\mu^{(\lambda,\tilde{Q})}(s,a)  \left[ F(Q,s,a) - F(Q',s,a) \right]$ is bounded as follows,
	\nal{
		& \Big| \br{1-\mu^{(\lambda,Q)}(i,b) } \left[ Q(i,b)-Q'(i,b) \right]  \\
        &  + \mu^{(\lambda,Q)}(i,b)  \gamma \sum_{s'\in\cS} p(i,b,s') \max_{a'\in\cA} Q(s',a') - \max_{a'\in\cA} Q'(s',a')\Big|\\
		& \le \br{1-\mu^{(\lambda,Q)}(i,b) } | Q(i,b)-Q'(i,b)  | \\
        & + \gamma \mu^{(\lambda,Q)}(i,b) \sum_{s'\in\cS} p(i,b,s') | \max_{a'\in\cA} Q(s',a') - \max_{a'\in\cA} Q'(s',a')|\\
		& \le \br{1-\mu^{(\lambda,Q)}(i,b) } \max_{a'\in\cA}| Q(i,a')-Q'(i,a')  | \\
        & + \gamma \mu^{(\lambda,Q)}(i,b) \sum_{s'\in\cS} p(i,b,s') \max_{a'\in\cA} |  Q(s',a') - Q'(s',a')|\\
		& \le  \br{1-(1-\gamma)\mu^{(\lambda,Q)}(i,b) } \|Q-Q'\|_{\infty}\\
		& \le  \br{1-\at(\lambda)} \|Q-Q'\|_{\infty}.
	}
	This completes the proof.
\end{proof}
In summary, Proposition~\ref{prop:q_learn_assum} allows us to use the concentration result of Theorem~\ref{th:main_concentration} to Boltzmann Q-learning.~Condition~\ref{con:sa1} reduces to the following conditions.~Henceforth we will assume that the hyperparameters are chosen so that Condition~\ref{con:hyper1} is satisfied. 

\begin{condition}\label{con:hyper1}
The quantities $\fraka,\kappa\s_1$ satisfy
\nal{
\fraka,\kappa\s_1 \in [0,1], 2\fraka + 6\kappa\s_1<1.
}
Furthermore, consider the following four cases (these four possibilities arise while applying Lemma~\ref{lemma:recursion_bound}): 

Case A: $\fraka = 0$, then we have,
\nal{
\beta  \ge 2(1-2\kappa\s_1),\qquad 2\kappa_1 < 1.
}
Case B: $\fraka >0$, then we have
\nal{
\fraka + \kappa\s_1 < 1,\qquad \fraka + 2\kappa\s_1<1,
}   
moreover 
$n_0 \ge \br{ \frac{ 2 \left[ (1-\fraka)-2\kappa_1\right]  }{\beta}   }^{\frac{1}{1-\rho}}$,
$n_0 \ge \br{\frac{2\br{ 1-\fraka -\kappa_1 }}{\beta}}^{1\slash \fraka }$
$n_0 \ge \br{\frac{2\br{1-\br{\fraka + 2\kappa_1}}}{\beta}}^{\frac{1}{\fraka}}$
$n_0 \ge \br{\frac{1-\br{2\kappa_1}}{\beta}}^{1\slash \fraka}$.

Case C: $\fraka = \kappa_1$: We have
\nal{
7\fraka < 1, \qquad \beta \frakc_3 \ge 1- 7\fraka, \qquad \beta \frakc_3 \ge 2- 8\fraka.
}

Case D: $\fraka > \kappa\s_1$: We have $\fraka+6\kappa\s_1 < 1$. Moreover,

$n_0 \ge \br{\frac{1-\br{\fraka + 6\kappa_1 } }{\beta \frakc_3}}^{1\slash (\fraka - \kappa_1)}$, $n_0 \ge \br{\frac{2\br{1-\br{\fraka + 4\kappa_1}}}{\beta \frakc_3}}^{1\slash (\fraka - \kappa_1)}$, $n_0 \ge \br{\frac{2\br{1-\br{4\kappa_1}}}{\beta \frakc_3}}^{1\slash (\fraka - \kappa_1)}$.
\end{condition}

\subsection{Concentration of $\{Q_n\}$ (Theorem~\ref{th:Q_conc})}\label{sec:th_Q_conc}
By Proposition~\ref{prop:q_learn_assum}, if the MDP $\cM$ satisfies Assumption~\ref{assum:16} and the
hyperparameters satisfy Condition~\ref{con:hyper1}, then Boltzmann Q-learning satisfies
Assumptions~\ref{assum:1}-~\ref{assum:14} of the stochastic approximation framework. We may therefore
apply Theorem~\ref{th:main_concentration} to analyze $\{Q_n\}$.~Thus, on the high-probability event from Theorem~\ref{th:main_concentration}, for all $n\in\mathbb{N}$,
\nal{
	\|Q_n - Q\ust\|_{\infty} \le 2\log(n+n_0)\frac{\br{n+n_0}^{2\kappa\s_1+\kappa\s_3}}{\frakc_{3}} \br{C\s_4 g(n,\delta) + C\s_5 g_1(n)+C\s_6 g_2(n)}.
} 
We now analyze this bound in detail.~Upon substituting for $C\s_4,C\s_5,C\s_6$ from~\eqref{def:C4-6} the above becomes,
\al{
	\|Q_n - Q\ust\|_{\infty} & \le \frac{4|\cA|^3|\cS|^2 \frakb}{(1-\gamma)^2  \mu\s_{\min,\cS}} R_{\max} \cD(\cM)    \log(n+n_0)\notag \\
    & \times \br{n+n_0}^{2\kappa\s_1+\kappa\s_3}\br{C\s_1 g(n,\delta) + C\s_2 g_1(n)+C\s_3 g_2(n)},\label{ineq:q-qstar_adhoc}
    } 
where $C\s_1,C\s_2,C\s_3$~\eqref{def:C1-C3} reduce to, 
\nal{
C\s_1 & = 16 R_{\max}|\cS| |\cA| \cD(\cM) \sqrt{\beta},\\
C\s_2 & = 6\beta |\cS|^2 |\cA|^2 \cD(\cM)^2  R_{\max} \br{\frac{\mu_{\min,\cS} }{|\cA|} + \frac{2 R_{\max}}{1-\gamma} },\\
C\s_3 & =  2 D(\cM)^2 R^2_{\max} \frakb \br{ \mu\s_{\min,\cS}  + 2|\cA|}.
}
Moreover,
\nal{
\br{n+n_0}^{2\kappa\s_1 +\kappa\s_3} g(n,\delta) & = \tilde{O}\br{ (n+n_0)^{-\br{ .5 - (\fraka + 3\kappa\s_1)  }} },\\
\br{n+n_0}^{2\kappa\s_1 +\kappa\s_3} g_1(n) & = \tilde{O}\br{ (n+n_0)^{-\br{1- (2\fraka + 4\kappa\s_1 )  } }  },\\
\br{n+n_0}^{2\kappa\s_1 +\kappa\s_3} g_2(n) & = 
\tilde{O}\br{ (n+n_0)^{-\br{1- (\fraka + 4\kappa\s_1 )  } }  }.
}
Upon substituting these in the bound~\eqref{ineq:q-qstar_adhoc}, and keeping only the explicit
dependence on $\beta$, $n$, and $1-\gamma$ gives
\nal{
\|Q_n - Q\ust\|_{\infty} & \le \frac{\tilde{O}\br{\sqrt{\beta}(n+n_0)^{-\br{ .5 - (\fraka + 3\kappa\s_1)  }}} }{(1-\gamma)^{2}}+ \frac{\tilde{O}\br{\beta (n+n_0)^{-\br{1- (2\fraka + 4\kappa\s_1 )  } }  }}{(1-\gamma)^3} \\
& + \frac{\tilde{O}\br{ (n+n_0)^{-\br{1- (\fraka + 4\kappa\s_1 )  } }  }}{(1-\gamma)^2}.
}

We will now choose the parameters $\fraka,\kappa\s_1$ so as to optimize the convergence rates w.r.t. $n$. The parameters $\fraka,\kappa\s_1$ must satisfy Condition~\ref{con:hyper1}. We consider the following two possibilities:

\begin{enumerate}[1.)]
    \item \label{case:1_main_paper} $\fraka=0$: It is required that $0< \kappa\s_1 < \frac{1}{6}$, and the rate is given by $\tilde{O}\br{n^{-\min\left\{ .5 - 3\kappa\s_1,1-4\kappa\s_1\right\} }}$, which is optimized by letting $\kappa\s_1 \to 0$. We choose $\beta = 2$, which gives $\frac{\tilde{O}\br{n^{-\br{ .5 - 3\kappa\s_1} }}}{\br{1-\gamma}^2}$.
    \item $\fraka >0$: We further have the following possibilities:
    \begin{itemize}
        \item $\fraka = \kappa\s_1$: It is required that $\fraka < 1\slash 7$, and the rate is given by $\tilde{O}\br{n^{-\min\left\{ .5 - 4\fraka,1-6\fraka,1-5\fraka\right\} }}$. Note that we must also have $\fraka>0$ since $\kappa\s_1>0$.~The condition on $\beta$ becomes $\beta \ge \frac{\max \{1-7\fraka,2-8\fraka\}}{\mu\s_{\min,\cS}(1-\gamma)}$ and we choose $\beta = \frac{2|\cA|}{\mu\s_{\min,\cS}(1-\gamma)}$.~This gives us,
\nal{
\|Q_n - Q\ust\|_{\infty} &\le \frac{\tilde{O}\br{(n+n_0)^{-\br{ .5 -  4\fraka  }}} }{(1-\gamma)^{2.5}}+ \frac{\tilde{O}\br{ (n+n_0)^{-\br{1-6\fraka   } }  }}{(1-\gamma)^4}\\
&+ \frac{\tilde{O}\br{ (n+n_0)^{-\br{1-5\fraka   } }  }}{(1-\gamma)^2},
}
and optimal rate is obtained by letting $\fraka \to 0^+$.
        \item $\fraka > \kappa\s_1$: It is required that $\fraka + 6\kappa\s_1<1$, $2\kappa\s_1<1+\fraka$, and the rate is given by $\tilde{O}\br{\br{n+n_0}^{-\min\left\{.5 - (\fraka + 3\kappa\s_1),1- (2\fraka + 4\kappa\s_1 ),1- (\fraka + 4\kappa\s_1 ) \right\} }}$. To obtain optimal rates we let $\fraka,\kappa\s_1 \to 0^{+}$ with $\fraka$ slightly greater than $\kappa\s_1$ which gives the rate $\tilde{O}\br{\br{n+n_0}^{-\br{.5 - \br{\fraka+3\kappa\s_1}   }}}\slash \br{1-\gamma}^2$
        \item $\fraka < \kappa\s_1$: It is required that $\fraka + \kappa\s_1<1$, $2\kappa\s_1<1+\fraka$, and the rate is given by $\tilde{O}\br{\br{n+n_0}^{-\min\left\{.5 - (\fraka + 3\kappa\s_1),1- (2\fraka + 4\kappa\s_1 ),1- (\fraka + 4\kappa\s_1 ) \right\} }}$. Optimal rate is approached by setting $\fraka = \kappa\s_1\slash 2$, and letting $\kappa\s_1 \to 0^+$.~The rate is then given by $\frac{\tilde{O}\br{ (n+n_0)^{-\br{.5 - \frac{7}{2}\kappa\s_1}} }}{(1-\gamma)^2}$. 
    \end{itemize}
\end{enumerate}
We note that only the first case, i.e. with $\fraka = 0$, is presented in the main paper.~This completes the proof.

\begin{remark}
We note that $\kappa\s_1 = \frac{\frakb R_{\max}}{1-\gamma}$. Since $\kappa\s_1 \le 1$, note that $\frakb$ would be tuned according to $\gamma$ so as to satisfy $\kappa\s_1 = \frac{\frakb R_{\max}}{1-\gamma} = O(1)$. This would ensure that it does not blow-up in the limit $\gamma \to 1$. This is required in order to obtain a decent convergence result.
\end{remark}

\subsection{Regret of Boltzmann~Q-learning}\label{subsec:BQL_regret}
Here we provide a proof sketch of the regret bound of Boltzmann Q-learning. More specifically, this subsection gives the regret decomposition used to prove Theorem~\ref{th:regret_bound}.~The
two resulting terms are bounded in Section~\ref{pf:prop_5}, and the final regret bound is
obtained in Section~\ref{sec:regret_BQL_detailed} by summing these estimates over time.

\begin{definition}
	For the stationary policy $\sigma(Q\slash \lambda)$ we let $V^{(\lambda,Q)}: \cS \mapsto \bR$ denote its value function, defined by $V^{(\lambda,Q)}(s) := \bE\left(\sum_{n=0}^{\infty} \gamma^n r(s_n,a_n)| s_0 = s \right)$ where expectation is taken under the measure induced by $\sigma(Q\slash \lambda)$. 
\end{definition}

\textit{Proof sketch}: For each time $n\in\bN$, define the instantaneous regret (the regret from time $n$ onwards) 
\nal{
\reg^{(n)} : =  (1-\gamma)\br{V\ust(s_{n}) -  \bE\left[\sum_{m=n}^{\infty} \gamma^{m-n} r(s_m,a_m) \Big| \cF_{n} \right]  }.
}
Then the cumulative regret up to time $N$ is $\cR_N  = \sum_{n=1}^{N} \reg^{(n)}$.~We will bound $\regn$ with high probability and then sum the resulting bounds over $n$.~We decompose $\regn$ as
\begin{subequations}
		\al{
&	\regn   =  (1-\gamma)\br{V\ust(s_{n}) - V^{(\lambda_n,Q_n)}(s_n)   }\label{adhoc:2_1}		\\
& +(1-\gamma) \br{ V^{(\lambda_n,Q_n)}(s_n) -  \bE\left[\sum_{m=n}^{\infty} \gamma^{m-n} r(s_m,a_m) \Big| \cF_{n} \right] }.\label{adhoc:2_2}		
}
\end{subequations}
The first term measures the loss of the stationary Boltzmann policy
$\sigma(Q_n \slash \lambda_n)$ relative to an optimal policy. The second term measures the discrepancy between the value of the frozen policy
$\sigma(Q_n \slash \lambda_n)$ and the actual future return of the learning algorithm,
whose policy continues to change after time $n$.

Lemma~\ref{lemma:sensitivity_discounted} shows that~\eqref{adhoc:2_1} can be bounded by the distance between the
transition kernel induced by the optimal policy and the transition kernel induced by the current Boltzmann
policy $\sigma(Q_n\slash \lambda_n)$.~After bounding the distance between these transition kernels, we get the following bound on~\eqref{adhoc:2_1}. It is proved in Appendix~\ref{pf:prop_5}.

\begin{proposition}[Bounding~\ref{adhoc:2_1}]
\label{prop:bound_R_bar}
For the Boltzmann Q-learning algorithm in Algorithm~\ref{algo:Q_learning}, on the event $\cG$, we
have, for every $n\in\mathbb{N}$,
\nal{
& (1-\gamma)\br{V\ust(s_{n}) - V^{(\lambda_n,Q_n)}(s_n)   } \\
& \le \frac{\gamma R_{\max}}{(1-\gamma)} \tilde{O}\left\{\sqrt{\frac{1}{n^{1-(2\fraka+6\kappa\s)}}} +\frac{1}{n^{1-(2\fraka + 4\kappa\s)}} + \frac{1}{n^{\frakb gap(\cM)}}\right\}.
}	
\end{proposition}
The corresponding bound for the second term,~\eqref{adhoc:2_2}, is stated next and proved
in Appendix~\ref{pf:prop_5}.
\begin{proposition}\label{prop:1}
	(Bounding~\eqref{adhoc:2_2})
	For the Boltzmann Q-learning algorithm (Algorithm~\ref{algo:Q_learning}) we have the following bound for $n\in\bN$,
	\nal{
		 \Bigg|  V^{(\lambda_{n},Q_{n})}(s_{n}) - \bE\br{\sum_{m=n}^{\infty} \gamma^{m-n} r(s_m) \Big| \cF_{n}} \Bigg|  & \le     \tilde{O}\br{
    \frac{1}{n^{\frac{1}{2} - (\fraka + 3\kappa\s ) + \frac{\frakb gap(\cM)}{2}  }} }\\
    &
    + \tilde{O}\br{
    \frac{1}{n^{1- (2\fraka + 4\kappa\s) + \frac{\frakb gap(\cM)}{2} }} + \frac{1}{n^{\frac{\frakb gap(\cM)}{2}}}
    }.
	}	
\end{proposition}

\begin{proof}[Proof of Theorem~\ref{th:regret_bound}] (Regret bound for Boltzmann Q-learning)
Recall the decomposition of the instantaneous regret into the two terms
~\eqref{adhoc:2_1},~\eqref{adhoc:2_2}. Proposition~\ref{prop:bound_R_bar} bounds~\eqref{adhoc:2_1}, while Proposition~\ref{prop:1} bounds~\eqref{adhoc:2_2}.~Substituting these two bounds into the regret decomposition gives a
high-probability upper bound on $\regn$. Since
\[
\cR_N=\sum_{n=1}^{N}\regn,
\]
the cumulative regret bound follows by summing the resulting estimate over
$n=1,\ldots,N$. The detailed summation argument, together with the additional
hyperparameter condition needed for the regret bound, is given in Appendix~\ref{sec:regret_BQL_detailed}.
\end{proof}

\subsection{Proofs of Proposition~\ref{prop:bound_R_bar} and~\ref{prop:1}}
\label{pf:prop_5}
We first prove Proposition~\ref{prop:bound_R_bar}. Recall that Proposition~\ref{prop:bound_R_bar} states that on the event $\cG$,
\nal{
& (1-\gamma)\br{V\ust(s_{n}) - V^{(\lambda_n,Q_n)}(s_n)   } \\
& \le \frac{\gamma R_{\max}}{(1-\gamma)} \times \tilde{O}\left\{\sqrt{\frac{1}{(n)^{1-(2\fraka+6\kappa\s_1)}}} +\frac{1}{\br{n}^{1-(2\fraka + 4\kappa\s_1)}} + \frac{1}{(n)^{\frakb gap(\cM)}}\right\}.
}
\begin{proof}(Proposition~\ref{prop:bound_R_bar})
Recall that for Boltzmann Q-learning, we have $\kappa_{2}\s=\kappa_{3}\s=0$ and $\kappa\s_1 = \frac{\frakb R_{\max}}{1-\gamma}, \lambda_{n}\s = \frac{1}{\frakb \log(n+n_0)}$ so that,
\nal{
 g(n,\delta) = \sqrt{\frac{\log\br{\frac{n^2}{\delta}}}{(n+n_0)^{1-(\fraka+2\kappa\s_1)}}}, g_1(n)  = \frac{\log(n+n_0)}{\br{n+n_0}^{1-(\fraka + 2\kappa\s_1)}},g_2(n) = \frac{1}{(n+n_0)^{1-2\kappa\s}}.
}
On the event $\cG$, Theorem~\ref{th:Q_conc} gives the explicit concentration bound,
	\nal{
		\|Q_n - Q\ust\|  &\le \frac{2\frakb\log(n+n_0)^2\br{n+n_0}^{2\kappa\s_1}}{\frakc_{3}} \br{C_{4}\s g(n,\delta) + C_{5}\s g_1(n)+C_{6}\s g_2(n)},\\
        &\forall n\in\bN.
	}
We have,
\al{	
	& (1-\gamma)\br{V\ust(s_{n}) - V^{(\lambda_n,Q_n)}(s_n)   }  \left[\frac{R_{\max} \gamma}{(1-\gamma)}\right]^{-1} \notag \\
    & \le \|p^{(0,Q\ust)} - p^{(\lambda_{n},Q_{n})} \| \notag\\
	& \le  \left[\frac{\|Q_{n}-Q\ust\|}{\lambda_{n}}+ |\cA|e^{-\frac{gap(\cM)}{\lambda_{n}}} \right] \notag\\
	& \le \left[ \frac{\|Q_{n}-Q\ust\|}{\lambda_{n}}+ \frac{|\cA|}{(n+n_0)^{\frakb  gap(\cM)}} \right] \notag\\
	& \le \frakb \log(n+n_0) \left[\frac{2\frakb\log(n+n_0)^2\br{n+n_0}^{2\kappa\s_1}}{\frakc_{3}} \br{C_{4}\s g(n,\delta) + C_{5}\s g_1(n)+C_{6}\s g_2(n)}\right] \notag \\
    & +  \frac{|\cA|}{(n+n_0)^{\frakb gap(\cM)}} \label{adhoc:star_1},
}
where in the first inequality we have used Lemma~\ref{lemma:sensitivity_discounted}, in the second inequality we have used Lemma~\ref{lemma:sensitivity_f} while the third inequality follows since $\lambda_n = \frac{1}{\frakb \log(n+n_0)}$, we have also substituted the bound on $\|Q_{n}-Q\ust\|$ that was derived in Theorem~\ref{th:Q_conc}.~Proof is completed by noting that since $n_0$ is chosen sufficiently large so that Condition~\ref{con:n0_BQL}-\eqref{con:n0_BQL-I} holds, the term $C\s_6g_2(n)$ is absorbed into $C\s_5g_1(n)$.
\end{proof}
\begin{remark}
    We note that we have used the fact that since Condition~\ref{con:hyper1} is satisfied, we have $\fraka + 3\kappa\s_1<1$. This ensures that our upper-bounds decay with $n$.
\end{remark}
Proposition~\ref{prop:1} claims that for Boltzmann Q-learning we have
\nal{
 \Bigg| \bE\br{\sum_{m=n}^{\infty} \gamma^{m-n} r(s_m) \Big| \cF_{n}} - V^{(\lambda_{n},Q_{n})}(s_{n}) \Bigg|   & \le     \tilde{O}\br{
    \frac{1}{n^{\frac{1}{2} - (\fraka + 3\kappa\s_1 ) + \frac{\frakb gap(\cM)}{2}  }} } \\
    & + \tilde{O}\br{
    \frac{1}{n^{1- (2\fraka + 4\kappa\s_1) + \frac{\frakb gap(\cM)}{2} }} + \frac{1}{n^{\frac{\frakb gap(\cM)}{2}}}
    }.
}	

\begin{proof}(Proposition~\ref{prop:1})
	We have,
	\nal{
		\bE\br{\sum_{m=n}^{\infty} \gamma^{m-n} r(s_m) \Big| \cF_{n}}  =  \sum_{m=n}^{\infty} \gamma^{m-n} \sum_{i\in\cS} \bP(s_m = i|\cF_{n} ) r(i).
	}
	Also,
	\nal{
		V^{(\lambda_{n},Q_{n})}(s) = \sum_{m=n}^{\infty} \gamma^{m-n} \sum_{i\in\cS} (p^{(\lambda_{n},Q_{n})})^{m-n}(s,i) r(i).
	}
	Note that $\bE\br{\sum_{m=n}^{\infty} \gamma^{m-n} r(s_m,a_m) \Big| \cF_{n}}$ depends upon the probabilities $\left\{\bP(s_m |\cF_{n})\right\}_{m>n}$, while $V^{(\lambda_{n},Q_{n})}(s_{n})$ depends upon $(p^{(\lambda_{n},Q_{n})})^{m-n}, m\ge n$.~Hence, in order to upper-bound $\Big| \bE\br{\sum_{m=n}^{\infty} \gamma^{m-n} r(s_m,a_m) \Big| \cF_{n}} - V^{(\lambda_{n},Q_{n})}(s_{n})\Big|$, we will need to bound the quantities $|\bP(s_m |\cF_{n})-(p^{(\lambda_{n},Q_{n})})^{m-n}|$ for $m \ge n$. This has been done in Lemma~\ref{lemma:diff_prob_transition}, and we will use those bounds here.
	
	Upon subtracting the above two expressions and taking the magnitude, we get the following,
	\nal{
		& \Big|\bE\br{\sum_{m=n}^{\infty} \gamma^{m-n} r(s_m) \Big| \cF_{n}} - V^{(\lambda_{n},Q_{n})}(s_{n})  \Big| \\
		& \le \Big| \sum_{m=n}^{\infty} \gamma^{m-n} \br{\sum_{i\in\cS} \bP(s_m = i|\cF_{n} ) r(i) -   (p^{(\lambda_{n},Q_{n})})^{m-n}(s_{n},i) } r(i)\Big|  \\
		& \le R_{\max} \sum_{m=n}^{\infty} \gamma^{m-n} (m-n) \times \\
        &\left[ \frac{4|\cA|\log(n+n_0)}{\frakc_{3} (n+n_0)^{\frakb gap(\cM) \slash 2- 2\kappa\s_1}}  \br{C_{4}\s g(n,\delta)+C_{5}\s g_1(n)+C_{6}\s g_2(n)} 
        + \frac{|\cA| \frakb \log(n+n_0)}{(n+n_0)^{\frakb gap(\cM) \slash 2}} \right]\\	
		& \le \br{R_{\max} \int_{0}^{\infty} x e^{-x} dx } \times \\
        & \left[\frac{4|\cA|\log(n+n_0)}{\frakc_{3} (n+n_0)^{\frakb gap(\cM) \slash 2- 2\kappa\s_1}}  \br{C_{4}\s g(n,\delta)+C_{5}\s g_1(n)+C_{6}\s g_2(n)} \right. \\
        & \left. \qquad + \frac{|\cA| \frakb \log(n+n_0)}{(n+n_0)^{\frakb gap(\cM) \slash 2}}\right]\\
        & \le \br{R_{\max} \int_{0}^{\infty} x e^{-x} dx } \left[\frac{4|\cA|\log(n+n_0)}{\frakc_{3} (n+n_0)^{\frakb gap(\cM) \slash 2- 2\kappa\s_1}}  \br{C_{4}\s g(n,\delta)+2C_{5}\s g_1(n)} \right. \\
        & \left. \qquad + \frac{|\cA| \frakb \log(n+n_0)}{(n+n_0)^{\frakb gap(\cM) \slash 2}}\right],
	}	
	where in the second inequality we have used Lemma~\ref{lemma:diff_prob_transition}, while in the last we have used the fact that $n_0$ is sufficiently large so that Condition~\ref{con:n0_BQL}~\eqref{con:n0_BQL-I} is satisfied.~Note that in our proof we have used the fact that under Condition~\ref{con:hyper1} we have
    $\fraka + 3\kappa\s_1 < \frac{1}{2}$, which implies that the conditions $\fraka + 3\kappa\s_1 < \frac{1}{2}, \fraka + 2\kappa\s_1< \frac{1}{2}$ are satisfied, which in turn ensures that our bounds decay with $n$.~This completes the proof. 
\end{proof}

\subsection{Theorem~\ref{th:regret_bound} (Boltzmann Q-learning regret bound)}
\label{pf:regret_boltzmann}\label{sec:regret_BQL_detailed}
We first record the additional exploration condition used in the regret
analysis.
\begin{condition}\label{con:hyper3}
\nal{
    \frac{\frakb R_{\max}}{1-\gamma} \le \fraka.
    }
\end{condition}
Equivalently, the temperature decay parameter $\frakb$ is small enough relative to
$\fraka$.~This ensures sufficient exploration.~We note that the above condition can always be satisfied by choosing $\frakb$ to be sufficiently small.
\begin{proof}
Cumulative regret is given by $\cR_N = \sum_{n=1}^{N} \regn$, where $\regn$ is the instantaneous regret at time $n$ and can be decomposed as in~\eqref{adhoc:2_1}-\eqref{adhoc:2_2}.~This is repeated here for convenience,
\nal{
\regn   & =  (1-\gamma)\br{V\ust(s_{n}) - V^{(\lambda_n,Q_n)}(s_n)   } \\
& +(1-\gamma) \br{ V^{(\lambda_n,Q_n)}(s_n) -  \bE\left[\sum_{m=n}^{\infty} \gamma^{m-n} r(s_m,a_m) \Big| \cF_{n} \right] }.	
}
The first term above is bounded in Proposition~\ref{prop:bound_R_bar}, while the second term is bounded in Proposition~\ref{prop:1}. Upon substituting these bounds into the above decomposition, we get the following bound on $\cG$,
\nal{
& \regn \le \frac{\gamma R_{\max}\frakb \log(n+n_0)}{1-\gamma} \\
& \times \left[\frac{2\frakb\log(n+n_0)^2\br{n+n_0}^{2\kappa\s_1}}{\frakc_{3}} \br{C_{4}\s g(n,\delta) + C_{5}\s g_1(n)+C_{6}\s g_2(n)}\right] \notag \\
& +  \frac{\gamma R_{\max}|\cA|}{(1-\gamma)(n+n_0)^{\frakb gap(\cM)}}  + \br{R_{\max} \int_{0}^{\infty} x e^{-x} dx } \times \\
& \Bigg\{4\frac{|\cA|\log(n+n_0)}{\frakc_{3} (n+n_0)^{\frakb gap(\cM) \slash 2- 2\kappa\s_1}}  \br{C_{4}\s g(n,\delta)+2C_{5}\s g_1(n)} + \frac{|\cA| \frakb \log(n+n_0)}{(n+n_0)^{\frakb gap(\cM) \slash 2}}\Bigg\}.
}
Upon summing the above, we get ($\frakc_{3} = \frac{(1-\gamma)\mu\s_{\min,\cS}}{|\cA|}$),
\nal{
& \cR_N \le  \frac{2\frakb^2 \log(N+n_0)^3 \gamma R_{\max}}{(1-\gamma)\frakc_{3}} \times \left\{  \frac{C_{4}\s \sqrt{\log(N^2\slash \delta)} \br{N+n_0}^{\frac{1}{2}+ \fraka + 3\kappa\s_1} }{\frac{1}{2}+ \fraka + 3\kappa\s_1} \right.\\
& \left. \qquad +  \frac{2C_{5}\s \log(N+n_0)\br{N+n_0}^{2\fraka + 4\kappa\s_1} }{2\fraka + 4\kappa\s_1}
+  \frac{C_{6}\s \br{N+n_0}^{\fraka + 4\kappa\s_1} }{\fraka + 4\kappa\s_1}
\right\}\\
& +  \frac{\gamma R_{\max}|\cA| \br{(N+n_0)^{1-\frakb gap(\cM)} +  n^{1-\frakb gap(\cM)}_0}       }{(1-\gamma)|1-\frakb gap(\cM)|}\\
&+ 4 \br{R_{\max} \int_{0}^{\infty} x e^{-x} dx } \frac{|\cA|C_{4}\s\log(N+n_0)\sqrt{\log(N^2\slash \delta)}}{\frakc_{3}\Big|\frac{1}{2} + (\fraka+3\kappa\s_1) -\frac{\frakb gap(\cM)}{2} \Big|} \\
& \times \left\{ \br{N+n_0}^{\frac{1}{2} + (\fraka+3\kappa\s_1) -\frac{\frakb gap(\cM)}{2} } + \br{n_0}^{\frac{1}{2} + (\fraka+3\kappa\s_1) -\frac{\frakb gap(\cM)}{2} }     \right\}\\
&+ 2 \br{R_{\max} \int_{0}^{\infty} x e^{-x} dx } \frac{|\cA|C_{5}\s\log(N+n_0)^2}{\frakc_{3}\Big|  (2\fraka + 4\kappa\s_1) -\frac{\frakb gap(\cM)}{2}  \Big|} \\
& \times \left\{ \br{N+n_0}^{ (2\fraka + 4\kappa\s_1) -\frac{\frakb gap(\cM)}{2} } + \br{n_0}^{ (2\fraka + 4\kappa\s_1) -\frac{\frakb gap(\cM)}{2} }     \right\}\\
& +  \br{R_{\max} \int_{0}^{\infty} x e^{-x} dx } \frac{|\cA| \frakb \log(N+n_0) \br{(N+n_0)^{1-\frakb gap(\cM) \slash 2} + (n_0)^{1-\frakb gap(\cM) \slash 2}} }{|1-\frakb gap(\cM) \slash 2 |},
}
where
\nal{
\frac{C\s_4}{\frakc_3} & =  \frac{2 \frakc_{11} \frakc_3 \frakc_6  \frakc_5}{\frakc^2_{3}} C\s_1 = 
\frac{32 d \sqrt{\beta} \frakb |\cS|^2 |\cA|^4 \cD(\cM)^2 R^2_{\max}}{(1-\gamma)^2 \mu^2_{\min,\cS}},\\
\frac{C\s_5}{\frakc_3} & =  \frac{2 \frakc_{11} \frakc_3 \frakc_6  \frakc_5}{\frakc^2_{3}}C\s_2 = 
\frac{24\beta \frakb^2 |\cS|^2 |\cA|^4 \cD(\cM)^3 R^3_{\max}\br{\mu_{\min,\cS}(1-\gamma) + 2R_{\max} |\cA|} }{|\cA|(1-\gamma)^3 \mu^2_{\min,\cS} \br{2\fraka(1-\gamma)+ 2\frakb R_{\max}} }
,\\
\frac{C\s_6}{\frakc_3} & =  \frac{2 \frakc_{11} \frakc_3 \frakc_6  \frakc_5}{\frakc^2_{3}} C\s_3 = 
\frac{4\frakb^2 |\cS|^2 |\cA|^4 \cD(\cM)^3 R^3_{\max} \br{\mu_{\min,\cS}+2|\cA|} }{(1-\gamma)^2\mu^2_{\min,\cS}}.
}
Note that hyperparameters must be chosen so that Condition~\ref{con:hyper1} and Condition~\ref{con:hyper3} are satisfied.~Recall that $\kappa\s_1 = \frac{\frakb R_{\max}}{1-\gamma}$, $\frakc_3 = \frac{(1-\gamma)\mu\s_{\min,\cS}}{|\cA|}$.~It is easily verified that the following constraints ensure that Condition~\ref{con:hyper1} and Condition~\ref{con:hyper3} hold 
\nal{
2\fraka + 6 \kappa\s_1 < 1, \frac{\frakb R_{\max}}{1-\gamma} \le \fraka,
}
together with the corresponding case-specific conditions:

Case A: $\fraka = 0$.~Then 
\nal{
0 \le \kappa\s_1 < \frac{1}{2}, \qquad \beta\s \ge 2(1-2\kappa\s_1).
}

Case B: $\fraka >0$.~Then
\nal{
\fraka + 2 \kappa\s_1 <1.
}

Case C: $\fraka = \kappa\s_1$. Then
\nal{
7\fraka < 1, \beta \frakc_3 \ge 1- 7\fraka, \beta \frakc_3 \ge 2- 8\fraka.
}

Case D: $\fraka > \kappa\s_1$. Then $\fraka+6\kappa\s_1 < 1$. 

Note that we cannot choose $\fraka > \kappa\s_1$ since Condition~\ref{con:hyper3} forbids this.~Moreover, since $\frakb >0$, we cannot choose $\fraka = 0$.~Thus, both Case A and Case D are ruled out. We are thus left with Case B and Case C, which are discussed next.

Case B: $\fraka \neq \kappa\s_1, \fraka >0$. It is easily seen that optimal dependence on $N$ is attained by letting $\fraka \to 0^+$, and $\kappa\s_1 = \frac{1}{2\br{ \frac{\br{1-\gamma}gap(\cM)}{R_{\max}}+ 3}}$, which gives 
\al{
\cR_N \le \frac{\tilde{O}\br{ N^{\frac{\br{1-\gamma}gap(\cM) + 6R_{\max}}{ 2\br{  \br{1-\gamma}gap(\cM) + 3R_{\max}  }    }} }}{\br{1-\gamma}^4}.\label{ineq:regret_bound_2_adhoc}
}

Case C: $\fraka = \kappa\s_1$. We choose $\fraka = \kappa\s_1$ and $\beta\s =\frac{2|\cA|}{(1-\gamma)\mu\s_{\min,\cS}}$.~In this case the above bound can be written as
\al{
\cR_N \le \frac{\tilde{O}(N^{.5 + 4\fraka })}{(1-\gamma)^{3.5}} + \frac{\tilde{O}(N^{6\fraka})}{(1-\gamma)^5} + \frac{\tilde{O}(N^{1- \frakb gap(\cM)})}{(1-\gamma)} + \tilde{O}\br{ N^{1- \frac{\frakb gap(\cM)}{2}}},\label{ineq:regret_bound_1_adhoc}
}
where $\fraka$ must satisfy $\fraka \le \frac{1}{8}$.

Note that in the main paper we have only presented the bound~\eqref{ineq:regret_bound_1_adhoc}. Both~\eqref{ineq:regret_bound_1_adhoc} and~\eqref{ineq:regret_bound_2_adhoc} yield near-linear in $N$ regret bound in the limit $gap(\cM)\to 0$, however~\eqref{ineq:regret_bound_1_adhoc} has a better dependence upon the effective horizon $\frac{1}{1-\gamma}$.   
\end{proof}

\begin{remark}[Choice of $\fraka$ when $gap(\cM)$ is known]
The power of $N$ in the bound~\eqref{ineq:regret_bound_1_adhoc} is governed by 
\nal{
\max\left\{ 6 \fraka, .5 + 4\fraka, 1-\frac{\fraka (1-\gamma)gap(\cM)}{2 R_{\max}}  \right\}, \mbox{ where } 0< \fraka < \frac{1}{8},
}
where we have substituted $\fraka = \kappa\s_1 = \frac{\frakb R_{\max}}{1-\gamma}$.

Since $\fraka < \frac{1}{8}$, we have $6\fraka < .5 + 4 \fraka$, so that the above reduces to, 
\nal{
\max\left\{ .5 + 4\fraka, 1-\frac{\fraka (1-\gamma) gap(\cM)}{2 R_{\max}}  \right\}, \mbox{ where } 0< \fraka < \frac{1}{8}.
}
If the quantity $gap(\cM)$ is known, then the optimal choice of $\kappa\s_1$ is obtained by solving 
\nal{
\min_{0\le \fraka < \frac{1}{8}} \max\left\{.5 + 4\fraka , 1-\frac{(1-\gamma) \fraka gap(\cM)}{2 R_{\max}}  \right\}.
}
The optimal choice of $\fraka$ is $\frac{1}{8 + \frac{(1-\gamma)gap(\cM)}{R_{\max}}}$, and the optimal value of the objective function is $\frac{16 + \frac{(1-\gamma)gap(\cM)}{R_{\max}} }{16 + 2\frac{(1-\gamma)gap(\cM)}{R_{\max}}}$.
\end{remark}

\subsection{Conditions on $n_0$}\label{sec:con_n0_sm}
The regret analysis of Boltzmann Q-learning relies on the concentration results
from Section~\ref{sec:SA}.~Therefore, $n_0$ must satisfy Condition~\ref{con:n0} presented in Section~\ref{sec:cond_n0_sa}.~In
addition, the regret proof requires the following conditions specific to Boltzmann Q-learning.

\begin{condition}\label{con:n0_BQL}
(Conditions on $n_0$ for Boltzmann Q-learning)
\begin{enumerate}[(I)]
    \item \label{con:n0_BQL-I}
    \al{
    C\s_{5} g_1(n) & > C\s_{6} g_2(n),~\forall n\in\bN.\label{con:6}
    }
    \item \label{con:n0_BQL-II}
    \al{
    \frac{gap(\cM)}{2}  & > \frac{4\log(\ell+n_0) (\ell +n_0)^{2\kappa\s_1}}{\frakc_{3}} \br{C_{4}\s g(\ell,\delta)+C_{5}\s g_1(\ell)+C_{6}\s g_2(\ell)},\notag\\
    &~\forall \ell \in\bN. \label{con:7}
    }
\end{enumerate}
Note that the functions $g(\cdot,\delta),g_1(\cdot),g_2(\cdot)$ are parameterized by $n_0$.

\end{condition}
These two additional conditions are specific to Boltzmann~Q-learning and will be assumed to be satisfied while performing analysis in Section~\ref{app:pf_q_learning} and Section~\ref{app:aux_q-learning}. 

%% file: appendix_q-learning_auxiliary.tex
\section{Auxiliary results for Appendix~\ref{app:pf_q_learning} (Boltzmann Q-learning)}
\label{app:aux_q-learning}
In this section we will prove results that are used in proofs of Appendix~\ref{app:pf_q_learning}. 

Lemmas~\ref{lemma:sensitivity_discounted} and~\ref{lemma:sensitivity_f} are used while proving Proposition~\ref{prop:bound_R_bar}. These are proved now.   
\begin{lemma}\label{lemma:sensitivity_discounted}
	We have,
	\nal{
		(1-\gamma)\br{V\ust(s_{n}) - V^{(\lambda_n,Q_n)}(s_n)   }  \le \frac{\gamma R_{\max}}{(1-\gamma)}\|p^{(0,Q\ust)} - p^{(\lambda_{n},Q_{n})} \|,~n\in\bN,
	}
where $p^{(\lambda,Q)}$ is as in~\eqref{def:CTP}.    
\end{lemma}
\begin{proof}
	We begin by bounding the difference between the value functions of two given policies $\sigma(Q\slash \lambda), \sigma(Q'\slash \lambda')$. For state $i\in\cS$ we have,
	\al{
		| V^{(\lambda,Q)}(i) - V^{(\lambda',Q')}(i) | & \le \sum_{\ell = 0}^{\infty} \gamma^{\ell} \sum_{j \in\cS} | (p^{(\lambda,Q)})^{\ell}(i,j)-(p^{(\lambda',Q')})^{\ell}(i,j)| r(j) \notag \\
		& \le R_{\max} \sum_{\ell = 0}^{\infty} \gamma^{\ell} \ell \| p^{(\lambda,Q)} - p^{(\lambda',Q')} \| \notag\\
		& \le R_{\max}\| p^{(\lambda,Q)} - p^{(\lambda',Q')} \| \sum_{\ell = 0}^{\infty} \gamma^{\ell} \ell \notag \\
		& =R_{\max}\| p^{(\lambda,Q)} - p^{(\lambda',Q')} \|\frac{\gamma}{(1-\gamma)^2}, \label{ineq:adhoc_3}
	}
	where in the second inequality we have used Lemma~\ref{lemma:sensitivity_prob_dist}.~Now
	\nal{
		(1-\gamma)\br{V\ust(s_{n}) - V^{(\lambda_{n},Q_{n})}(s_{n}) } & = (1-\gamma)\br{V^{(0,Q\ust)}(s_{n}) - V^{(\lambda_{n},Q_{n})}(s_{n}) } \\
		& \le R_{\max}\frac{\gamma}{(1-\gamma)}\|p^{(0,Q\ust)} - p^{(\lambda_{n},Q_{n})} \|,
	}
	where in the last inequality we have used~\eqref{ineq:adhoc_3} with $(\lambda,Q)$ set equal to $(0,Q\ust)$ and $(\lambda',Q')$ to $(\lambda_{n},Q_{n})$. This concludes the proof.
\end{proof}

\begin{lemma}\label{lemma:sensitivity_f}
	The difference between the action probabilities of $\sigma(Q\slash \lambda)$ and the optimal policy $\sigma(Q\ust\slash 0)$ can be bounded as follows,
	\nal{
		\|\sigma(Q\slash \lambda)- \sigma(Q\ust\slash 0) \|  & \le \frac{\|Q-Q\ust\|}{\lambda}+ |\cA| e^{-\frac{gap(\cM)}{\lambda}},\mbox{ where } (\lambda,Q) \in \bR_+ \times \cQ.
	}
\end{lemma}
\begin{proof}
	
	\al{
		\|\sigma(Q\slash \lambda)- \sigma(Q\ust\slash 0) \| 
		& =  \| \sigma(Q\slash \lambda) - \sigma(Q\ust\slash \lambda)\| + \|\sigma(Q\ust\slash \lambda) - \sigma(Q\ust \slash 0) \| \notag\\	
		& \le \frac{\|Q-Q\ust\|}{\lambda}+ \|\sigma(Q\ust\slash \lambda) - \sigma(Q\ust \slash 0) \| ,\label{eq:adhoc_1}
	}
	where the second inequality follows from the fundamental theorem of calculus after using the bound on derivatives established in Lemma~\ref{lemma:diff_softmax_Q} $\|\frac{\partial \sigma(Q(s,\cdot))(a)\slash \lambda)}{\partial Q(s,b)}\| \le \frac{1}{\lambda}$.~To bound $\|\sigma(Q\ust\slash \lambda) - \sigma(Q\ust \slash 0) \|$ we note that the probability assigned by $\sigma(Q\ust\slash \lambda)$ to a sub-optimal action $a$ in state $s\in\cS$ can be upper-bounded as,
	\nal{
		\sigma(Q\ust(s,\cdot)\slash \lambda)(a) & = \frac{\exp\br{Q\ust(s,a)\slash \lambda}}{\sum_{b\in\cA} \exp\br{Q\ust(s,b)\slash \lambda}} \\
		& \le \frac{\exp\br{Q\ust(s,a)\slash \lambda}}{ \exp\br{\max_{b\in\cA}Q\ust(s,b)\slash \lambda}} \\
		& \le \exp\br{-\frac{\br{\max_{b\in\cA} Q\ust(s,b)-Q\ust(s,a)}}{\lambda}}\\
		& \le \exp\br{\frac{-gap(\cM)}{\lambda}}.
	}
	This also gives $\sigma(Q(s,\cdot)\ust\slash \lambda)(a\ust(s)) \ge 1 - |\cA| \exp\br{\frac{-gap(\cM)}{\lambda}}$.~For the vector $\sigma(Q(s,\cdot)\ust\slash 0)$ all elements except that corresponding to $a\ust(s)$ are $0$, while that corresponding to $a\ust(s)$ is equal to $1$. Thus the magnitude of the elements of the vector $\sigma(Q\ust\slash \lambda) - \sigma(Q\ust\slash 0)$ can be upper-bounded by $|\cA| \exp\br{\frac{-gap(\cM)}{\lambda}}$.~The proof is then completed by substituting this bound in~\eqref{eq:adhoc_1}.
\end{proof}
We next prove Lemma~\ref{lemma:diff_prob_transition}, which is used in the proof of Proposition~\ref{prop:1}.

\begin{lemma}\label{lemma:diff_prob_transition}
	For Boltzmann Q-learning (Algorithm \ref{algo:Q_learning}), on the event $\cG$, the following holds for all $m>n$:
	\nal{
&	| \bP(s_{m+1}= i| s_{n}=j) - (p^{(\lambda_{n},Q_n)})^{m+1-n}(j,i)|  \\
&\le (m-n) \left\{ 4\frac{|\cA|\log(n+n_0)}{\frakc_{10} (n+n_0)^{\frakb gap(\cM) \slash 2- 2\kappa\s_1}}  \br{C_{4}\s g(n,\delta)+C_{5}\s g_1(n)+C_{6}\s g_2(n)} \right. \\
& \left. + \frac{|\cA| \frakb \log(n+n_0)}{(n+n_0)^{\frakb gap(\cM) \slash 2}} \right\}
	}
where $i,j \in\cS$.
\end{lemma}
\begin{proof}
We first present a standard perturbation bound for products of stochastic matrices. Let $p$ be a fixed transition matrix and let $p_1,\ldots,p_k$ be time-varying transition matrices. Suppose that we have $\|p-p_{\ell}\|_{\infty} < \Delta$ for all $m$. Then,
\nal{
p^n - \Pi_{\ell=1}^{n}p_{\ell} = (p-p_1)p_2 \cdots p_n + p(p-p_2)p_3 \cdots p_n + \cdots + p^{n-1}(p-p_n),
}
and hence,
\al{
\| p^n - \Pi_{\ell=1}^{n}p_{\ell} \|_{\infty} \le \sum_{k=1}^{n} \|p^{k-1}\|_{\infty} \|p-p_k\|_{\infty}\|p_{k+1}p_{k+1}\cdots p_{n}\|_{\infty} \le n\Delta,\label{ineq:adhoc_1}
}
where we have used $ \|p^{k-1}\|_{\infty}=1$ and $\|p_{k+1}p_{k+1}\cdots p_{n}\|_{\infty} =1$ and also $\|p-p_k\|_{\infty} \le \Delta$.

We now come back to the proof of statement in lemma.~We apply this bound with $p = p^{(\lambda_n,Q_n)}$, $p_{\ell}$ to be $p^{(\lambda_{\ell},Q_{\ell})}, n \le \ell \le m$.~By Lemma~\ref{lemma:diff_kernels}, on the event $\cG$ we have,
\nal{
\|p^{(\lambda_n,Q_n)} - p^{(\lambda_{\ell},Q_{\ell})}\|  & \le 4\frac{|\cA|\log(n+n_0)}{\frakc_{3} (n+n_0)^{\frakb gap(\cM) \slash 2- 2\kappa\s_1}}  \br{C_{4}\s g(n,\delta)+C_{5}\s g_1(n)+C_{6}\s g_2(n)}\\
        & + \frac{|\cA| \frakb \log(n+n_0)}{(n+n_0)^{\frakb gap(\cM) \slash 2}}.        
}
Proof is then completed by substituting in~\eqref{ineq:adhoc_1} the above bound on $p^{(\lambda_n,Q_n)} - p^{(\lambda_{\ell},Q_{\ell})}$. 
\end{proof}
The next technical lemma supplies the kernel perturbation estimate used above.
\begin{lemma}\label{lemma:diff_kernels}
	For Boltzmann Q-learning algorithm (Algorithm \ref{algo:Q_learning}), on the event $\cG$, for all $m>n$,
	\nal{
			& \Big| p^{(\lambda_{m},Q_{m})}(j,i) -  p^{(\lambda_{n},Q_{n})}(j,i) \Big| \\
            & \le 4\frac{|\cA|\log(n+n_0)}{\frakc_{10} (n+n_0)^{\frakb gap(\cM) \slash 2- 2\kappa\s_1}}  \br{C_{4}\s g(n,\delta)+C_{5}\s g_1(n)+C_{6}\s g_2(n)}\\
        & + \frac{|\cA| \frakb \log(n+n_0)}{(n+n_0)^{\frakb gap(\cM) \slash 2}},~i,j\in\cS.
	}
\end{lemma}
\begin{proof}[Proof of Lemma~\ref{lemma:diff_kernels}]
For notational simplicity, assume that each state has a unique optimal action,
denoted by $a^\star(s)$. The same argument applies under a fixed tie-breaking
rule.
	
	\emph{Step I}: We will derive upper-bounds on $\Big| \frac{\partial \sigma(Q(s,\cdot)\slash \lambda)(a)}{\partial Q} |_{(\lambda_n,Q_n)} \Big|$ and $\Big| \frac{\partial \sigma(Q(s,\cdot)\slash \lambda)(a)}{\partial \lambda} |_{(\lambda_n,Q_n)}  \Big|$.~The concentration result of Theorem~\ref{th:main_concentration} gives us the following for $n\in\bN$,
	\nal{
		 | Q_n(s,a) - Q\ust(s,a)|  &\le \frac{2\log(n+n_0) (n+n_0)^{2\kappa\s_1}}{\frakc_{3}} \br{C_{4}\s g(n,\delta)+C_{5}\s g_1(n)+C_{6}\s g_2(n)}\\
         & =: bound(n,\delta), \forall n\in\bN.
	}
	Note that $Q\ust$ is the optimal Q-function. Hence for an action $a$ that is sub-optimal in state $s$, we have the following upper-bound on its probability
	\nal{
		\sigma(Q_{\ell}(s,\cdot)\slash \lambda_{\ell})(a) & = \frac{\exp\br{Q_\ell(s,a)\slash \lambda_\ell}}{\sum_{b\in\cA} \exp\br{Q_\ell(s,b)\slash \lambda_\ell}} \\
		& \le  \frac{\exp\br{Q\ust(s,a)\slash \lambda_\ell + bound(\ell,\delta) \slash \lambda_\ell}}{ \exp\br{ Q\ust(s,a\ust(s))\slash \lambda_\ell - bound(\ell,\delta) \slash \lambda_\ell } } \\
		& \le e^{-gap(\cM)\slash \lambda_\ell + 2 bound(\ell,\delta) \slash \lambda_\ell}\\
		& \le e^{\frakb \log(\ell+n_0)\br{ -gap(\cM) + 2 bound(\ell,\delta) } }.
	}
	Note that $n_0$ is sufficiently large so that Condition~\ref{con:n0_BQL}-\eqref{con:n0_BQL-II} is met.~Hence
	\nal{
		\sigma(Q_{\ell}(s,\cdot)\slash \lambda_{\ell})(a) & \le e^{  -\frac{\frakb gap(\cM) \log(\ell+n_0)}{2}  }\\
		& =  \frac{1}{(\ell+n_0)^{ \frakb gap(\cM) \slash 2}}.
	}	
	Hence, the action probabilities are upper-bounded by $\frac{1}{(\ell+n_0)^{ \frakb gap(\cM) \slash 2}}$ for sub-optimal actions, and lower bounded by $1- \frac{|\cA|}{(\ell+n_0)^{ \frakb gap(\cM) \slash 2}}$ for optimal actions. We substitute this bound on action probabilities into the expression for derivatives derived in Lemma~\ref{lemma:diff_softmax_Q}, and obtain the following:
	\begin{enumerate}[(i)]
		\item 
		\nal{
			\Bigg|  \frac{\partial \sigma(Q(s,\cdot)\slash \lambda)(a)}{\partial Q_\ell(s,a')} \Big|_{\lambda = \lambda_{\ell},Q=Q_{\ell}}  \Bigg| \le  
            \frac{|\cA|}{(\ell+n_0)^{\frakb gap(\cM) \slash 2}}.
		}
		\item 
		\nal{
			\Bigg| \frac{\partial \sigma(Q(s,\cdot)\slash \lambda)(a)}{\partial \lambda} \Big|_{\lambda = \lambda_\ell,Q=Q_\ell } \Bigg|		
			\le 
            \frac{|\cA|}{(\ell+n_0)^{\frakb gap(\cM) \slash 2}} (\frakb \log(\ell+n_0))^2.
		}     
	\end{enumerate}
	
	\emph{Step II}: On the set $\cG$ we have,
	\nal{
		& \Big|   p^{(\lambda_{m},Q_{m})}(j,i) -p^{(\lambda_{n},Q_{n})}(j,i) \Big| \\
		& \le \Big| p^{(\lambda_{n},Q_{m})}(j,i) -p^{(\lambda_{n},Q_{n})}(j,i) \Big| + \Big|   p^{(\lambda_{m},Q_{m})}(j,i) -p^{(\lambda_{n},Q_{m})}(j,i)    \Big|\\
        & \le \frac{|\cA|}{(n+n_0)^{\frakb gap(\cM) \slash 2}} \|Q_m - Q_n\| +
        \frac{|\cA|}{(n+n_0)^{\frakb gap(\cM) \slash 2}} (\frakb \log(n+n_0))^2 |\lambda_m - \lambda_n|\\
        & \le 2\frac{|\cA|}{(n+n_0)^{\frakb gap(\cM) \slash 2}} \times \\
        & \left\{ \frac{2\log(n+n_0) (n+n_0)^{2\kappa\s_1}}{\frakc_{3}} \br{C_{4}\s g(n,\delta)+C_{5}\s g_1(n)+C_{6}\s g_2(n)} \right\} \\
        & +\frac{|\cA|}{(n+n_0)^{\frakb gap(\cM) \slash 2}} (\frakb \log(n+n_0))^2 \lambda_n \\
        & = 4\frac{|\cA|\log(n+n_0)}{\frakc_{3} (n+n_0)^{\frakb gap(\cM) \slash 2- 2\kappa\s_1}}  \br{C_{4}\s g(n,\delta)+C_{5}\s g_1(n)+C_{6}\s g_2(n)}\\
        & + \frac{|\cA| \frakb \log(n+n_0)}{(n+n_0)^{\frakb gap(\cM) \slash 2}},
	}
	where in the second inequality we have used the bounds (i), (ii) on derivatives derived in Step I above along with the fundamental theorem of calculus~\citep{folland1999real}, while the third inequality follows since $\|Q_m - Q_n\|\le \|Q_m - Q\ust\|+ \|Q\ust-Q_n\|\le 2 bound(n,\delta)$. For fourth inequality we have used $\lambda_m = \frac{1}{\frakb \log(m+n_0)}$.
\end{proof}

%% file: appendix_eps_softmax.tex
\section{Analysis of \seg~Q-learning}\label{app:eps_softmax}
This section proves the concentration and regret bounds for S$\epsilon$G
Q-learning.~Section~\ref{sec:acd_segql} verifies that S$\epsilon$G Q-learning fits into the
stochastic approximation framework of Section~\ref{sec:SA}. Section~\ref{sec:conc_seg-ql} proves
concentration of the iterates $\{Q_n\}$. Section~\ref{sec:regret_seg-ql} derives the regret
decomposition and regret bound. Finally, Section~\ref{sec:cond_n0_eps_sm} records the additional
requirements on $n_0$ needed for the regret analysis.~These conditions will be assumed to hold throughout in Section~\ref{app:eps_softmax} and Section~\ref{sec:eps_softmax_auxiliary}.

Recall that for \seg~Q-learning (Algorithm~\ref{algo:eps-softmax}), the temperature sequence is given by $\lambda_n = 1/(n+n_0)^{\frake}$, while exploration sequence by $\eps_n = 1/(n+n_0)^{\frakd}$. The step size is given by $\beta_n = \frac{\beta}{(n+n_0)^{1-\fraka}}$.

\subsection{Assumptions, conditions and definitions}\label{sec:acd_segql}
We impose the following recurrence condition on the underlying MDP. Recall that $\mu^{(\epsilon,\lambda,Q)}$ denotes the stationary distribution of the
state-action Markov chain induced by the policy $f^{(\epsilon,\lambda,Q)}$.
\begin{assumption}\label{assum:16_eps}
	Define
	\nal{
    \mu\es_{\min,\cS} := \min_{s\in\cS}\inf_{\substack{ \eps\ge 0\\\lambda\in\bR_+ \\ Q\in\cQ}}\sum_{a\in\cA} \mu^{(\eps,\lambda,Q)}(s,a).
	}    
Then $\mu\es_{\min,\cS} >0$.
\end{assumption}
This condition says that every state has uniformly positive stationary mass
under the stationary policies considered in the analysis. Equivalently, it rules out states that become transient under some stationary policy.

We now show that \seg~Q-learning can be written in the stochastic
approximation form~\eqref{def:gen_SA}. This allows us to apply Theorem~\ref{th:main_concentration}.

\begin{proposition}\label{prop:q_learn_assum_eps}
\seg~Q-learning can be written as the stochastic approximation recursion~\eqref{def:gen_SA}. Moreover, if the MDP $\cM$ satisfies Assumption~\ref{assum:16_eps}, then Assumptions~\ref{assum:1}-~\ref{assum:14} hold with
\nal{
& \frakc\es_1 = \cD(\cM), \frakc\es_2 = 1-\gamma, \frakc\es_3 =  \frac{2(1-\gamma)\mu\es_{\min,\cS}}{|\cA|}, \frakc\es_4 = \frac{2 R_{\max}}{1-\gamma}\\
& \frakc\es_5 = 1,\frakc\es_6 = 1+\gamma,
\frakc\es_7 = \frakc\es_8 = \frakc\es_9 = \frac{R_{\max}}{1-\gamma}, \frakc\es_{10} = \mu\es_{\min,\cS} \cD(\cM).
}
Furthermore,
\nal{
\kappa\es_1 = \frakd,~~\kappa\es_2 = \frake,~~\kappa\es_3 = \frake.
}
The functions in Assumption~\ref{assum:3} may be chosen as, 
\nal{
L_1(n) & = \frac{2}{(n+n_0)^{1-\frake}}, n\in\bN,\\
L_2(\eps,\lambda) &= \frac{1}{\lambda}, \eps,\delta>0.
}
\end{proposition}
\begin{proof}
The proof is parallel to the proof of Proposition~\ref{prop:q_learn_assum}, so we only highlight the
parts that differ from Boltzmann Q-learning. The external control variable is
now $\{(\eps_n,\lambda_n)\}$.

~We recall that we use $\eps_n = \frac{1}{(n+n_0)^{\frakd}},\lambda_n = \frac{1}{(n+n_0)^{\frake}}$, where $\frakd\in [0,1],\frake >0,\frakd > \frake$.~The stationary policy $f^{(\eps,\lambda,Q)}$ induces a Markov chain on the space $\cS\times \cA$, with the transition kernel $p^{(\eps,\lambda,Q)}$ as follows,
\al{
	p^{(\eps,\lambda,Q)}((s,a),(s',a'))  = p(s,a,s')f^{(\eps,\lambda,Q)}(s',a'),~\mbox{ where }(s,a), (s',a')\in\cS\times\cA,\label{def:CTP_1}
}
and stationary distribution is denoted $\mu^{(\eps,\lambda,Q)}$.~$\{M_n\}$, the function $F$ and the property $Q_n \in \cQ$ remain same as that defined previously for Boltzmann~Q-learning.~Proof of the contraction property is also almost the same, with a minor difference that now the contraction factor is given by $\alpha(\eps,\lambda)  = 1 -  (1-\gamma) \mu_{\min}(\eps,\lambda)$, where $\mu_{\min}(\eps,\lambda) = \inf_{ Q\in\cQ}\min_{(s,a)\in\cS\times\cA} \mu^{(\eps,\lambda,Q)}(s,a)$.

The lower bound on $\sigma(Q(s,\cdot)\slash \lambda)(a)$ from Lemma~\ref{lemma:lower_bound_action_prob} now yields the following lower bound on action probabilities,
\nal{
	\min_{a\in\cA,s\in\cS} \min_{Q \in\cQ}f^{(\eps,\lambda,Q)}(s,a) \ge  \frac{\eps}{|\cA|} + \frac{1}{|\cA|\exp\br{\frac{R_{\max}}{\lambda (1-\gamma)}}}.
}	
Hence 
\nal{
\mu_{\min}(\eps,\lambda) \ge  \frac{\frakc\es_7}{|\cA|} \br{\eps + \frac{1}{\exp\br{\frac{R_{\max}}{\lambda (1-\gamma)}}}},
}
where $\frakc\es_7$ is as in~Assumption~\ref{assum:16_eps}.~This shows $\mu_{\min}(\eps,\lambda)>0$.~Choose a state-action pair $(s\ust,a\ust)\in\cS\times\cA$ to be a designated state for the controlled Markov process $\{y_n\} = \{(s_n,a_n)\}$. Under the application of policy $f^{(\eps,\lambda,Q)}$, we can show that the mean hitting time to $(s\ust,a\ust)$ can be upper-bounded as $\frac{\cD(\cM)}{\mu_{\min}(\eps,\lambda)}$, where $\cD(\cM)$ is the MDP diameter. Hence, the hitting-time bound in Assumption~\ref{assum:2} holds with $\frakc_1 = \cD(\cM)$. We have $\at(\eps,\lambda) = (1-\gamma) \mu_{\min}(\eps,\lambda)$, so that Assumption~\ref{assum:at} holds with $\frakc_2 = 1-\gamma$.~Upon substituting $\lambda_n = \frac{1}{(n+n_0)^{\frake}}, \eps_n = \frac{1}{(n+n_0)^\frakd}$  into the expression for $\mu_{\min}(\eps,\lambda)$ we get,
\nal{
\at(\eps_n,\lambda_n)\ge (1-\gamma)\frac{\frakc\es_7}{|\cA|} \left\{\frac{1}{(n+n_0)^\frakd} + \exp\br{- \frac{R_{\max}}{1-\gamma} (n+n_0)^\frake } \right\}.
}
Since $n_0$ is chosen sufficiently large so that we have (Condition~\ref{cond:n0_seg-ql}-\eqref{cond:n0_seg-ql-IV})
\nal{
\frac{1}{(n+n_0)^\frakd} > \exp\br{- \frac{R_{\max}}{1-\gamma} (n+n_0)^\frake },\forall n\in\bN, 
}
we get
\nal{
\at(\eps_n,\lambda_n)\ge \frac{2(1-\gamma) \mu\es_{\min,s} }{|\cA|(n+n_0)^\frakd}.
}
Hence Assumption~\ref{assum:3}~(a) holds with $\kappa_1 = \frakd$ and $\frakc_{3} = \frac{2(1-\gamma)\mu\es_{\min,\cS}}{|\cA|}$.~We denote $\kappa\es_1 =\frakd$. Similar to proof of Proposition~\ref{prop:q_learn_assum} it can be shown that $Q_n(s,a) \le \frac{R_{\max}}{1-\gamma}$, and hence we can take $\frakc\es_7 = \frakc\es_8 = \frakc\es_9 = \frac{R_{\max}}{1-\gamma}$. Similarly, Lipschitz property of $F(\cdot,s,a)$ also follows similar to Proposition~\ref{prop:q_learn_assum}, and $\frakc\es_6$ can be taken to be $1+\gamma$.

To verify Assumption~\ref{assum:14}, we note that 
\nal{
	  & \|\mu^{(\eps,\lambda,Q)} -  \mu^{(\eps',\lambda',Q')}\|_{\infty} \\
	&\le  \frac{\max_{(s,a),(s',a')\neq (s,a)} \bE\br{ \cT_{hit.,(s',a')}| s_0 = s, a_0 = a }  }{2} \| p^{(\eps,\lambda,Q)} - p^{(\eps',\lambda',Q')} \|_{\infty}\\
	& \le \frac{\cD(\cM)}{\frac{1}{|\cA|}\br{ \eps + \frac{1}{\exp\br{\frac{R_{\max}}{\lambda (1-\gamma)}}} } } \| p^{(\eps,\lambda,Q)} - p^{(\eps',\lambda',Q')} \|_{\infty}\\
	& = \frac{ \mu\es_{\min,\cS}}{ \mu_{\min}(\lambda)} \cD(\cM) \| p^{(\eps,\lambda,Q)} - p^{(\eps',\lambda',Q')} \|_{\infty},    
}   
where the first inequality follows from Theorem 3.1 of~\citep{cho2000markov},  the second follows since the hitting time to $(s\ust,a\ust)$ can be bounded as the reciprocal of minimum action probability, while the third follows since we have shown $\mu_{\min}(\eps,\lambda) \ge  \frac{\mu\es_{\min,\cS}}{|\cA|} \br{\eps + \frac{1}{\exp\br{\frac{R_{\max}}{\lambda (1-\gamma)}}}}$. Hence Assumption~\ref{assum:14} holds with $\frakc_{10} = \mu\es_{\min,\cS} \cD(\cM)$.

We will now show Assumption~\ref{assum:3}, the Lipschitz property of the transition kernel. Upon using the expression for $p^{(\eps,\lambda,Q)}$ we get,
\nal{
	 & \sum_{(s',a') \in\cS \times \cA} |p^{(\eps,\lambda,Q)}((s,a),(s',a')) - p^{(\eps,\lambda,Q')}((s,a),(s',a'))| \\
	& \le \sum_{(s',a') \in\cS \times \cA} p(s,a,s') | f^{(\eps,\lambda,Q)}(s',a') - f^{(\eps,\lambda,Q')}(s',a')|\\
	& \le (1-\eps)\sum_{(s',a') \in\cS \times \cA} p(s,a,s') | \sigma(Q(s',\cdot)\slash \lambda)(a') - \sigma(Q'(s',\cdot)\slash \lambda)(a')|\\
	& \le \br{\sum_{(s',a') \in\cS \times \cA} p(s,a,s') \frac{1}{\lambda}}  \|Q-Q'\|_{\infty}\\
	& \le \frac{ \|Q-Q'\|_{\infty}}{\lambda},
}
the third inequality follows from the fundamental theorem of calculus after using the expression for the derivatives of $\sigma(Q\slash \lambda)$ w.r.t. $Q$ that was derived in Lemma~\ref{lemma:diff_softmax_Q}. Thus, $L_2(\eps,\lambda)$ can be taken to be $\frac{1}{\lambda}$. Since $\lambda_n = \frac{1}{(n+n_0)^{\frake}}$, this gives us $L_2(\eps_n,\lambda_n) \le (n+n_0)^{\frake}$. This shows that Assumption~\ref{assum:3} also holds with $\frakc_{5} = 1$ and $\kappa_3 = \frake$. We denote $\kappa\es_{3} = \frake$.

For $y = (s,a), y' = (s',a')$ we have ,
\nal{
& |p^{(\eps_n,\lambda_n,Q)}(y,y') - p^{(\eps_{n+1},\lambda_{n+1},Q)}(y,y')| \\
& \le |p^{(\eps_n,\lambda_n,Q)}(y,y') - p^{(\eps_{n+1},\lambda_{n},Q)}(y,y')|  + |p^{(\eps_{n+1},\lambda_n,Q)}(y,y') - p^{(\eps_{n+1},\lambda_{n+1},Q)}(y,y')| \\
& = p(s,a,s') |   f^{(\eps_n,\lambda_n,Q)}(s',a') -  f^{(\eps_{n+1},\lambda_n,Q)}(s',a') | \\
& + p(s,a,s') |   f^{(\eps_{n+1},\lambda_n,Q)}(s',a') -  f^{(\eps_{n+1},\lambda_{n+1},Q)}(s',a') | \\
& \le  p(s,a,s') \br{\eps_n - \eps_{n+1}} + p(s,a,s') | \sigma(Q(s',\cdot)\slash \lambda_n)(a') - \sigma(Q(s',\cdot)\slash \lambda_{n+1})(a')  | \\
& \le  p(s,a,s') \frac{\frakd}{(n+n_0)^{1+\frakd}} + p(s,a,s') 
| \sigma(Q(s',\cdot)\slash \lambda_n)(a') - \sigma(Q(s',\cdot)\slash \lambda_{n+1})(a') | \\
& \le  \frac{p(s,a,s')\frakd}{(n+n_0)^{1+\frakd}}  + \frac{p(s,a,s')}{\lambda^2_n} \|Q(s',\cdot)\|_{1} |\lambda_n - \lambda_{n+1}| \\
& \le  \frac{p(s,a,s')\frakd}{(n+n_0)^{1+\frakd}}  + p(s,a,s')(n+n_0)^{2\frake} \|Q(s',\cdot)\|_{1} \frac{\frake}{(n+n_0)^{1+\frake} } \\
& \le \frac{p(s,a,s')\frakd}{(n+n_0)^{1+\frakd}}  + \frac{p(s,a,s') \|Q(s',\cdot)\|_{1}}{(n+n_0)^{1-\frake}}\\
& \le 2\frac{p(s,a,s') \|Q(s',\cdot)\|_{1}}{(n+n_0)^{1-\frake}},
}
where in the third inequality we have used $\eps_n = \frac{1}{(n+n_0)^{\frakd}}$ so that $\eps_n - \eps_{n+1} \approx \frac{\frakd}{(n+n_0)^{1+\frakd}}$, to get fourth inequality we use the fundamental theorem of calculus and the bound on the magnitude of the derivative of probabilities $\sigma(Q\slash \lambda)$ w.r.t. $\lambda$ from Lemma~\ref{lemma:diff_softmax_Q}.~To obtain the sixth inequality we have used $\lambda_n = \frac{1}{(n+n_0)^{\frake}}$ so that $\lambda_n - \lambda_{n+1} \approx \frac{\frake}{(n+n_0)^{1+\frake} }$. In the second last inequality we have also used $\frake \leq 1$. The last inequality follows since $n_0$ is sufficiently large so that Condition~\ref{cond:n0_seg-ql}-\eqref{cond:n0_seg-ql-I} is met.~Upon performing a summation over state $s'$ and substituting the bound on $\|Q\|$ we get 
\nal{
\sum_{y'\in\cY} |p^{(\eps_n,\lambda_n,Q)}(y,y') - p^{(\eps_{n+1},\lambda_{n+1},Q)}(y,y')| \le \frac{2 R_{\max}}{(1-\gamma)(n+n_0)^{1-\frake}}.
}		
Hence, Assumption~\ref{assum:3}-(c) holds, we can take $L_1(n) = \frac{2}{(n+n_0)^{1-\frake}}$, and $\frakc\es_4 = \frac{2R_{\max}}{1-\gamma}, \kappa\es_2 = \frake$.
\end{proof}

It is easily verified that for \seg~Q-learning, Condition~\ref{con:sa1} reduces to the following condition on the hyperparameters $\fraka,\frakd,\frake$. Hence, while analyzing  \seg~Q-learning, the following condition will be assumed to hold.
\begin{condition}(\seg~Q-learning Hyperparameters)
\label{con:hyper_seg}
We must have $\fraka,\frakd,\frake \ge 0$. Moreover, they satisfy
\nal{
2\fraka + 6\frakd +3\frake<1.
}
In addition, the following case-specific requirements are imposed:

Case A: $\fraka =0$. We require,
\nal{
\beta \ge 2(1-2\frakd).
}

Case B: $\fraka = \frakd$. We require, 
\nal{
\beta  \ge \frac{|\cA|}{2(1-\gamma)\mu\es_{\min,\cS}} \max \left\{ 1- (3\fraka + 4\frakd + 2 \frake),2-\br{4\fraka + 6\frakd +4 \frake},2-\br{ 2\fraka + 6\frakd  + 4 \frake}\right\}.
}
\end{condition}

\subsection{Concentration of $\{Q_n\}$ for \seg~Q-learning (Proof of Theorem~\ref{th:conc_q_eps_softmax})}
\label{sec:conc_seg-ql}
We will now prove Theorem~\ref{th:conc_q_eps_softmax}. Recall that for \seg~Q-learning we have $\kappa\es_{2}=\frake$ and $\kappa\es_{3} = \frake$, $\kappa\es_1 =\frakd$. Also,
$\lambda_{n} = \frac{1}{\br{n+n_0}^{\frake}}, \eps_n = \frac{1}{(n+n_0)^{\frakd}}$, and $\frakd > \frake$. Hence, the functions $g(\cdot,\cdot),g_1(\cdot),g_2(\cdot)$ take the following form,
\nal{
 g(n,\delta) &= \sqrt{\frac{\log\br{\frac{n^2}{\delta}}}{(n+n_0)^{1-(\fraka+2\kappa\es_1)}}} = \sqrt{\frac{\log\br{\frac{n^2}{\delta}}}{(n+n_0)^{1-(\fraka+2\frakd)}}},\qquad \delta >0, n\in\bN,\\
 g_1(n) & = \frac{\log(n+n_0)}{\br{n+n_0}^{1-(\fraka + 2\kappa\es_1 +\kappa\es_{3})}} =  \frac{\log(n+n_0)}{\br{n+n_0}^{1-(\fraka + 2\frakd +\frake)}},\qquad n\in\bN,\\
 g_2(n) &= \frac{1}{(n+n_0)^{1-(2\kappa\es_1 +\kappa\es_{2})}} = \frac{1}{(n+n_0)^{1-(2\frakd+\frake)}},\qquad n\in\bN.
}
Under Assumption~\ref{assum:16_eps} and Condition~\ref{con:hyper_seg}, Proposition~\ref{prop:q_learn_assum_eps} allows us to use Theorem~\ref{th:main_concentration} to infer that the following concentration result holds: Fix $\delta \in (0,1)$. Then, with a probability greater than $1- \delta \frac{\pi^2 d}{6}$ we have (note that $\frakc\es_{3} = \frac{2(1-\gamma)\frakc\es_7}{|\cA|}$):
\al{
	& \|Q_n - Q\ust\|_{\infty} \notag \\
    & \le 2\log(n+n_0)\br{n+n_0}^{2\frakd+\frake} \br{\frac{C\es_4}{\frakc\es_{3}} g(n,\delta) + \frac{C\es_5}{\frakc\es_{3}} g_1(n)+ \frac{C\es_6}{\frakc\es_{3}} g_2(n)},\notag\\
    & \forall n\in\bN.\label{qn-q*_final} 
}
The quantities $\frac{C\es_4}{\frakc\es_{3}},\frac{C\es_5}{\frakc\es_{3}},\frac{C\es_6}{\frakc\es_{3}}$ are computed using~\eqref{def:C4-6} and~\eqref{def:C1-C3},
\nal{
\frac{C\es_4}{\frakc\es_{3}} &=  \frac{8|\cA|^2 \frakd  \sqrt{\beta} R^2_{\max} \cD(\cM)^2   }{(1-\gamma)^2 \mu\es_{\min,\cS}},\\
\frac{C\es_5}{\frakc\es_{3}} & = \frac{6|\cA|^2 \beta R^3_{\max} \cD(\cM)^3 }{(1-\gamma)^3 \mu\es_{\min,\cS}(2\fraka+2\frakd +\frake)}\br{ (1-\gamma)\mu\es_{\min,\cS}+2R_{
\max}}, \\
\frac{C\es_6}{\frakc\es_{3}} &= \frac{2|\cA|^2 R^3_{\max} \cD(\cM)^3 }{(1-\gamma)^2\mu\es_{\min,\cS}}\br{\mu\es_{\min,\cS}+2}.
} 
Also,
\nal{
\br{n+n_0}^{2\frakd+\frake} g(n,\delta) & = n^{-\br{.5 - \br{\frac{\fraka}{2} + 3\frakd +\frake} }   }\\
\br{n+n_0}^{2\frakd+\frake} g_1(n) & = n^{-\br{1-(\fraka + 4\frakd +2\frake)}}\\
\br{n+n_0}^{2\frakd+\frake} g_2(n) & = n^{-\br{1-(4\frakd +2\frake)}}.
}
The detailed bound is now obtained by substituting these into~\eqref{qn-q*_final}. 

\begin{remark}[On Condition~\ref{con:hyper_seg}]
It is easily verified that Condition~\ref{con:hyper_seg} implies $.5-\br{\frac{\fraka}{2}+3\frakd+\frake}>0,1- \br{\fraka+ 4\frakd +2\frake }>0, 1- \br{ 4\frakd +2\frake }>0$, which in turn ensures that the upper-bound on $\|Q_n - Q\ust\|_{\infty}$ decays with $n$. This allows us to show that $Q_n$ converges to $Q\ust$.
\end{remark}
We now optimize the convergence rate by choosing the hyperparameters of~\seg~Q-learning appropriately.~These must be chosen so as to ensure that Condition~\ref{con:hyper_seg} holds. We set $\fraka = \frakd = \frake=0$. We also set $\beta  = \frac{|\cA|}{(1-\gamma)\mu\es_{\min,\cS}}$. These satisfy Condition~\ref{con:hyper_seg}, and give us the following convergence guarantees for~\seg~Q-learning.

Fix a $\delta \in (0,1)$. Then, with a probability greater than 
$1- \delta \frac{\pi^2 d}{6}$ we have
\nal{
\|Q_n - Q\ust \|_{\infty} &\le \frac{1}{(1-\gamma)^{2.5}}\tilde{O}\br{\frac{1}{(n+n_0)^{.5 }} }
+ \frac{1}{(1-\gamma)^{4}}\tilde{O}\br{ \frac{1}{n+n_0 } },~\forall n\in\bN.
}
For a fixed $\gamma<1$ we write this as,
\nal{
\|Q_n - Q\ust \|_{\infty} &\le \frac{1}{(1-\gamma)^{2.5}}\tilde{O}\br{\frac{1}{(n+n_0)^{.5 }} },~\forall n\in\bN,
}
which gives us a sample complexity of $\tilde{O}\br{ \frac{1}{\eps^2 (1-\gamma)^5} }$.~This concentration bound is stated in Theorem~\ref{th:conc_q_eps_softmax} of main paper.

\subsection{Regret of \seg~Q-learning}\label{sec:regret_seg-ql}
For analyzing regret of \seg~Q-learning we require that the hyperparameters be chosen so that they satisfy the following additional condition.
\begin{condition}
(\seg~Q-learning Hyperparameters)
\label{con:seg_regret_1}
 \al{
2\fraka + 6\frakd +4\frake <1,\frake < \frakd , \fraka -\frake <1, \frakd \le \fraka, \frake>0.\label{hyper:6}
}   
\end{condition}
It is easily verified that the first condition above is equivalent to $2\fraka + 6\kappa\es_1+2\kappa\es_3 + 2\frake<1$.~We now state some definitions.
\begin{definition}
For $(\eps,\lambda,Q)\in \bR_+ \times \bR_+ \times \cQ$, we let $V^{(\eps,\lambda,Q)}$ be the value function of the policy $f^{(\eps,\lambda,Q)}$, i.e. 
\nal{
V^{(\eps,\lambda,Q)}(s):= \bE\left\{ \sum_{n=0}^{\infty} \beta^n r(s_n,a_n) | s_0 = s \right\},~s\in\cS,
}
where $a_n$ is sampled from the distribution $f^{(\eps,\lambda,Q)}(s_n,\cdot)$, i.e. $a_n \sim f^{(\eps,\lambda,Q)}(s_n,\cdot)$, and expectation is taken under the measure induced by this policy.    
\end{definition}
We use $f^{(0,0,Q\ust)}$ to denote an optimal policy that in each state $s$ samples actions uniformly at random from the set $\cA\ust(s)$.~Recall the definition of instantaneous regret $\reg^{(n)} =  (1-\gamma)\br{V\ust(s_{n}) -  \bE\left[\sum_{m=n}^{\infty} \gamma^{m-n} r(s_m,a_m) \Big| \cF_{n} \right]  }$, so that the cumulative regret is given by $\cR_N  = \sum_{n=1}^{N} \reg^{(n)}$.~We have the following decomposition of $\regn$,
\begin{subequations}
		\al{
\regn   & =  (1-\gamma)\br{V\ust(s_{n}) - V^{(\eps_n,\lambda_n,Q_n)}(s_n)   }\label{adhoc:2_1_eps}		\\
& +(1-\gamma)\br{ V^{(\eps_n,\lambda_n,Q_n)}(s_n) -  \bE\left[\sum_{m=n}^{\infty} \gamma^{m-n} r(s_m,a_m) \Big| \cF_{n} \right] }.\label{adhoc:2_2_eps}		
}
\end{subequations}
We will upper-bound the terms~\eqref{adhoc:2_1_eps} and~\eqref{adhoc:2_2_eps} separately.

We begin by showing that~\eqref{adhoc:2_1_eps} can be upper-bounded as the distance between the transition probabilities under the optimal policy $f^{(0,0,Q\ust)}$, and that under the current policy $f^{(\eps_n,\lambda_n,Q_n)}$.~We omit its proof since it is similar to that of Lemma~\ref{lemma:sensitivity_discounted}.
\begin{lemma}\label{lemma:sensitivity_discounted_eps}
	We have,
	\nal{
		(1-\gamma)\br{V\ust(s_{n}) - V^{(\eps_n,\lambda_n,Q_n)}(s_n)   }  \le \frac{\gamma R_{\max}}{1-\gamma}\|p^{(0,0,Q\ust)} - p^{(\eps_n,\lambda_{n},Q_{n})} \|,~\forall n\in\bN.
	}    
\end{lemma}

Upon bounding these transition probabilities $\|p^{(0,0,Q\ust)} - p^{(\eps_n,\lambda_{n},Q_{n})} \|$ in the above result, we get the following bound.

\begin{proposition}[Bounding~\ref{adhoc:2_1_eps}]\label{prop:bound_R_bar_eps}
For \seg~Q-learning (Algorithm~\ref{algo:eps-softmax}) we have the following bound that holds on the high probability set $\cG$~\eqref{def:G}:
\nal{
& (1-\gamma)\br{V\ust(s_{n}) - V^{(\eps_n,\lambda_n,Q_n)}(s_n)   }  \le \frac{\gamma R_{\max}}{1-\gamma} \times \\
& \left[\frac{2}{(n+n_0)^{\frakd}} + 2\log(n+n_0)\frac{\br{n+n_0}^{2\kappa\es_1+\kappa\es_{3}+\frake}}{\frakc\es_{10}} \br{C\es_4 g(n,\delta) + 2 C\es_5 g_1(n)} \right].
}	
\end{proposition}
\begin{proof}
    We have,
\al{	
	& (1-\gamma)\br{V\ust(s_{n}) - V^{(\eps_n,\lambda_n,Q_n)}(s_n)   }  \left[\frac{R_{\max} \gamma}{(1-\gamma)}\right]^{-1} \notag \\
    & \le \|p^{(0,0,Q\ust)} - p^{(\eps_n,\lambda_{n},Q_{n})} \| \notag\\
	& \le  \left[\eps_n + \frac{\|Q_{n}-Q\ust\|_{\infty}}{\lambda_{n}}+ |\cA|e^{-\frac{gap(\cM)}{\lambda_{n}}} \right] \notag\\
	& =  \left[\frac{1}{(n+n_0)^{\frakd}} + \frac{\|Q_{n}-Q\ust\|_{\infty}}{\lambda_{n}}+ |\cA|e^{-(n+n_0)^{\frake}gap(\cM)} \right] \notag\\    
	& \le  \frac{1}{(n+n_0)^{\frakd}} \notag \\
    &+ 2\log(n+n_0)\frac{\br{n+n_0}^{2\kappa\es_1+\kappa\es_{3}+\frake}}{\frakc\es_{3}} \br{C\es_4 g(n,\delta) + C\es_5 g_1(n)+C\es_6 g_2(n)} \notag \\
    &+ |\cA|e^{-(n+n_0)^{\frake}gap(\cM)} \notag\\ 
	& \le \frac{2}{(n+n_0)^{\frakd}} + 2\log(n+n_0)\frac{\br{n+n_0}^{2\kappa\es_1+\kappa\es_{3}+\frake}}{\frakc\es_{3}} \br{C\es_4 g(n,\delta) + 2 C\es_5 g_1(n)}, \notag     
}
where the first and second inequality follow from Lemma~\ref{lemma:sensitivity_discounted_eps} and Lemma~\ref{lemma:sensitivity_f_eps} respectively.~Third inequality follows by substituting the bound on $\|Q_{n}-Q\ust\|_{\infty}$ from Theorem~\ref{th:conc_q_eps_softmax}.~Last inequality follows by noting that $n_0$ is sufficiently large so that Condition~\ref{cond:n0_seg-ql}-\eqref{cond:n0_seg-ql-II} holds and that $\frake>0$.~This completes the proof. Note that the condition $2\fraka + 6\kappa\es_1+2\kappa\es_{3} + 2\frake<1$ (which is satisfied under Condition~\ref{con:seg_regret_1}) is required in order to ensure that our upper-bound decays with $n$.
\end{proof}
We will now bound the term~\eqref{adhoc:2_2_eps}, which will allow us to bound $\regn$.
\begin{proposition}\label{prop:1_eps}
	(Bounding~\eqref{adhoc:2_2_eps})
    For \seg~Q-learning (Algorithm~\ref{algo:eps-softmax}) we have,
	\nal{
		 \Bigg| \bE\br{\sum_{m=n}^{\infty} \gamma^{m-n} r(s_m) \Big| \cF_{n}} - V^{(\eps_n,\lambda_{n},Q_{n})}(s_{n}) \Bigg|   & \le    
         \br{R_{\max} \int_{0}^{\infty} x^2 e^{-x} dx } \Bigg\{\frac{3|\cA|(1+\frake)}{(n+n_0)^{\frakd-\frake}}\Bigg\},\\
         &\forall ~n\in\bN, a.s.
	}	
\end{proposition}
\begin{proof}
    We have,
	\nal{
		\bE\br{\sum_{m=n}^{\infty} \gamma^{m-n} r(s_m) \Big| \cF_{n}}  =  \sum_{m=n}^{\infty} \gamma^{m-n} \sum_{i\in\cS} \bP(s_m = i|\cF_{n} ) r(i).
	}
	Also,
	\nal{
		V^{(\eps_n,\lambda_{n},Q_{n})}(s) = \sum_{m=n}^{\infty} \gamma^{m-n} \sum_{i\in\cS} (p^{(\eps_n,\lambda_{n},Q_{n})})^{m-n}(s,i) r(i).
	}
	Note that $\bE\br{\sum_{m=n}^{\infty} \gamma^{m-n} r(s_m) \Big| \cF_{n}}$ depends upon the probabilities $\left\{\bP(s_m |\cF_{n})\right\}_{m>n}$, while $V^{(\eps_n,\lambda_{n},Q_{n})}(s_{n})$ depends upon $(p^{(\eps_n,\lambda_{n},Q_{n})})^{m-n}, m\ge n$.~Hence, in order to upper-bound $\Big| \bE\br{\sum_{m=n}^{\infty} \gamma^{m-n} r(s_m,a_m) \Big| \cF_{n}} - V^{(\eps_n,\lambda_{n},x_{n})}(s_{n})\Big|$, we will need to bound the quantities $|\bP(s_m |\cF_{n})-(p^{(\eps_n,\lambda_{n},Q_{n})})^{m-n}|$ for $m \ge n$. This has been done in Lemma~\ref{lemma:diff_prob_transition_eps}, and we will use those bounds here.
	
	Upon subtracting the above two expressions and taking magnitude, we get the following,
	\nal{
		& \Big|\bE\br{\sum_{m=n}^{\infty} \gamma^{m-n} r(s_m,a_m) \Big| \cF_{n}} - V^{(\eps_n,\lambda_{n},Q_{n})}(s_{n})  \Big| \\
		& \le \Big| \sum_{m=n}^{\infty} \gamma^{m-n} \br{\sum_{i\in\cS} \bP(s_m = i|\cF_{n} ) r(i) -   (p^{(\eps_n,\lambda_{n},Q_{n})})^{m-n}(s_{n},i) } r(i)\Big|  \\
		& \le R_{\max} \sum_{m=n}^{\infty} \gamma^{m-n} (m-n)^2\left\{ \frac{3|\cA|(1+\frake)}{(n+n_0)^{\frakd-\frake}} \right\} \\	
		& \le \br{R_{\max} \int_{0}^{\infty} x^2 e^{-x} dx } \Bigg\{\frac{3|\cA|(1+\frake)}{(n+n_0)^{\frakd-\frake}}\Bigg\},
	}	
	where in the second inequality we have used Lemma~\ref{lemma:diff_prob_transition_eps}. This completes the proof. We work under the assumption $\frakd > \frake$ (note that Condition~\ref{con:seg_regret_1} ensures this) since we want the bounds to decay with $n$. 
\end{proof}

\begin{proof}[\textbf{Proof of Theorem~\ref{th:regret_bound_eps_softmax} (Regret of \seg~Q-learning)}]
\label{pf:th_regret_eps_sm}
Recall the decomposition of the instantaneous regret $\regn$ into~\eqref{adhoc:2_1_eps}-\eqref{adhoc:2_2_eps}. The term~\eqref{adhoc:2_1_eps} is bounded in Proposition~\ref{prop:bound_R_bar_eps}, while the term~\eqref{adhoc:2_2_eps} is bounded in Proposition~\ref{prop:1_eps}. Upon substituting these bounds into the decomposition~\eqref{adhoc:2_1_eps}-\eqref{adhoc:2_2_eps}, we get the following,
\nal{
& \regn  \le \frac{\gamma R_{\max}}{(1-\gamma)} \times \\
&\left[\frac{2}{(n+n_0)^{\frakd}} + 2\log(n+n_0)\frac{\br{n+n_0}^{2\kappa\es_1 +\kappa\es_{3}+\frake}}{\frakc\es_{3}} \br{C\es_4 g(n,\delta) + 2 C\es_5 g_1(n)} \right]\\
& + \br{R_{\max} \int_{0}^{\infty} x^2 e^{-x} dx } \Bigg\{\frac{3|\cA|(1+\frake)}{(n+n_0)^{\frakd-\frake}}\Bigg\}.
}
Since $\cR_N  = \sum_{n=1}^{N} \reg^{(n)}$, upon summing the above bound we get (here we use the fact that if condition $2\fraka + 6\kappa\es_1 +2\kappa\es_{3} + 2\frake<1$ holds, then we have $1-\frakd >0$ and also $1-(\frakd - \frake)>0$), 
\nal{
& \cR_N \le \frac{\gamma R_{\max}\left[\br{N+n_0}^{1-\frakd}-n^{1-\frakd}_0\right]}{\br{1-\gamma}\br{1-\frakd}} \\
& + \frac{2\log(N+n_0)\sqrt{\log(N^2\slash \delta)}\gamma  R_{\max} C\es_4}{(1-\gamma)\br{\frac{1}{2}+ \fraka + 3\kappa\es_1 + \kappa\es_{3} +\frake}\frakc\es_{3}} \br{N+n_0}^{\frac{1}{2}+ \fraka + 3\kappa\es_1 +\kappa\es_3 +\frake}\\
& + \frac{4\log(N+n_0)^2 \gamma  R_{\max} C\es_5}{(1-\gamma)\br{2\fraka + 4\frakd + 2\frake +\frake } \frakc\es_{3}} \br{N+n_0}^{2\fraka + 4\kappa\es_1 + 2\kappa\es_3+\frake }\\
&+ \br{R_{\max} \int_{0}^{\infty} x^2 e^{-x} dx } \frac{3|\cA|(1+\frake)\br{ (N+n_0)^{1-(\frakd-\frake)} - n^{1-(\frakd-\frake)}_0 }}{1-(\frakd - \frake)},~\forall N\in\bN. 
}
Since $\kappa\es_1 = \frakd, \kappa\es_2=\frake,\kappa\es_3 = \frake$, the above bound reduces to, 
\nal{
& \cR_N \le \frac{\gamma R_{\max}\left[\br{N+n_0}^{1-\frakd}\right]}{\br{1-\gamma}\br{1-\frakd}}
+ \frac{2\log(N+n_0)\sqrt{\log(N^2\slash \delta)}\gamma R_{\max}  C\es_4}{(1-\gamma)\br{\frac{1}{2}+ \fraka + 3\frakd + \frake +\frake}\frakc\es_{3}} \br{N+n_0}^{\frac{1}{2}+ \fraka + 3\frakd + 2\frake}\\
& + \frac{4\log(N+n_0)^2 \gamma R_{\max}C\es_5 }{(1-\gamma)\br{2\fraka + 4\frakd + 2\frake +\frake }\frakc_{3}} \br{N+n_0}^{2\fraka + 4\frakd + 3\frake }\\
&+ \br{R_{\max} \int_{0}^{\infty} x^2 e^{-x} dx } \frac{3|\cA|(1+\frake)\br{ (N+n_0)^{1-(\frakd-\frake)} }}{1-(\frakd - \frake)},~\forall N\in\bN. 
}
Here $\frakc\es_{3} = \frac{2(1-\gamma)\mu\es_{\min,\cS}}{|\cA|}$ and also
\nal{
\frac{C\es_4}{\frakc_{3}} & = \br{\frac{|\cA|^2 R_{\max}\cD(\cM)}{2(1-\gamma)^2 \mu\es_{\min,\cS}}}\br{\frac{8 d R_{\max} \cD(\cM)}{\sqrt{1+2\fraka + 2\frakd}}},\\
\frac{C\es_5}{\frakc_{3}} &  = \br{\frac{|\cA|^2 R_{\max}\cD(\cM)}{2(1-\gamma)^2\mu\es_{\min,\cS}}}\br{ \frac{6 \cD(\cM)^2 R^2_{\max}  }{2\fraka+2\frakd+\frake }    } \br{ \mu\es_{\min,s} + \frac{2 R_{\max}}{1-\gamma} }.
} 

This also gives us the following order result, which is useful for inferring scaling properties w.r.t. $N$ and also $\frac{1}{1-\gamma}$:
\nal{
\cR_N \le \frac{\tilde{O}\br{N^{1-\frakd} }}{1-\gamma} + \frac{\tilde{O}\br{N^{\frac{1}{2}+ \fraka + 3\frakd + 2\frake }} }{(1-\gamma)^3} 
+ \frac{\tilde{O}\br{ N^{2\fraka + 4\frakd + 3\frake } }}{(1-\gamma)^4} + \tilde{O}\br{N^{1-(\frakd - \frake)} }.
}
\end{proof}

\begin{remark}[A note on optimal scaling with $N$]
We note that the hyperparameters satisfy the conditions $2\fraka + 6\frakd +4\frake <1,\frake < \frakd , \fraka -\frake <1, \frakd \le \fraka$.~For a fixed $\gamma<1$, the exponent associated with the optimal scaling with $N$ is given by the solution of the following optimization problem:
\nal{
\min_{\fraka,\frakd,\frake} & \max \left\{1-\frakd,\frac{1}{2}+ \fraka + 3\frakd + 2\frake,2\fraka + 4\frakd + 3\frake, 1-(\frakd - \frake) \right\}\\
& \mbox{ s.t. }  2\fraka + 6\frakd +4\frake <1,\frake < \frakd , \fraka -\frake <1, \frakd \le \fraka, 0\le \fraka, \frakd.
}
This minimization problem does not have an exact optimal solution. The optimal value as an infimum is $9\slash 10$.~There is no feasible point that attains it, because the constraints include the strict condition $\frake >0$.~Here we discuss a sequence of feasible points which yield an objective value arbitrary close to $.9$: Consider the sequence $\fraka_{\ell} = \frac{1}{10}$, $\frake_{\ell} = \frac{1}{\ell}$, $\frakd_{\ell} =\frac{1}{10} - \frac{1}{\ell}$. It is easily verified that for $\ell>20$, the tuple $(\fraka_{\ell},\frakd_{\ell},\frake_{\ell})$ is a valid assignment for $(\fraka,\frakd,\frake)$, i.e. it satisfies the constraints and hence is feasible.~The objective function evaluates to $\frac{9}{10}+\frac{2}{\ell}$, and regret bound is 
\nal{
\cR_N  \le \frac{\tilde{O}\br{N^{ \frac{9}{10} + \frake_{\ell}  }}  }{1-\gamma} + \frac{\tilde{O}\br{N^{ \frac{9}{10} - \frake_{\ell}  }}  }{\br{1-\gamma}^3 } + \frac{\tilde{O}\br{N^{ \frac{6}{10} - \frake_{\ell}  }}  }{(1-\gamma)^4} + 
\tilde{O}\br{N^{ \frac{9}{10} + 2\frake_{\ell} }},
}
and can be made arbitrarily close to $\frac{\tilde{O}\br{N^{ \frac{9}{10}}}  }{\br{1-\gamma}^3 }$ by letting $\ell \to \infty$ ($\frake_{\ell} \to 0^+$).

\end{remark}

\subsection{Conditions on $n_0$}\label{sec:cond_n0_eps_sm}
The regret analysis of \seg~Q-learning utilizes the concentration results from Section~\ref{sec:SA}, Hence, for analysis of \seg~Q-learning it is required that $n_0$ satisfy Condition~\ref{con:n0} of Section~\ref{sec:cond_n0_sa}. Additionally, the following conditions are also needed, and these are specific to \seg~Q-learning. Note that these are different from the two conditions presented in Section~\ref{sec:con_n0_sm}, which appeared while analysing \seg~Q-learning.
\begin{condition}[Condition on $n_0$ for~\seg~Q-learning]
\label{cond:n0_seg-ql}
We have the following conditions on $n_0$:
\begin{enumerate}[(I)]
    \item \label{cond:n0_seg-ql-I}
    \al{
    \frac{p(s,a,s') \|Q(s',\cdot)\|_{1}}{(\ell +n_0)^{1-\frake}} >\frac{p(s,a,s')\frakd}{(\ell +n_0)^{1+\frakd}} , \forall \ell \in\bN.\label{con:8}
    }
    \item \label{cond:n0_seg-ql-II}
    \al{
    C\es_{5} g_1(\ell) & > C\es_{6} g_2(\ell),\notag \\
    \mbox{ and } \frac{1}{(\ell+n_0)^{\frakd}}  & > |\cA|e^{-(\ell+n_0)^{\frake}gap(\cM)},~\forall \ell \in\bN.  \label{con:9}
    }
    \item \label{cond:n0_seg-ql-III}
    \al{
    \frac{gap(\cM)}{2} & > 4\log(n+n_0)\frac{\br{\ell+n_0}^{2\kappa\es_1+\kappa\es_{3}}}{\frakc\es_{3}}\notag \\
& \times \br{C\es_4 g(\ell,\delta) + C\es_5 g_1(\ell)+C\es_6 g_2(\ell)},~\forall \ell \in\bN.   \label{con:12}
    }
    \item \label{cond:n0_seg-ql-IV}
    \al{
     \frac{1}{(\ell + n_0)^{\frakd}} > e^{-\br{gap(\cM)\slash 2}(\ell+n_0)^{\frake}    } ,~\forall \ell \in\bN. \label{con:10}
    }
    \item \label{cond:n0_seg-ql-V}
    \al{
    \frac{|\cA|(1+\frake)}{(\ell+n_0)^{\frakd-\frake}} > \frac{1}{(\ell+n_0)^{\frakd}} + |\cA|\frac{\br{\frakc_1 +\frakc_2 +\frakc_3}}{(\ell+n_0)^{1+\frakd-\fraka}} , \forall \ell \in\bN.  \label{con:11}
    }
\end{enumerate}  
\end{condition}

%% file: appendix_eps_softmax_auxiliary.tex
\section{Auxiliary results for Appendix~\ref{app:eps_softmax}}
\label{sec:eps_softmax_auxiliary}
\begin{lemma}
\label{lemma:sensitivity_f_eps}
	The difference between the action probabilities of $f^{(\eps,\lambda,Q)}$ and the optimal policy $f^{(0,0,Q\ust)}$ can be bounded as follows,
	\nal{
		\|f^{(\eps,\lambda,Q)}- f^{(0,0,Q\ust)} \|  & \le \eps+ \frac{\|Q-Q\ust\|}{\lambda}+ |\cA| e^{\br{-\frac{gap(\cM)}{\lambda}}},\\		
		& \mbox{ where } (\eps,\lambda,Q) \in (0,1) \times \bR_+ \times \cQ.
	}
\end{lemma}
\begin{proof}
	We have,
	\al{
		& \|f^{(\eps,\lambda,Q)}- f^{(0,0,Q\ust)} \| \notag \\
		& =  \| f^{(\eps,\lambda,Q)} - f^{(0,\lambda,Q)} + f^{(0,\lambda,Q)} - f^{(0,\lambda,Q\ust)}  + f^{(0,\lambda,Q\ust)} -  f^{(0,0,Q\ust)} \| \notag\\
		& \le \| f^{(\eps,\lambda,Q)} - f^{(0,\lambda,Q)}\| + \|f^{(0,\lambda,Q)} - f^{(0,\lambda,Q\ust)} \| + \| f^{(0,\lambda,Q\ust)} -  f^{(0,0,Q\ust)} \| \notag\\		
		& \le \eps + \frac{\|Q-Q\ust\|}{\lambda}+ \|f^{(0,\lambda,Q\ust)} - f^{(0,0,Q\ust)}\| ,\label{eq:adhoc_1_1}
	}
	where the second inequality follows since from Lemma~\ref{lemma:diff_softmax_Q} we have that $\|\frac{\partial \sigma(Q(s,\cdot)\slash \lambda)(a)}{\partial Q(s,b)}\| \le \frac{1}{\lambda}$.~To bound $\|f^{(0,\lambda,Q\ust)} - f^{(0,0,Q\ust)}\|$ we note that the probability assigned by $f^{(0,\lambda,Q\ust)}$ to a sub-optimal action $a$ in state $s\in\cS$ can be upper-bounded as,
	\nal{
		f^{(0,\lambda,Q\ust)}(s,a) & = \frac{\exp\br{Q\ust(s,a)\slash \lambda}}{\sum_{b\in\cA} \exp\br{Q\ust(s,b)\slash \lambda}} \\
		& \le \frac{\exp\br{Q\ust(s,a)\slash \lambda}}{ \exp\br{\max_{b\in\cA}Q\ust(s,b)\slash \lambda}} \\
		& \le \exp\br{-\frac{\br{\max_{b\in\cA} Q\ust(s,b)-Q\ust(s,a)}}{\lambda}}\\
		& \le \exp\br{\frac{-gap(\cM)}{\lambda}}.
	}
	This also gives $f^{(0,\lambda,Q\ust)}(s,a\ust(s)) \ge 1 - |\cA| \exp\br{\frac{-gap(\cM)}{\lambda}}$.
	For the vector $f^{(0,0,Q\ust)}(s,\cdot)$ all elements except that of $a\ust(s)$ are $0$, while that corresponding to $a\ust(s)$ is equal to $1$. Thus the magnitude of the elements of the vector $f^{(0,\lambda,Q\ust)} - f^{(0,0,Q\ust)}$ can be upper-bounded by $|\cA| \exp\br{\frac{-gap(\cM)}{\lambda}}$.
	
	The proof is completed by substituting this bound in~\eqref{eq:adhoc_1_1}.
\end{proof}

\begin{lemma}
\label{lemma:diff_kernels_eps}
	Consider the \seg~Q-learning algorithm and assume that the hyperparameters satisfy $\fraka -\frake < 1$ and $\frakd > \frake$.~We have the following bound for $m>n$,
	\nal{
		\Big| p^{(\eps_m,\lambda_{m},Q_{m})}(j,i) -  p^{(\eps_n,\lambda_{n},Q_{n})}(j,i) \Big| \le \frac{3|\cA|(1+\frake)(m-n)}{(n+n_0)^{\frakd-\frake}},~i,j\in\cS.
	}
\end{lemma}
\begin{proof}
	For simplicity suppose that $a\ust(s)$ consists of only a single action in every state. 
	
	\emph{Step I}: We will derive upper-bounds on $\Big| \frac{\partial f^{(\eps,\lambda,Q)}(s,a)}{\partial Q} |_{(\eps_n,\lambda_n,Q_n)} \Big|$ and $\Big| \frac{\partial f^{(\eps,\lambda,Q)}(s,a)}{\partial \lambda} |_{(\eps_n,\lambda_n,Q_n)}  \Big|$.~The concentration result of Theorem~\ref{th:conc_q_eps_softmax} gives us the following for $n\in\bN$,
	\nal{
         & \|Q_n - Q\ust\|_{\infty}  \\
         & \le 2\log(n+n_0)\frac{\br{n+n_0}^{2\kappa\es_1+\kappa\es_{3}}}{\frakc\es_{3}} \br{C\es_4 g(n,\delta) + C\es_5 g_1(n)+C\es_6 g_2(n)}\\
         & = : bound(n,\delta),~\forall n\in\bN.
	}
	Note that $Q\ust$ is the optimal Q-function. Hence for an action $a$ that is sub-optimal in state $s$, we have the following upper-bound on its probability 
	\nal{
		f^{(\eps_\ell,\lambda_\ell,Q_\ell)}(s,a) & = \eps_\ell + \frac{\exp\br{Q_\ell(s,a)\slash \lambda_\ell}}{\sum_{b\in\cA} \exp\br{Q_\ell(s,b)\slash \lambda_\ell}} \\
		& \le \eps_\ell + \frac{\exp\br{Q\ust(s,a)\slash \lambda_\ell + bound(\ell,\delta) \slash \lambda_\ell}}{ \exp\br{ Q\ust(s,a\ust(s))\slash \lambda_\ell - bound(\ell,\delta) \slash \lambda_\ell } } \\
		& = \eps_\ell + e^{-gap(\cM)\slash \lambda_\ell + 2 bound(\ell,\delta) \slash \lambda_\ell}\\
		& \le \eps_\ell + e^{(\ell+n_0)^{\frake}\br{ -gap(\cM) + 2 bound(\ell,\delta) } }.
	}
	Note that $n_0$ is sufficiently large so that Condition~\ref{cond:n0_seg-ql}-\eqref{cond:n0_seg-ql-IV} and hence
	\nal{
		f^{(\eps_\ell,\lambda_\ell,Q_\ell)}(s,a)  \le \frac{2}{(\ell + n_0)^{\frakd}}.
	}
    Hence, the action probabilities are upper-bounded by $\frac{2}{(\ell + n_0)^{\frakd}}$ for sub-optimal actions, and lower bounded by $1- |\cA| \frac{2}{(\ell + n_0)^{\frakd}}$ for optimal actions. We use this bound on the action probabilities in conjunction with Lemma~\ref{lemma:diff_softmax_Q} to infer the following:
	\begin{enumerate}[(i)]
		\item 
		\nal{
			\Bigg|  \frac{\partial \sigma(Q(s,\cdot)\slash \lambda)(a)}{\partial Q_\ell(s,a')} \Big|_{\lambda = \lambda_{\ell},Q=Q_{\ell}}  \Bigg| \le  
            |\cA|\frac{1}{(\ell+n_0)^{\frakd}}.
		}
		\item 
		\nal{
			\Bigg| \frac{\partial \sigma(Q(s,\cdot)\slash \lambda)(a) }{\partial \lambda} \Big|_{\lambda = \lambda_\ell,Q=Q_\ell } \Bigg|		
			\le 
            \frac{2|\cA|}{(\ell+n_0)^{\frakd}} (\ell+n_0)^{2\frake}.
		}     
	\end{enumerate}
	
	\emph{Step II}: On the set $\cG$ we have, 
	\al{
		& \Big|   p^{(\eps_m,\lambda_{m},Q_{m})}(j,i) -p^{(\eps_n,\lambda_{n},Q_{n})}(j,i) \Big| 
         \le \Big|   p^{(\eps_m,\lambda_{m},Q_{m})}(j,i) -p^{(\eps_n,\lambda_{m},Q_{m})}(j,i) \Big| \notag\\
        & + \Big|   p^{(\eps_n,\lambda_{m},Q_{m})}(j,i) -p^{(\eps_n,\lambda_{n},Q_{m})}(j,i) \Big| 
         + \Big|   p^{(\eps_n,\lambda_{n},Q_{m})}(j,i) -p^{(\eps_n,\lambda_{n},Q_{n})}(j,i) \Big| \notag\\
        & \le \br{\eps_n - \eps_m} +  \frac{2|\cA|}{(m+n_0)^{\frakd-2\frake}}  |\lambda_m - \lambda_n | + |\cA|\frac{1}{(n+n_0)^{\frakd}} \|Q_m - Q_n\| \notag\\
        & \le \frac{m-n}{(n+n_0)^{\frakd}} +  \frac{2|\cA|(1+\frake)(m-n)}{(m+n_0)^{\frakd-2\frake}(n+n_0)^{\frake}}  + |\cA|\frac{\br{\frakc_7 +\frakc_8 +\frakc_9}(m-n)}{(n+n_0)^{1+\frakd-\fraka}}  \notag \\ 
        & \le \frac{m-n}{(n+n_0)^{\frakd}} +  \frac{2|\cA|(1+\frake)(m-n)}{(n+n_0)^{\frakd-\frake}} + |\cA|\frac{\br{\frakc_7 +\frakc_8 +\frakc_9}(m-n)}{(n+n_0)^{1+\frakd-\fraka}},      \label{ineq:adhoc_regret_eps_sm}
	}
	where in the second inequality we have used the bounds (i), (ii) on derivatives derived in Step I above along with the fundamental theorem of calculus~\citep{folland1999real}, while third inequality is obtained by noting that Assumption~\ref{assum:9} implies $\|F(x_n,y_n) -x_n +M_{n+1}\| \le \frakc_1 +\frakc_2 +\frakc_3$ so that
	we have $\|x_m - x_{n}\|\le \br{\sum_{\ell=n}^{m} \beta_{\ell}}\br{\frakc_1 +\frakc_2 +\frakc_3}$. We have also used the following: Since $\beta_{\ell} = \frac{1}{(\ell +n_0)^{1-\fraka}}, \sum_{\ell=n}^{m} \beta_{\ell} \le \frac{m-n}{(n+n_0)^{1-a}}$. Since $\lambda_{\ell} = \frac{1}{(\ell+n_0)^{\frake}}$, we have $|\lambda_m -\lambda_n| \le \frac{(1+\frake)(m-n)}{(n+n_0)^{\frake}}$. Also, $\eps_{\ell}=\frac{1}{(\ell + n_0)^{\frakd}}$, and hence $\eps_n - \eps_m \le \frac{m-n}{(n+n_0)^{\frakd}}$. We observe that the condition $\fraka -\frake < 1$ implies $1+\frakd -\fraka > \frakd -\frake$. Proof is then completed by noting that $n_0$ is taken to be sufficiently large so that Condition~\ref{cond:n0_seg-ql}-\eqref{cond:n0_seg-ql-V} is satisfied.~Proof then follows by substituting this into~\eqref{ineq:adhoc_regret_eps_sm}. 
\end{proof}
Proof of the next result is omitted since it is similar to that of Proposition~\ref{lemma:diff_prob_transition}.
\begin{lemma}\label{lemma:diff_prob_transition_eps}
	For \seg~Q-learning we have the following for $m > n$,
	\nal{
| \bP(s_{m+1}= i| s_{n}=j) - (p^{(\eps_n,\lambda_{n},Q_n)})^{m+1-n}(j,i)|   \le  (m-n) \left\{ \frac{3|\cA|(1+\frake)(m-n)}{(n+n_0)^{\frakd-\frake}} \right\},
	}
    where $i,j \in\cS$.
\end{lemma}

%% file: appendix_miscellaneous.tex
\section{Miscellaneous technical results}\label{app:miscellaneous}
\begin{lemma}\label{lemma:chi_approx}
	With $\beta_{\ell} = \frac{\beta}{(\ell+n_0)^{1-\fraka}}$ we have 
	\nal{
		\chi(i+1,n)   \le \exp\br{-\sum_{j=i+1}^{n}\frac{\beta}{(n_0 + j)^{1-\fraka}}}  \le \exp\br{-\beta \sum_{j=i+1}^{n}\frac{1}{j+n_0}} = \br{\frac{i+1+n_0}{n+n_0}}^\beta.
	}
	Hence $\beta_i \chi(i+1,n) \le \frac{\beta}{(i+n_0)^{1-\fraka}}\br{\frac{i+1+n_0}{n+n_0}}^\beta$.
\end{lemma}

\begin{lemma}\label{lemma:diff_softmax_Q}
	Recall $\sigma(Q(s,\cdot)\slash \lambda)(a) = \frac{e^{Q(s,a)\slash \lambda}}{\sum_{b\in\cA} e^{Q(s,b)\slash \lambda}}$. 
	We have the following derivatives:	
	\nal{
		\frac{\partial \sigma(Q(s,\cdot)\slash \lambda)(a)  }{\partial Q(s,b)} & = 
		\begin{cases}
			\frac{1}{\lambda}\sigma(Q(s,\cdot)\slash \lambda)(a) \br{1-\sigma(Q(s,\cdot)\slash \lambda)(a) } \mbox{ if } b = a,\\
			-\frac{1}{\lambda} \sigma(Q(s,\cdot)\slash \lambda)(a)\sigma(Q(s,\cdot)\slash \lambda)(b) \mbox{ if } b \neq a.
		\end{cases}\\
		\frac{\partial \sigma(Q(s,\cdot)\slash \lambda)(a) }{\partial \lambda}  
		& = \sigma(Q(s,\cdot)\slash \lambda)(a)\left[\sum_{b\neq a}  \sigma(Q(s,\cdot)\slash \lambda)(b)  \frac{Q(s,b)}{\lambda^2} \right].    
	}
\end{lemma}

\begin{lemma}\label{lemma:lower_bound_action_prob}
	We have 
	\nal{
		\min_{a\in\cA, Q\in\cQ} \sigma(Q(s,\cdot)\slash \lambda)(a)> \frac{1}{|\cA|\exp\br{\frac{R_{\max}}{\lambda(1-\gamma)}}},~\forall s\in\cS.
	}
\end{lemma}
\begin{proof} 
	Since the rewards are non-negative, we have $Q(s,a)\ge 0$, hence $\exp\br{\frac{Q(s,a)}{\lambda}}\ge 1$; so that the numerator in the definition of $\sigma(Q(s,\cdot)\slash \lambda)(a)$ is greater than or equal to $1$.~Moreover, since $r(s,a)\le R_{\max}$, we have $Q(s,b) \le \frac{R_{\max}}{1-\gamma}$ or $\exp\br{\frac{Q(s,b)}{\lambda}} \le \exp\br{\frac{R_{\max}}{(1-\gamma)\lambda}}$. Thus the denominator in the definition of $\sigma(Q(s,\cdot)\slash \lambda)(a)$ is less than $|\cA|\exp\br{\frac{R_{\max}}{(1-\gamma)\lambda}}$.~Proof then follows by substituting these bounds in the expression for $\sigma(Q(s,\cdot)\slash \lambda)$.
\end{proof}

\begin{lemma}(Lemma 11 from \cite{chandak_cdc})
\label{lemma:recursion_bound}
Consider the sequences 
\nal{
a_n=\frac{a}{(n+n_0)^{\rho}},\qquad  b_n=\frac{b}{(n+n_0)^{\rho'}},~n\in\bN.
}
Define 
\nal{
S_n : =\sum_{i=0}^{n-1} b_i\prod_{j=i+1}^{n-1}(1-a_j), n\in \bN.
}

Part 1 - 
Suppose $\rho=1$ and $\rho'\in(1,2]$. If $a\geq 2(\rho'-1)$,
then 
\nal{
S_n\leq \frac{2b_n}{a_n}.
}

Part 2 - Suppose $\rho<1$. If $n_0\geq \left(\frac{2(\rho'-\rho)}{a}\right)^{1/(1-\rho)}$ and $\rho'>\rho$, then 
\nal{
S_n\leq \frac{2b_n}{a_n}.
}
\end{lemma}

\begin{lemma}\label{lemma:sensitivity_prob_dist}
    Let $P = \{p(i,j)\}_{i,j\in\cS},P' = \{p'(i,j)\}_{i,j\in\cS}$ be two Markov chain transition matrices. Then we have $\| (P)^{\ell}(i,j) -  (P')^{\ell}(i,j)   \| \le \ell \|P-P'\|$. 
\end{lemma}
\begin{proof}
Then, the $\ell$-step transition probabilities of first chain is given by $P^{\ell}$, while that of the second is given by $(P')^{\ell}$. Now,
\nal{
(P)^{\ell} - (P')^{\ell}  = (P-P')P' \cdots P' + P(P-P')P' \cdots P' + \cdots + P^{n-1}(P-P'),
}
and hence,
\nal{
\|(P)^{\ell} - (P')^{\ell} \|_{\infty} \le \sum_{k=1}^{\ell} \|(P)^{k-1}\|_{\infty} \|P-P'\|_{\infty}\|P' P'\cdots P'\|_{\infty} \le n \|P-P'\|,
}
where we have used $ \|P^{k-1}\|_{\infty}=1$ and $\|P' P'\cdots P\|_{\infty} =1$.

\end{proof}